\documentclass[twoside,twocolumn]{article}

\usepackage[scaled=.95]{newtxtext,newtxmath} 
\usepackage[T1]{fontenc} 
\usepackage[utf8]{inputenc} 
\usepackage[polish,slovene,english]{babel} 
\usepackage{microtype} 
\usepackage{graphicx} 
\usepackage{makecell} 
\usepackage{tabularx} 
\usepackage{booktabs}
\usepackage{bchart} 
\usepackage{multirow} 
\usepackage[title]{appendix}
\usepackage{tcolorbox}
\usepackage{adjustbox}
\usepackage{siunitx}
\newcolumntype{Y}{>{\raggedright\arraybackslash}X}

\usepackage{tikz}
\usetikzlibrary{calc}

\newcommand{\thumb}[1]{%
  \adjustbox{valign=m}{%
    \fboxsep=0pt\relax
    \fboxrule=0.3pt\relax
    \fbox{\includegraphics[height=12mm,keepaspectratio]{#1}}
  }%
}

\newtcbox{\thumbcontainer}{colback=white, colframe=gray!25,
  boxrule=0.3pt, arc=2mm, left=2mm, right=2mm, top=1.5mm, bottom=1.5mm}

\newcommand{\thumbstrip}[1]{%
  \adjustbox{max width=\linewidth, valign=m}{%
    \thumbcontainer{%
      \tikz[baseline=(base.base)]{
        \node (base) at (0,0) {};
        \foreach \f [count=\i] in {#1} {
          \node[anchor=base] at (0.0,0) {\thumb{\f}};
          \ifnum\i<1000\hspace{5pt}\fi
        }
      }%
      \hspace{3mm}
    }%
  }%
}

\newcommand{\miniwithlabels}[1]{%
\begin{tikzpicture}[baseline=(img.base)]
  \node[inner sep=0,outer sep=0] (img) {\includegraphics[width=\dimexpr\textwidth-6pt\relax]{#1}\hspace*{6pt}};
  \path[use as bounding box] (img.north west) rectangle (img.south east);
  \begin{scope}[overlay]
    \node[rotate=90, anchor=east, font=\scriptsize]
          at ([xshift=0pt,yshift=15pt]img.west) {Samples};
    \node[anchor=north, font=\scriptsize]
          at ([yshift=5pt]img.south) {Score};
  \end{scope}
\end{tikzpicture}%
}

\newcommand{\miniwithlabelsShift}[2]{
\begin{tikzpicture}[baseline=(im.base)]
  \node[inner sep=0,outer sep=0] (im)
    {\hspace*{#2}\includegraphics[width=\dimexpr\textwidth-#2\relax]{#1}};
  \path[use as bounding box] (im.north west) rectangle (im.south east);
  \begin{scope}[overlay]
    \node[rotate=90, anchor=east, font=\scriptsize]
          at ([xshift=#2,yshift=15pt]im.west) {Samples};
    \node[anchor=north, font=\scriptsize]
          at ([xshift=#2,yshift=5pt]im.south) {Score};
  \end{scope}
\end{tikzpicture}
}

\usepackage[hang]{footmisc} 
\usepackage[hmarginratio=1:1,margin={1.9cm,1.9cm},top=26mm,bottom=26mm,columnsep=20pt]{geometry}
\usepackage{caption} 
\usepackage{threeparttable}
\usepackage{array}
\usepackage{lettrine} 

\usepackage{enumitem} 
\setlist[itemize]{noitemsep} 

\usepackage{abstract} 

\usepackage{titlesec} 
\titleformat{\section}[block]{\large\bfseries}{\thesection}{1em}{\MakeUppercase}{} 
\titlespacing*\section{0pt}{6pt}{5pt}
\titleformat{\subsection}[block]{\large\bfseries}{\thesubsection}{1em}{} 
\titlespacing*\subsection{0pt}{6pt}{5pt}
\titleformat{\subsubsection}[runin]{\normalsize\itshape}{\thesubsubsection}{1em}{} \titlespacing*\subsubsection{0pt}{6pt}{5pt}

\usepackage{fancyhdr} 
\fancypagestyle{specialfooter}{%
  \fancyhf{}
  
  \fancyfoot[L]{©2025 Copyright held by the authors. This is the authors' version of the work. It is posted here for your personal use. Not for redistribution.}
}

\usepackage{titling} 
\usepackage{hyperref} 

\usepackage[round]{natbib} 
\newlength{\bibitemsep}
\newlength{\bibparskip}
\let\oldthebibliography\thebibliography
\renewcommand\thebibliography[1]{%
  \oldthebibliography{#1}%
  \setlength{\parskip}{\bibitemsep}%
  \setlength{\itemsep}{\bibparskip}%
}

\usepackage{algorithm,algorithmic}

\usepackage{amsmath}
\makeatletter
\@ifundefined{openbox}{}{}
\makeatother
\usepackage{amsthm}

\newtheorem{theorem}{Theorem}
\theoremstyle{definition}
\newtheorem{definition}{Definition}
\newtheorem{lemma}{Lemma} 
\newtheorem{assumption}{Assumption} 
\newtheorem{remark}{Remark}
\newtheorem*{customthm}{Theorem 2}
\usepackage{float}

\title{Video and Language Alignment in 2D Systems for 3D Multi-object Scenes with Multi-Information Derivative-Free Control} 
\author{%
\textsc{Jason Armitage} \\
\normalsize University of Zurich \\
\normalsize Switzerland
\and 
\textsc{Rico Sennrich} \\
\normalsize University of Zurich \\
\normalsize Switzerland
}
\date{} 

\begin{document}

\maketitle
\pagestyle{empty} 
\thispagestyle{specialfooter}

\begin{abstract}
\vspace*{-.9em}
\noindent 
Cross-modal systems trained on 2D visual inputs are presented with a dimensional shift when processing 3D scenes. An in-scene camera bridges the dimensionality gap but requires learning a control module. We introduce a new method that improves multivariate mutual information estimates by regret minimisation with derivative-free optimisation. Our algorithm enables off-the-shelf cross-modal systems trained on 2D visual inputs to adapt online to object occlusions and differentiate features. The pairing of expressive measures and value-based optimisation assists control of an in-scene camera to learn directly from the noisy outputs of vision-language models. The resulting pipeline improves performance in cross-modal tasks on multi-object 3D scenes without resorting to pretraining or finetuning.
\end{abstract}



\label{sec:intro}

Vision-language models (VLMs) trained on 2D visual inputs ingest 3D scenes in the form of sequences of discrete viewpoints. Systems rely on this method to perform tasks where samples at test include an additional dimension~\cite{wang2022clip, voigt-etal-2023-paparazzi, xue2024ulip, liu2024openshape}. Consider a visual sequence to be a set of views rendered from a camera moving over a 3D scene. In this case, the problem reduces to predicting in order the location of viewpoints over $x$-, $y$-, and $z$-axes. Given a system and a 3D reconstructed scene, we posit that an optimal sequence of 2D viewpoints exists, that is the sequence that - when paired with the respective linguistic representations - is most likely to return an accurate prediction from the VLM.

\begin{figure}[hbt!]
  \begin{tcolorbox}[
    colframe=black,
    colback=white,
    boxrule=1.0pt,
    width=\linewidth,
    left=0pt,
    right=0pt,
    top=0pt,
    bottom=0pt,
    arc=0pt
  ]
    \begin{minipage}{\linewidth}
      \centering
      \hspace*{-1.0em}%
      \adjustbox{%
        trim=10mm 0mm 8mm 0mm,clip,%
        width=\textwidth%
      }{%
        \includegraphics{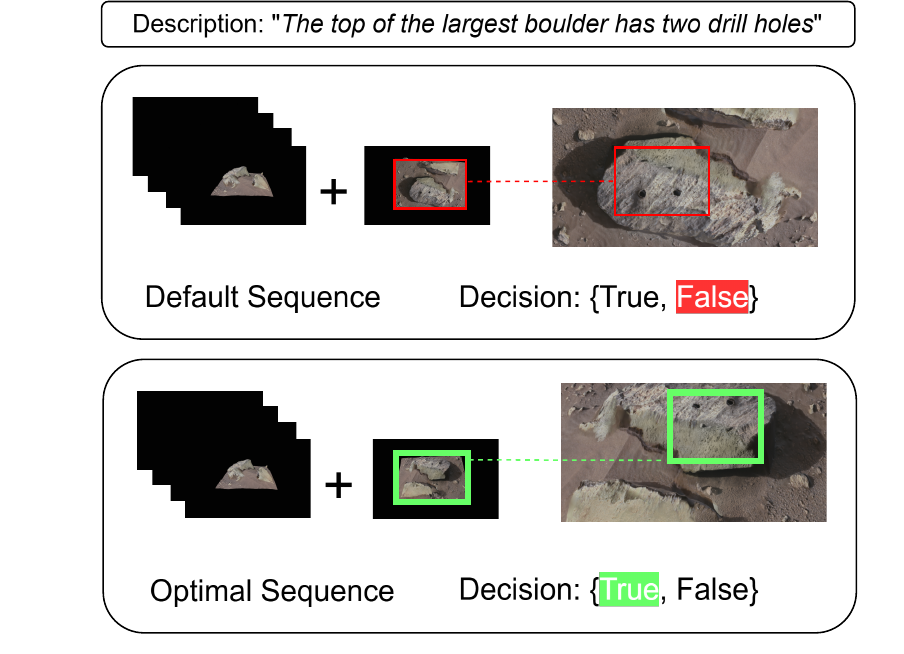}%
      }%
      \vspace{-1em}
      \caption{Textures of objects in 3D scenes vary in appearance depending on the position of the in-scene camera. Aligning a description with a scene is even harder when referenced objects belong to a group such as a single boulder in an outcrop on Mars. An optimal sequence of viewpoints improves the decisions of VLMs trained on 2D data where understanding a 3D reconstructed scene relies on a set of views.}
      \label{fig:front}
    \end{minipage}
  \end{tcolorbox}
\end{figure}

Why should we be concerned with the precise locations of the viewpoints presented to the VLM? Reasoning over the features of several objects at the same time is highly challenging for these systems \cite{linghu2024multi}. In the generation of 3D scenes, systems can expect with some confidence that the left door of a car will be the same colour as the right. The front view of a building is in contrast a less reliable indicator of the appearance of its roof or rear door. A high level of confidence in a VLM's assessments is of utmost importance if these systems are to be of use for experts performing scientific analysis. In planetary science, features such as a specific boulder's texture form the core evidence when performing geological analysis on visual data from a Mars rover (see Figure \ref{fig:front}). Object appearances vary depending on the lighting at the time of day, wind conditions, and the position of the cameras in relation to the objects~\cite{bell2022geological,paar2023three}. Rock surface details are highly variable: features of interest clearly visible on one side of a boulder may be covered by detritus on all other sides. 

In addition to these variances in 3D scenes, the internal dynamics of VLMs are also unknown. Generalisation is improved for tasks where the test samples are 2D by finetuning. In contrast, availability of cross-modal training data with 3D visual inputs is restricted~\cite{conficlrPooleJBM23}. This problem is acute in exploratory science where visual evidence is novel and costly to generate. Additional training is compromised in frontier research by limitations on domain-relevant 3D assets~\cite{hofmann2023record,barnes2018geological,yuan2024method}.

We propose a straightforward approach to predicting a VLM's errors and tune an in-scene controller guided by a measure that quantifies both the information presented in the scene and model uncertainty. The success of this approach will rely on the ability of our measure to express the complexity of the scene. A measure of one or two features alone is insufficient - all available variables should be accounted for. The estimated value of this metric should combine the information for all variables in a way that limits redundancies - or regret as defined in \cite{krichevsky1968relation} - to a minimum. To limit compute and data costs, we target a process that avoids differentiable training associated with auxiliary models or finetuning the VLM.

Information theory provides an efficient and principled framework in which to develop statistical approaches with provable bounds~\cite{kullback1997information}. We opt for a form of mutual information estimation over multiple variables (MI) as the basis for our metric. Assuming no access to the VLM's parameters at inference time, we apply a zeroth-order algorithm (ZO) to optimise the expressiveness of our MI measure by reducing the redundancies between variables~\cite{chen2017zoo,malladi2023fine,wollstadt2023rigorous} of the model's predictions. Our multi-information metric~\cite{studeny1987asymptotic} optimised with ZO algorithm (MI-ZO) overcomes the theoretical challenges of estimating MI measures over $n>2$~\cite{berrett2019efficient} mixed discrete and continuous~\cite{gao2017estimating} variables using active regret minimisation.  

Accurate and fast reasoning on 3D scenes is a key to unlocking the potential of integrating 2D VLMs into processes where automation complements human expertise. An acute challenge is presented by scenes populated with similar objects containing minor differences. In addition to forms of scientific analysis noted above, this challenge is evident in asset creation for 3D virtual environments. Examples include generating a set of buildings to populate a street and 3D representations of components in a particle accelerator~\cite{chen2022learning,peixoto2020getting,huang2022method}. A deficit of efficient methods to evaluate and perform rapid testing on 3D generations slows development and constrains adoption of systems~\cite{hofmann2023record}. New methods for assisting experts in planetary science and adjacent domains are also required as the volume of observations and data processing are on the rise \cite{nasa2024expanding,cuillandre2024euclid}. 

We formulate the problem of the dimensional shift presented by 3D scenes as adaptive control of the in-scene camera to return an optimal sequence of actions. In robotics, efficient methods from control theory are used to optimise the selection of sequences of views~\citep{swain1993promising,tallamraju2019active}. Here control of the camera is optimised by measuring the information of the visual scene in relation to the linguistic description. The result is a method for applications with complex 3D scenes that requires no access to the VLM's parameters or costly backpropagation~\cite{malladi2023fine}. Our research presents four contributions to improve VLMs trained on 2D data in analysis performed on 3D multi-object scenes:
\begin{itemize}
\item A new algorithm denoted as MI-ZO that estimates online multi-information with active regret minimisation on $n>2$ variables using zeroth-order optimisation. Our method reduces overlaps between a set of continuous and discrete variables representing cross-modal inputs with a  novel application of zeroth-order optimisation. A set of multivariate multi-information proposals is provided. Theoretical proofs for our method are added in Appendix~\ref{sec:app_proofs}.
\item A novel framework to apply adaptive control with multivariate information theoretic measures to predict the actions of an in-scene camera using noisy VLM outputs. Our controller is tailored to low-data settings with compute-efficient polynomial regression, least-squares approximation, and an interaction matrix. 
\item A diagnostic comprising 3D scenes of polygons and descriptions with two levels of complexity termed UC-3DS-MI (Uniform and Complex 3D Scenes for Multi-information) and a new custom measure of variance to demonstrate the relation of online feedback and active regret minimisation in estimating multivariate mutual information. 
\item Three new cross-modal benchmarks with 3D multi-object scenes (GeoProperties-3DS, FeatureID-3DS, and PartialView-3DS). We introduce sets of 3D scenes with language descriptions to evaluate methods for controlling an in-scene camera to enable a VLM system to analyse object properties, handle feature identification on constrained action counts, and adapt to visual occlusions.
\end{itemize}
Additional details are available on our project page at \url{https://mi-zo.github.io/mi-zo/}.


\section{Problem}
\label{sec:setup}
We reduce the challenge of promoting accurate and fast reasoning by VLMs on 3D scenes to predicting a sequence of camera actions on the scene that returns a correct response with the least number of views. To predict the optimal sequence of actions, we propose measuring the information capacity of our multimodal inputs. 

An optimal method will return accurate assessments of the information in the inputs at low cost. Information theoretic principles are a natural choice for these measurements: 
\begin{itemize}
\item Simple measurements on the visual and linguistic elements of the inputs can be obtained at negligible cost and serve as additional entropy sources.
\item Entropy is a fundamental logarithmic measure of information and a basis for the statistical theory of estimation~\cite{shannon1949mathematical, kullback1997information}. In cases with $n\text{=}2$ variables, entropy estimators have been proposed that provide theoretical guarantees of convergence at different scales of data - including in problems where sample sizes are small~\cite{kraskov2004estimating}. 
\item The problem is that mutual information provides no theoretical guarantees of optimality when the number of variables or individual measurements exceeds two. Short of perfect knowledge on inputs, a measurement with $n>2$ variables may be negative when some nonzero quantity of information is redundant~\cite{te1980multiple} invalidating our estimate. 
\end{itemize}

Our next step is to demonstrate what the problem of negative measurement in MI estimation and the resulting regret documented in theoretical work~\cite{krichevsky1968relation} means in practice. To make this relevant to 3D multi-object scenes, we have designed a novel diagnostic called UC-3DS-MI. 

\subsection{An Empirical Demonstration of Regret}
Our diagnostic UC-3DS-MI is a novel set of 3D scenes with varying numbers of objects and a range of features (see Appendix~\ref{sec:app_data}). The aim is to demonstrate how minimising regret can result in accurate expression of the information in both simple uniform and complex scenes. Complexity is defined in simple terms as the number of colours and sides of the simple polygons in sample scenes. In contrast quantifying complexity in relation to real-world objects is undermined by biases in the training data of a VLM~\cite{parashar2024neglected}. The International Space Station is a complex device with a non-uniform geometry, but trivial for large VLMs to handle as it appears frequently in datasets.

We next define the setup used for the UC-3DS-MI diagnostic and in our experiments (see Figure \ref{fig:benchmark_setup}). Prediction errors in VLM decisions and complexity estimates are collected for $n$ demonstrations ahead of starting the runs. During evaluation (see also Appendix~\ref{sec:app_data}), the sequence of camera actions in the measurement round contains default rotations. In the correction round, the controller predicts camera actions tuned on feedback from the demonstrations and measurement round. VLM decisions on scene summary questions in this second round are scored and averaged over all samples in the benchmark.

\begin{figure}[hbt!]
    \begin{tcolorbox}[colframe=black,colback=white,boxrule=1.0pt,width=\linewidth,
                      left=0pt,right=0pt,top=0pt,bottom=0pt,arc=0pt]
        \begin{minipage}{\linewidth}
            \raggedright
            \bf \small a) Measurement round
            \par
            \normalfont Run measurement on a default sequence of actions.
            \par
            \vspace{0.25em}
            
            \begin{minipage}{\linewidth}
                \centering
                \hspace{-1.0em}
                \includegraphics[width=\textwidth]{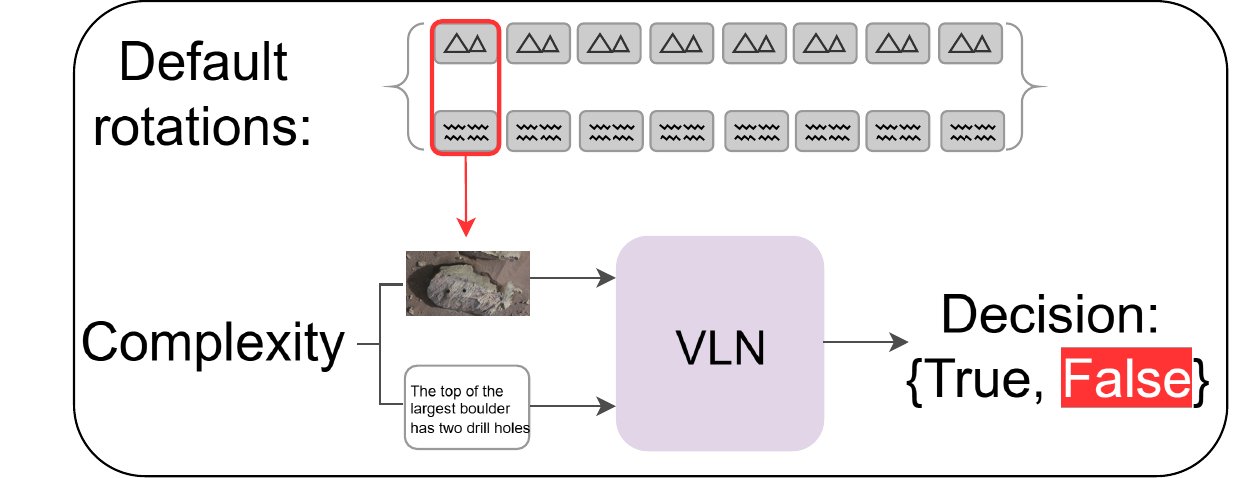}%
                \label{fig:am}
            \end{minipage}

            \vspace{0.25em}
            \bf \small b) Correction round
            \par
            \normalfont Run correction on a controller predicted sequence of actions.
            \par
            \vspace{0.25em}
            
            \begin{minipage}{\linewidth}
                \centering
                \hspace{-1.0em}
                \includegraphics[width=\textwidth]{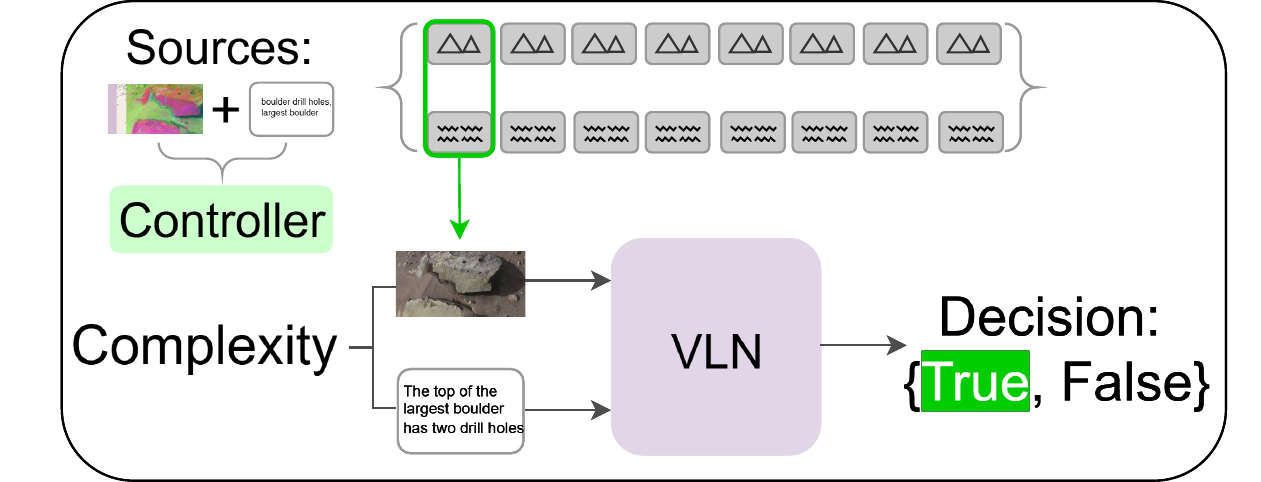}%
                \label{fig:bc}
            \end{minipage}

            \vspace{-0.75em}
            \caption{Evaluation consists of two rounds with a set number of viewpoints. In the 
            correction round, an in-scene camera controller uses multiple sources of information to predict a sequence of camera actions. In each round a VLM makes decisions based on viewpoints 
            and descriptions.}
            \label{fig:benchmark_setup}
        \end{minipage}
    \end{tcolorbox}
\end{figure}

In UC-3DS-MI, our analysis compares the impact of MI measurements for samples where the VLM system responses are correct and incorrect. We compute the means of multivariate measurements and correctness labels for each viewpoint to estimate the joint probability densities over both the full set of scenes - and the uniform and complex subsets. Two variants of MI metrics learned by our MI-ZO algorithm are assessed against bivariate and multivariate measures estimated by standard methods:

\begin{itemize}
\item $GO\text{-}LED\text{-}OL\textsubscript{ar}$. Regret minimisation over four inputs including global and local visual inputs from the CIELAB colour space.
\item $GH\text{-}LED\textsubscript{ar}$. Regret minimisation over three inputs including a global visual hue variable extracted from the HSV colour model.
\end{itemize}

Results for our information theoretic measurements are reported by variant and method in Figure~\ref{fig:hist_go_led_cv_all} (see also Appendix~\ref{sec:app_results} for additional results). Note in the figure that variants of the two metrics with active regret minimisation $MI\text{-}ar$ separate samples with correct and incorrect labels by learning the relation between information content in a scene and system error. The outcome with camera control is that the most advantageous view of the scene can be selected. Additional results from the diagnostic are provided in Appendix~\ref{sec:app_numerical}.

\begin{figure}[hbt!]
    \begin{tcolorbox}[
        colframe=black,colback=white,boxrule=1.0pt,width=\linewidth,
        left=0pt,right=0pt,top=0pt,bottom=0pt,arc=0pt
    ]
        \begin{minipage}{\linewidth}
            \raggedright
            \centering
            
            \underline{\bf \small GO-LED-OL}
            \begin{minipage}{\linewidth}
                \centering
                \begin{minipage}{0.45\textwidth}
                    \centering
                    \miniwithlabelsShift{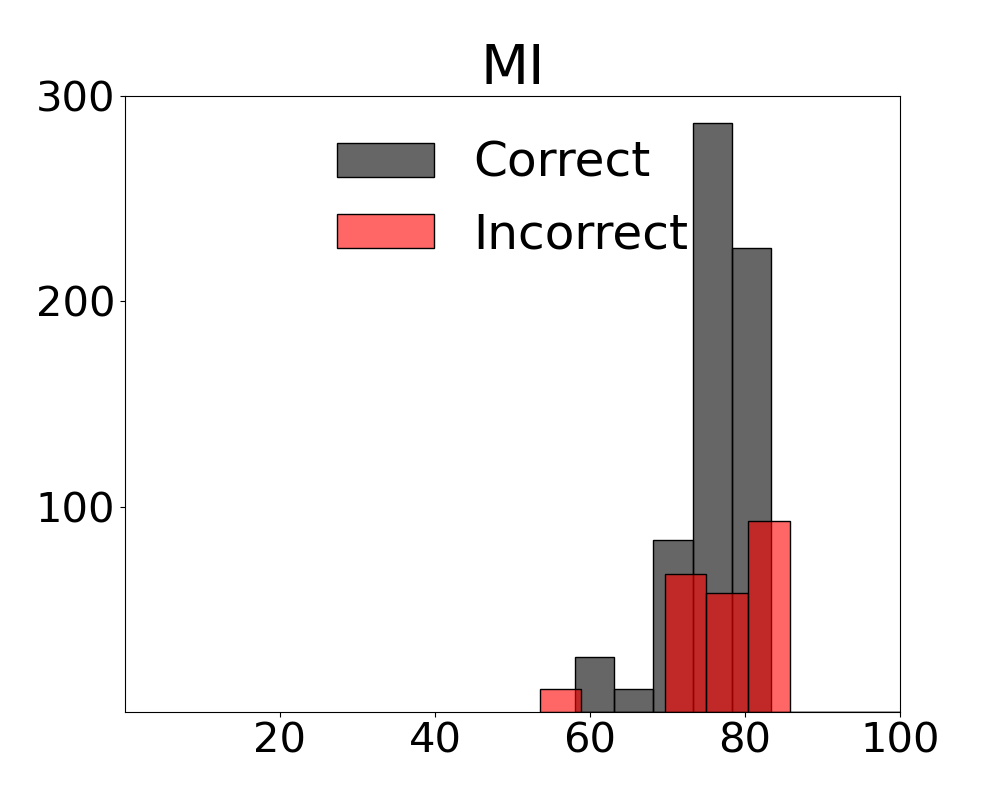}{6pt}
                    \label{fig:left1}
                \end{minipage}
                \hfill
                \begin{minipage}{0.45\textwidth}
                    \centering
                    \miniwithlabels{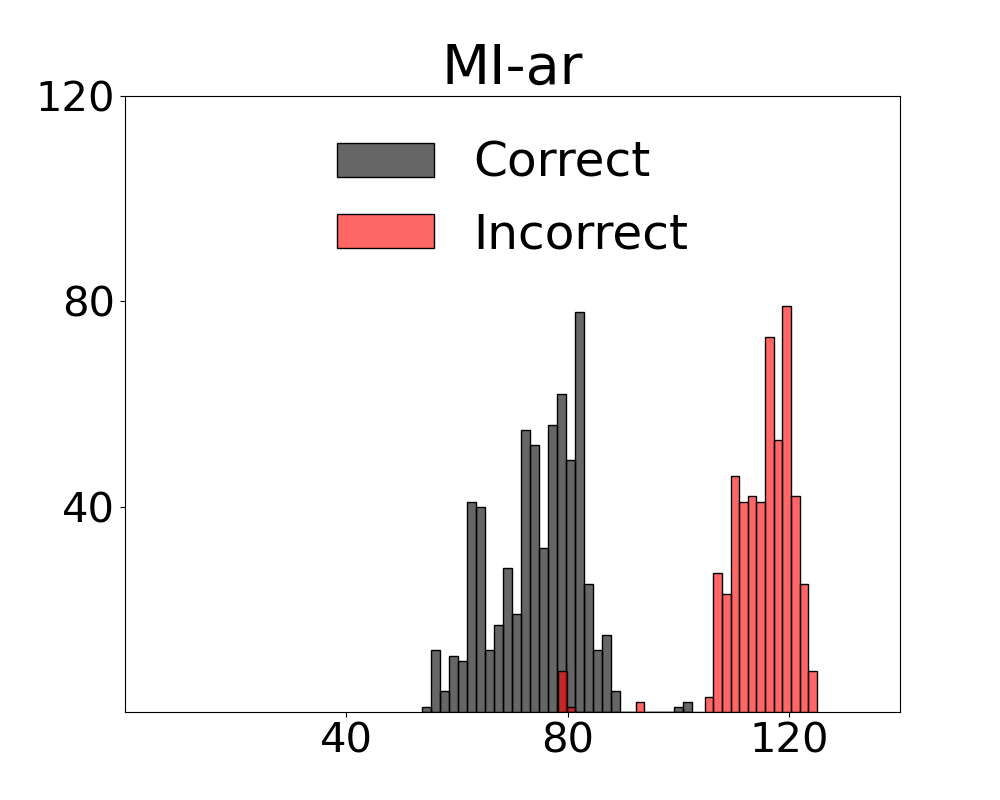}
                    \label{fig:right1}
                \end{minipage}
            \end{minipage}
            \vspace{1em}

            \vspace{-2.0em}
            \underline{\bf \small GH-LED}
            \begin{minipage}{\linewidth}
                \centering
                \begin{minipage}{0.45\textwidth}
                    \centering
                    \miniwithlabelsShift{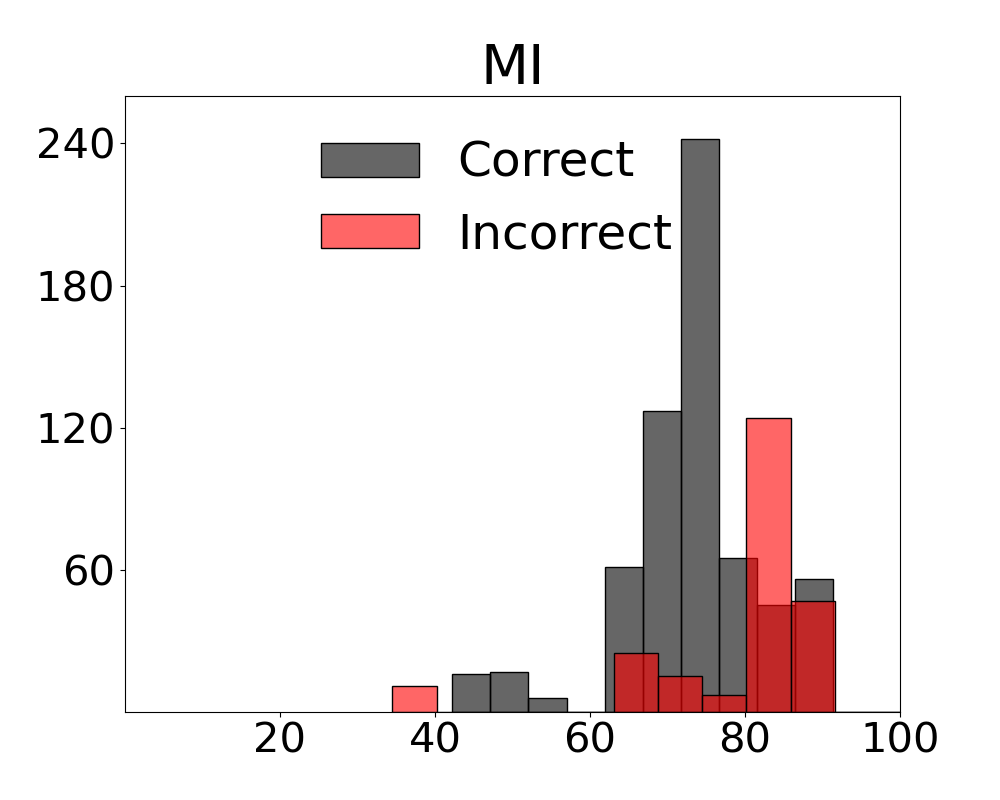}{6pt}
                    \label{fig:ghled_left1}
                \end{minipage}
                \hfill
                \begin{minipage}{0.45\textwidth}
                    \centering
                    \miniwithlabels{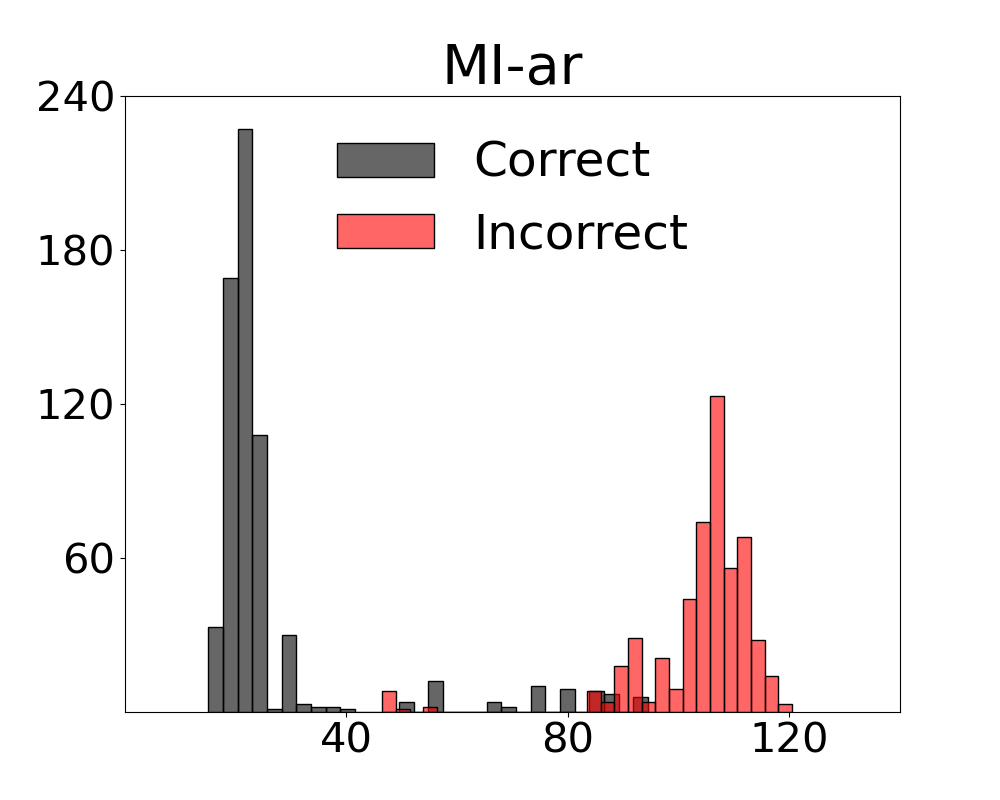}
                    \label{fig:ghled_right1}
                \end{minipage}
            \end{minipage}

            \vspace{-1.5em}
            \caption{Variants of our two multi-information metrics with active 
            regret minimisation ($MI\text{-}ar$) compare favourably with versions 
            calculated with standard methods (MI) on how the measures distribute scores in relation to the decisions of a VLM. $MI\text{-}ar$ groups the respective centres of mass for correctness decisions over separate ranges of the distributions.}
            \label{fig:hist_go_led_cv_all}
        \end{minipage}
    \end{tcolorbox}
\end{figure}


\section{Related Work}
Ahead of detailing our method to obtain expressive MI estimates, we present prior work in control theory for computer vision, measuring mutual information over multiple variables, and zeroth-order optimisation.

\textbf{Control Theoretic Algorithms in Computer Vision} Our work defines a controller that uses simple estimates of information content in the inputs for cross-modal tasks with 3D data to control an in-scene camera. Efficient control methods for tasks with 3D visual inputs form a canonical topic in robotics~\cite{christie2008camera,wiedemann2023training,gonultas2023system} but have only been explored in a handful of cases for tasks in 3D scenes. System identification - a standard method in control theory - is adapted by \cite{jaques2021vision} to calibrate a camera and estimate 3D poses of objects in scenes using videos. The same authors introduced system identification and control into deep learning models to learn physics from images \cite{jaques2020iclr}. System identification is the term used in control theory for the process of learning or improving a model with measurements~\cite{ljung1994modeling}. We also propose a method that fits the broad definition of system identification in using data to learn a system - in this case, to learn the system that leads to a 2D VLM incorrectly interpreting a 3D scene. Our approach differs in using efficient information theoretic measures to identify system errors. The result is that no pretraining or finetuning is used and only a handful of samples are required to run demonstrations.  

\textbf{Multivariate Mutual Information} Our proposal for measuring multi-information is grounded in prior work assessing the quantity of information in sets of entropy sources expressed in statistical measurements~\cite{watanabe1960information}. Seminal theoretical analysis on challenges in estimating mutual information on multiple variables includes \cite{te1980multiple} and \cite{berrett2019efficient}. \cite{mohammadi2018multivariate} propose a method to eliminate redundant mutual information in sets of variables ahead of learning. \cite{cabeli2021reliable} improve conditional mutual information estimation for mixed variable types designed for causal models~\cite{verny2017learning} by switching negative results to null values. \cite{ijcai17steeg} extended the hierarchical decomposition in the work of \cite{watanabe1960information} and earlier publications~\cite{ver2016information} for unsupervised representation learning in several domains. Our approach to measuring information content aligns closer to multi-information extending the bivariate case outlined by \cite{studeny1987asymptotic} to $n > 2$ variables. Sources in our setup are visual variables extracted from scenes.

\textbf{Zeroth-order Optimisation} Zeroth-order optimisation (ZO) methods minimise an objective function without resort to direct evaluation of the gradient. In this way, ZO dispenses with the computational costs of backpropagating through the graph. Methods that work directly with function values were an early area of investigation in minimising the cost in the form of a function~\citep{conn2009introduction}. \cite{malladi2023fine} modified the ZO-SGD process with a memory-efficient algorithm that finetunes LLMs with non-differentiable metrics. Efficient optimisation on function values motivated \cite{hoffman2022optimizing} to use the ZO-AdaMM~\cite{chen2019zo} algorithm to enhance molecule design. In our use case of combining entropies, ZO fits our requirements from an optimisation method in limiting computational operations and replacing access to model parameters with non-differentiable scores. Our ZO algorithm ensures MI expresses the information in inputs by reducing the redundancies between the variables. To our knowledge, the application of ZO in multi-information estimation is a novel proposal in optimisation.


\section{Multi-information Estimation with Zeroth-order Optimisation}
In this section, we specify a method for estimating multi-information with active regret minimisation that minimises redundancies in a set of entropy sources with a ZO algorithm (see Algorithm \hyperref[alg:algo1]{1}). We then outline the design of our controller and illustrate how it uses our MI measures to guide the camera in a 3D scene to optimal views. 

\subsection{Active Regret Minimisation}
\label{sub:mizo_method}
Our aim in designing a method to manage regret is to ensure that the addition of variables is always contributive to assessing if the VLM is correct or not. The measure for quantifying these relationships is MI. In MI-ZO, this is implemented by ensuring the addition of information is performed in line with the feedback in the initial demonstrations and after each round of inputs (see Algorithm \hyperref[alg:algo1]{1}). The redundancies between a set of variables are reduced by forming these into a weighted mixture distribution and the subsequent novel application of a ZO algorithm. 

\begin{algorithm}[hbt!]
\caption{MI-ZO}
\label{alg:algo1}
\begin{algorithmic}
\STATE {\bfseries Input:} Sources $H_n$, system responses $Y$, sequence of views $V$
\STATE {\bfseries initialise:} policy $\pi$, parameters $\theta^{mix}$
\FOR{$t$ in $\tau$}
    \STATE $Mixture^t(X) \leftarrow \sum_{n=1}^{N} \pi_n(\theta^{mix};\phi_t)\, H_n(v_t)$
    \STATE $S_t \leftarrow Mixture^t(X)\,\cdot\,\Lambda^t$
    \STATE $MI_t \leftarrow MI(S_t, Y)$
    \STATE $\theta \leftarrow \arg\max_{\theta}\ MI\!\big(S_t(\theta),\,Y\big)$
    \STATE $MI'_t \leftarrow \frac{1}{t}\sum_{k=1}^{t} MI_k,\quad MI'_0:=0$
    \STATE $Regret_t(\theta) \leftarrow MI'_{t-1} - MI\!\big(S_t(\theta),\,Y\big)$
    \STATE $Loss_t(\theta) \leftarrow -\,MI\!\big(S_t(\theta),\,Y\big)$
    \STATE $\theta^{mix} \leftarrow \arg\min_{\theta}\, Regret_t(\theta) = \arg\min_{\theta}\, Loss_t(\theta)$
    \FOR{i = 1, \ldots, N}    
        \STATE $Loss_{t,i}(\theta_i) \leftarrow -\,MI\!\big(S_t(\theta_i),\,Y\big)$
        \STATE $\theta_i \leftarrow \arg\min_{\theta_i}\, Loss_{t,i}(\theta_i)$
    \ENDFOR
    \STATE Maximise $d(\mathfrak{b}_{\mathfrak{u}}$, $\mathfrak{b}_{\mathfrak{d}}$) in $H_n$
\ENDFOR
\STATE {\bfseries return} $MI'_\tau$
\end{algorithmic}
\end{algorithm}

\textbf{Definitions} We start with a rendered view $v$ of a 3D scene at time $t$. An \textbf{entropy source} $H$ is a score for a view  and one of multiple component distributions in a \textbf{mixture distribution} $Mixture^t(X)$. The mixture distribution is a probability distribution over multiple entropy sources. By setting weights, $Mixture^t(X)$ provides a single score $S_t$ indicating if the \textbf{VLM's responses} $Y$ are correct or not. A \textbf{policy} $\pi$ is a set of values for the mixing weights on the sources. Mixing weights are further updated by \textbf{scaling factors} $\Lambda$ for textual inputs. The aim is to maximise MI and capture in full the information capacity in the inputs. A movement in the direction of maximum MI is described as \textbf{positive}.

\textbf{Sources} We describe the inputs constituting the sources of entropy. A viewpoint render is converted to a colour space and a series of intervals is constructed over $n$ axes where $n\text{=}2$ for the CIELAB colour space and the single hue axis is selected for the HSV colour model (see our ablation on colour spaces in Subsection \ref{subsection:ablate_cs}). These intervals are adjacent and enable multivariate estimation with mixed variable types~\cite{hall1993estimation}. A Sobel operator is passed over grayscale object masks segmented from the image to estimate local edge density ($LED$). MI variant $GH\text{-}LED$ is the product of global hue ($GH$) and local edge density histograms. $GO\text{-}LED\text{-}OL$ is composed of global AB axes ($GO$) with an additional series at object level on LAB axes ($OL$). Scaling factors $\Lambda$ are computed over counts of noun phrases and descriptor terms in the textual input.

\textbf{Weighted Mixture Distribution} Our selection of a mixture distribution for $Mixture(X)$ avoids integration when combining histograms for sources and measures of $\Lambda$ linguistic inputs. Sources are combined at time $t$ with values for each component weight $\theta^{mix}$ set by $\pi$ summing to 1:
\begin{equation}
Mixture^t(X)=\sum_{h=1}^{N} \pi_h(\theta^{mix}; \phi_t) \, H_h.
\end{equation} 

\newlength{\thumbheight}
\setlength{\thumbheight}{13mm} 
\makeatletter
\renewcommand{\thumb}[1]{%
  \adjustbox{valign=m}{%
    \fboxsep=0pt\relax
    \fboxrule=0.3pt\relax
    \fbox{\includegraphics[height=\thumbheight,keepaspectratio]{#1}}%
  }%
}
\makeatother

\textbf{Optimising Feedback to Minimise Regret}
Assume the information of a constituent source is composed of units $\mathfrak{b}_u$ and $\mathfrak{b}_d$ in the set $\mathfrak{B}$ specified by the states $Up$ and $Down$. Orientation is defined as the contribution to MI where $Up$ is additive and $Down$ is reductive in relation to the product of the MI calculation. A reduction means that MI is failing to express a proportion of the information capacity in the inputs. Consider the similarities and differences in the visual sources in Table~\ref{table:infoset}. Our ZO algorithm is applied to eliminate the redundancies between these sources - in practical terms, to express only the differences that contribute to a correct assessment. At each step, signed changes to the mixing weights are initialised and only retained in the update if the running mean on MI increases. A regret minimisation process of this form is equivalent to an optimal mixture of unit states from each source. The goal is to identify a separating margin that minimises the loss on the distance of the closest units $\mathfrak{b}_{\mathfrak{u}}$ and $\mathfrak{b}_{\mathfrak{d}}$:
\begin{equation}
\text{Regret}
= \min_{\vec{w},\,bias}\,\{
    \max_{\mathfrak{b}_{\mathfrak{u}}\in\mathfrak{B},\;\mathfrak{b}_{\mathfrak{d}}\in\mathfrak{Bd}}%
    \bigl|\vec{w}^T\mathfrak{b}_{\mathfrak{u}} + bias - 
     (\vec{w}^T\mathfrak{b}_{\mathfrak{d}} + bias)\bigr|
  \}.
\end{equation}

The objective in our gradient-free optimisation step is to reduce the delta on losses for the current set of sources in relation to the mean over all sets accumulated to $i-1$ rounds. A finite difference approximates a set of parameters for each set of components $\phi$ in $Mixture(X)$. We minimise the finite difference between the two loss terms as follows 
\begin{multline}
f_t(\theta_t)
=
\left(
\frac{1}{t-1}\sum_{i=1}^{t-1} Loss(\hat{\theta}_{\phi_i}) - Loss(\hat{\theta}_{\phi_t})
\right)\,\vec{v}_t
\end{multline}
where $\vec{v}_t$ consists of multiple iterations $r$ $\{\vec{v}_{t,1},\ldots,\vec{v}_{t,r}\}$ averaged to get the single update $\vec{v}_t = \frac{1}{r}\sum_{j=1}^{r}\vec{v}_{t,j}$,

To demonstrate the ability of our algorithm to minimise regret at a rate that is constant with the quantity of feedback available, we provide scores for MI metrics on the UC-3DS-MI diagnostic in Table~\ref{table:infoset} in three levels of feedback setting. Note that performance improves in line with the number of variables added and only variants with active regret minimisation benefit from feedback. Theoretical analysis for our method is provided next and detailed in Appendix~\ref{sec:app_proofs}.

\subsection{Outline of the Theoretical Analysis}
We begin our theoretical analysis by defining multi-information MI and specifying the condition observed in the numerical analysis where - precluding access to advance information on all outcomes - the respective densities for uniform and complex inputs tend toward equivalent means. If the definition of Shannon entropy is adhered to, there is no guarantee of returning a positive value in all cases~\cite{cover2006elements} (see Lemma 1 and the subsequent proof in Appendix~\ref{sec:app_proofs}). 

Our aim in the remainder of the analysis is to describe the conditions where a function $g$ acts on an estimate of $MI(x, y)$ such that the reductive contribution of $\mathfrak{b}_{\mathfrak{d}}$ is bounded as $\leq$ $g\bigl(\vec{w}, \alpha, \gamma\bigr)$, where $\vec{w}$ is the weight vector normal to a separating hyperplane, $\alpha$ is the vector representing observed information, and $\gamma$ is the margin between additive $\mathfrak{b}_{\mathfrak{u}}$ and reductive $\mathfrak{b}_{\mathfrak{d}}$ units. The subsequent analysis proposes that if the hyperplane is selected and updated online using regret minimisation as defined in the previous subsection, then the negative contribution to the mutual information will remain bounded. 

We proceed after stating Theorem 1 and the initial definitions to prove the following:
\begin{customthm}[Function to maximise the margin]
For any $X$, an optimisation process $f(x) = \vec{w}^T x + b$ exists to solve the minimax form of deriving the margin using the hyperplane in $\mathbb{R}^D$ that maximises the distance between $\mathfrak{b}_{\mathfrak{u}}$ and $\mathfrak{b}_{\mathfrak{d}}$:
\begin{equation}
\max_{\mathfrak{b}_{\mathfrak{d}} \in \mathfrak{Bd}} \min_{\mathfrak{b}_{\mathfrak{u}} \in \mathfrak{Bu}} f(\mathfrak{b}_{\mathfrak{u}}, \mathfrak{b}_{\mathfrak{d}}) = \min_{\mathfrak{b}_{\mathfrak{d}} \in \mathfrak{Bd}} \max_{\mathfrak{b}_{\mathfrak{u}} \in \mathfrak{Bu}} f(\mathfrak{b}_{\mathfrak{d}}, \mathfrak{b}_{\mathfrak{u}}).
\end{equation}
\end{customthm}
Our proof defines an optimisation function $f(x) = \vec{w}^T x + b$ that characterises the separating hyperplane in $\mathbb{R}^D$. By converting the primal optimisation problem into a dual form~\cite{rockafellar1974conjugate}, we obtain an expression for the margin $\gamma = \frac{2}{\|\vec{w}\|}$. The analysis is then extended to the online setting where the solution $\omega_t$ is a projection to the nearest point in the convex set of variable values $C$, which is updated by information $\alpha_t$ and the gradient $\nabla$ of the loss $L$ scaled by step size $\eta_t$:
\begin{equation}
\omega^t \leftarrow \operatorname{proj}_C\Bigl( \omega^{t-1} - \eta_t \cdot \alpha_t \cdot \nabla_{\omega} L(\omega^{t-1}) \Bigr).
\end{equation}
Regret is dependent on the interaction of the margin and the inner product $\langle \alpha, \vec{w} \rangle$. Analysis of this relation confirms that the selected hyperplane determined by our function maximises the separation between unit types and guarantees that the negative impact of $\mathfrak{b}_{\mathfrak{d}}$ is bounded.

\begin{table}[hbt!]
\centering
\begin{tabularx}{\linewidth}{@{} l Y S S S @{}}
\toprule
\textbf{Group} & \textbf{Sources} & \multicolumn{3}{c}{\textbf{Feedback setting}} \\
 &  & \textbf{Full} & \textbf{50\%} & \textbf{20\%} \\
\midrule

\multicolumn{5}{@{}l}{\textbf{Single visual input}} \\
\addlinespace[2pt]
\hspace{1mm} - OL
& \thumbstrip{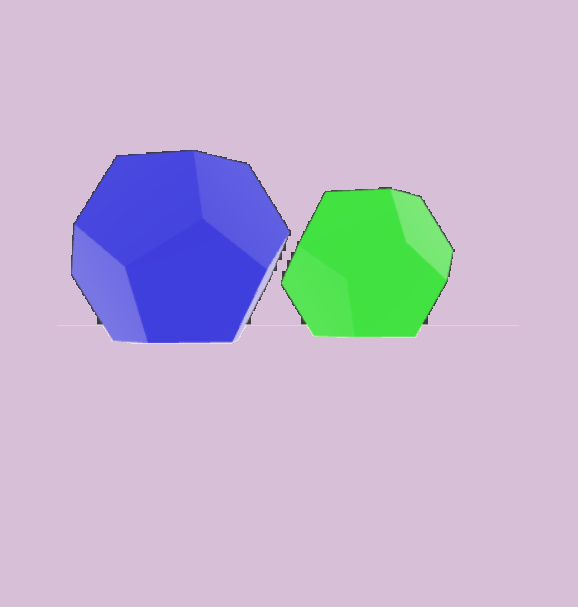}
& 0.37 & 0.37 & 0.37 \\

\addlinespace[2pt]
\hspace{1mm} - LED
& \thumbstrip{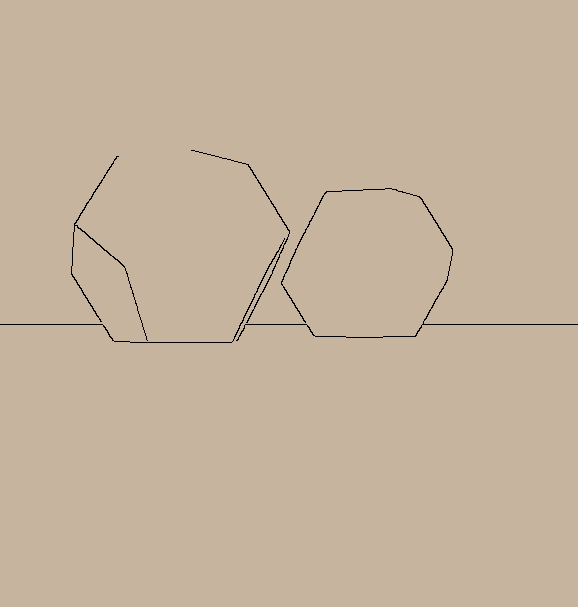}
& 0.42 & 0.42 & 0.42 \\

\addlinespace[2pt]
\hspace{1mm} - GO
& \thumbstrip{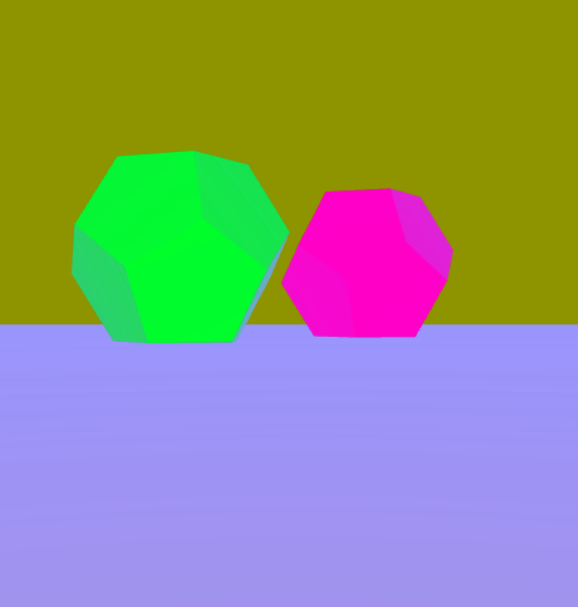}
& 0.46 & 0.44 & 0.45 \\
\cmidrule{1-5}

\multicolumn{5}{@{}l}{\textbf{Multivariate}} \\
\addlinespace[2pt]
\hspace{1mm} - GO-LED-OL (no ar)
& \thumbstrip{figs/scn_022_zm_4_front_score_global_lab.png,figs/scn_022_zm_4_front_led_hist.png,figs/scn_022_zm_4_front_ol_repr.png}
& 0.38 & 0.34 & 0.38 \\

\addlinespace[2pt]
\hspace{1mm} - GO-OL\textsubscript{ar}
& \thumbstrip{figs/scn_022_zm_4_front_score_global_lab.png,figs/scn_022_zm_4_front_ol_repr.png}
& 0.76 & 0.65 & 0.60 \\

\addlinespace[2pt]
\hspace{1mm} - GO-LED\textsubscript{ar}
& \thumbstrip{figs/scn_022_zm_4_front_score_global_lab.png,figs/scn_022_zm_4_front_led_hist.png}
& 0.78 & 0.68 & 0.60 \\
\cmidrule{1-5}

\multicolumn{5}{@{}l}{\textbf{Proposed metric}} \\
\addlinespace[2pt]
\hspace{1mm} - GO-LED-OL\textsubscript{ar}
& \thumbstrip{figs/scn_022_zm_4_front_score_global_lab.png,figs/scn_022_zm_4_front_led_hist.png,figs/scn_022_zm_4_front_ol_repr.png}
& 0.83 & 0.74 & 0.69 \\
\bottomrule
\end{tabularx}
\caption{Our active regret ($ar$) measure $GO\text{-}LED\text{-}OL\textsubscript{ar}$ is assessed against metrics with less inputs and a variant calculated with no $ar$. Scores for feedback setting demonstrate the impact when metrics with $ar$ are provided continuous feedback on $y$ correctness labels. Decomposition over variables indicates more inputs contribute to the sensitivity of multivariate metrics with $ar$ and demonstrate the benefits of minimising regret in securing additive effects from new sources when $n>2$. Feedback on correctness provides no guaranteed improvements in univariate or multivariate measures without an active process for minimising regret.}
\label{table:infoset}
\end{table}

\subsection{Controller}
Our controller consists of a chain of functions to predict camera actions $\mathfrak{a}$ for $n>1$ conversation rounds. We combine sample efficient data filters to exploit the information capacity in MI measurements estimated by the MI-ZO algorithm detailed below. A Central Unit predicts errors and confidence scores on axis-level traces and scores from two Component Models. View-level data updates a low dimensional representation of the 3D space in the form of an interaction matrix updated with a strong product at each step. A full specification of the controller is presented in Appendix~\ref{sec:app_method}.


\section{Evaluation on Benchmarks}
\label{experiments}

\subsection{Experiment Design}
We implement a controller framework for empirical testing of multi-information and optimal control methods with open source VLMs. Evaluation on closed systems is precluded to limit the likelihood of data contamination invalidating subsequent benchmarking~\cite{xu2024benchmark}. Our three benchmarks are designed to identify methods that predict the optimal sequence of viewpoints for appraising 3D scenes. Experiments in both benchmarks present a video-to-text system with a render of a viewpoint of the 3D scene and a description. Input pairs form turns in a conversation where state is accumulated to present the VLM with a complete view of the 3D scene. At the end of the sequence, the VLM assesses the state in relation to the scene and returns a boolean label $\{True, False\}$. 

An initial measurement round uses a camera initialised at the nearest of four distances to the scene and rotated to cardinal viewpoints by a default sequence of actions. In the correction round, the controller predicts camera actions based on data from the prior round and data for demonstrations on $n$ scenes run prior to the start of the evaluation (see Appendix~\ref{sec:app_experiments}). Demonstrations provide the controller with feedback in the form of the prediction errors made by the VLM.

\subsection{Experiment Settings}
\textbf{Systems:} Testing is performed with Video-LLaMA-13B~\cite{zhang-etal-2023-video} and Chat-UniVi-13B~\cite{jin2024chat} as the VLM baselines. Video-LLaMA-13B runs with Llama 2~\cite{touvron2023llama} for decoding and the number of frames defined by camera action count. Conversation hyperparameter settings are as follows: temperature $= 1.0$, beams $= 2$, repetition penalty $= 1.1$. Chat-UniVi-13B uses Vicuna-13B v1.5~\cite{NEURIPS2023_91f18a12} and hyperparameter settings: temperature $= 0.2$, beams $= 1$. \textbf{Methods Assessed} Controller performance using polynomial regression is assessed with standard MI and MI-ZO derived $MI\text{-}ar$ variants. Tests are conducted with two canonical control methods~\cite{johnson2005pid,ljung1979asymptotic}, a linear layer optimised with stochastic gradient descent (SGD), and a Radial Basis Function (RBF) network~\cite{orr1996introduction}. \textbf{Camera:} A single camera is added to the 3D scene at $x-, y-, -z-$ coordinates $(0, 0.5, -40)$ and pointed at the origin in all viewpoints. Initial viewpoint is front and a single rotation about the $x-$axis places the camera equal to $\pm 90^\circ$ to the left or right from its current position. Rotation about the $y-$axis is $\pm 45^\circ$ and constrained to the front or back as starting positions. Zoom operations are rotations on the $z-$axis performed in increments of $\pm 5$ in the range $[-10, -25]$ forward or backward to the origin. \textbf{Metrics:} We use balanced error rate (BER) to assess the performance on multi-object scenes for scientific analysis in our first benchmark~\cite{Zhao2020Conditional}. A mean is calculated over the False Positive Rate (FPR) and False Negative Rate (FNR) as $\frac12(FPR + FNR)$ where
\[
\mathrm{FPR} = \frac{\mathrm{FP}}{\mathrm{FP} + \mathrm{TN}}, 
\quad
\mathrm{FNR} = \frac{\mathrm{FN}}{\mathrm{FN} + \mathrm{TP}}.
\]

Test performance on our two evaluations on virtual environments is measured with mean accuracy $Acc$ on VLM decisions $Dec$ in the scene summary question $SQ$ and calculated over all scenes: 
\[
\text{Acc}_{\text{SQ}} = \left( 100 \times \sum_{i=1}^N \frac{\text{Dec}_{\text{Correct},i}^{\text{SQ}}}{\text{Dec}_{\text{Correct},i}^{\text{SQ}} + \text{Dec}_{\text{Incorrect},i}^{\text{SQ}}} \right).
\]

\begin{table}[hbt!]
\begin{center}
\begin{threeparttable}
\setlength{\tabcolsep}{0pt}
\begin{adjustbox}{max width=\columnwidth}
\begin{tabular}{l l l l l}
\hline
                                & \multicolumn{2}{l }{\textbf{Video-LLaMA-13B}} & \multicolumn{2}{l}{\textbf{Chat-UniVi-13B}} \\ \cline{2-5}
                                & \textbf{BER@8$\downarrow$}    & \textbf{$\Delta$ on R1}   & \textbf{BER@8$\downarrow$}    & \textbf{$\Delta$ on R1}   \\ \hline
\textbf{VLM (no control)}                   & 59.8 & -1.3 & 61.3 & -1.7   \\
\textbf{PID}                   & 57.2 & -4.6 & 58.5 & -3.6   \\ 
\textbf{Linear+SGD}          & 55.8 & -6.0 & 56.1 & -6.7   \\ 
\textbf{RBF Network}                   & 55.1 & -5.8 & 57.3 & -5.3   \\ 
\textbf{Extended Kalman}       & 54.5 & -6.5 & 55.5 & -6.1   \\
\cmidrule{1-5}
\textbf{Poly+ZO+MI (ours)} & 53.8 & -7.1 & 54.7 & -7.0   \\ 
\textbf{+GO-LED-OL}            & 48.4 & -13.9 & 50.3 & -12.7   \\ 
\textbf{+GH-LED}                & 47.7 & -14.0 & 49.2 & -13.0   \\ 
\textbf{+GO-LED-OL\_ar}            & 41.3 & -18.9 & 43.1 & -18.4   \\ 
\textbf{+GH-LED\_ar}            & \textbf{40.4} & -21.0 & \textbf{42.2} & -20.4   \\ \hline
\end{tabular}
\end{adjustbox}
\caption{Results for our benchmark on reasoning over visual properties to perform scientific analysis. The new GeoProperties-3DS benchmark compares control methods to variants of our $GO\text{-}LED\text{-}OL$ and $GH\text{-}LED$ metrics in reducing errors when reasoning over the properties of objects in 3D scenes. Versions incorporating active regret minimisation $MI\text{-}ar$ close to double the performance of a VLM with no control. Performance is measured using balanced error rate (BER).}
\label{table:geo_ber}
\end{threeparttable}
\end{center}
\end{table}

\subsection{Cross-Modal Reasoning on Visual Properties}
\textbf{Benchmark:} We introduce GeoProperties-3DS to evaluate control methods in reducing error rates of VLMs on 3D scenes for geological analysis in planetary science. A collection consists of five scenes with a single textual summary where only one scene matches the description. Scenes with a small set of rocks are extracted from large 3D models \cite{bell2022geological,owl_creek}. These models are reconstructions from visual data collected by scientific instruments installed on Mars Rovers~\cite{bell2021mars, paar2023three}. The textual summary refers to the physical properties of formations and outcrops. The objective for camera control algorithms is to assist the VLM in maximising TPR and minimising FPR. \textbf{Results:} Our controller with $GO\text{-}LED\text{-}OL\textsubscript{ar}$ reduces the BER on the GeoProperties-3DS benchmark (see the results in Table \ref{table:geo_ber}). Performance on all runs where Poly+ZO+MI is enhanced with our MI measurements exceeds the reduction in errors attained by the highest ranked standard control method. As detailed in Appendix~\ref{sec:app_results}, variations in results between runs relate to noise in VLM responses and stochastic operations in calculations performed by methods.

\begingroup
\renewcommand{\arraystretch}{1.2}
\begin{table}[ht]
\centering
\begin{tabularx}{\linewidth}{l Y S S}
\toprule
\textbf{Group} & \textbf{Sources} & \textbf{BER@8$\downarrow$} & \textbf{$\Delta$ on R1} \\ \hline
\addlinespace[2pt]
J\textsubscript{z}a\textsubscript{z}b\textsubscript{z} & \thumbstrip{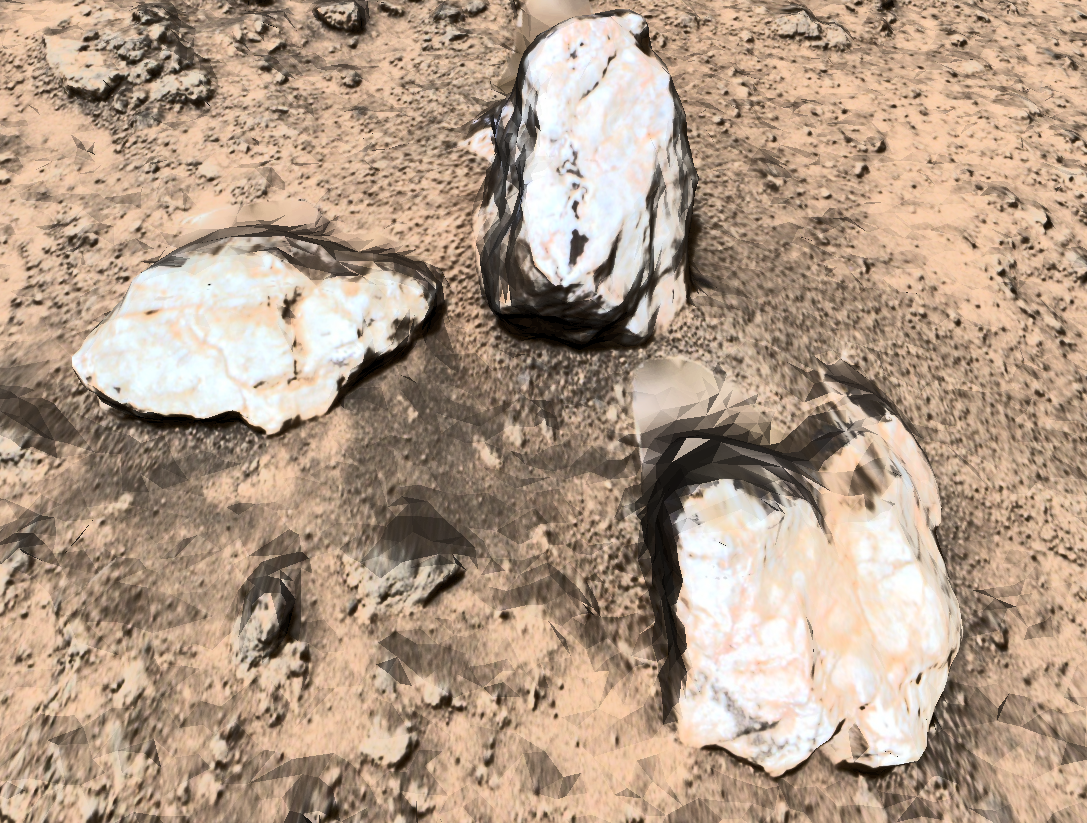,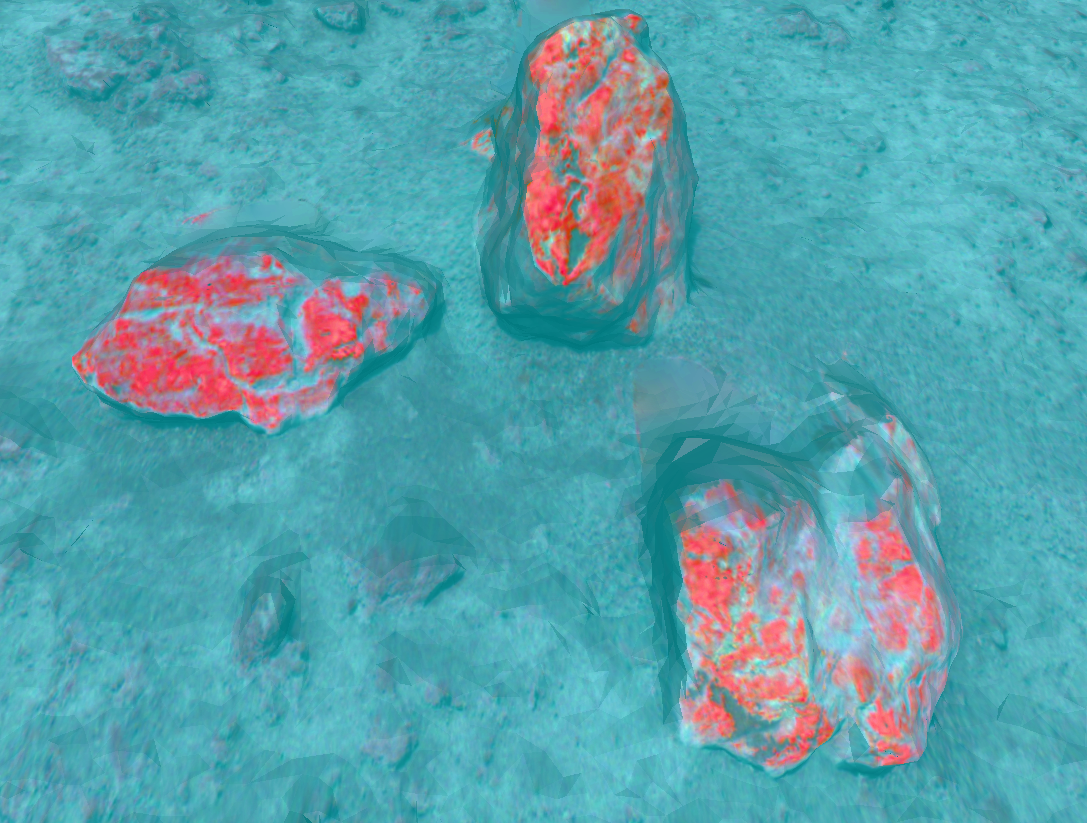} & 51.6 & -10.8 \\
\addlinespace[2pt]
Oklab & \thumbstrip{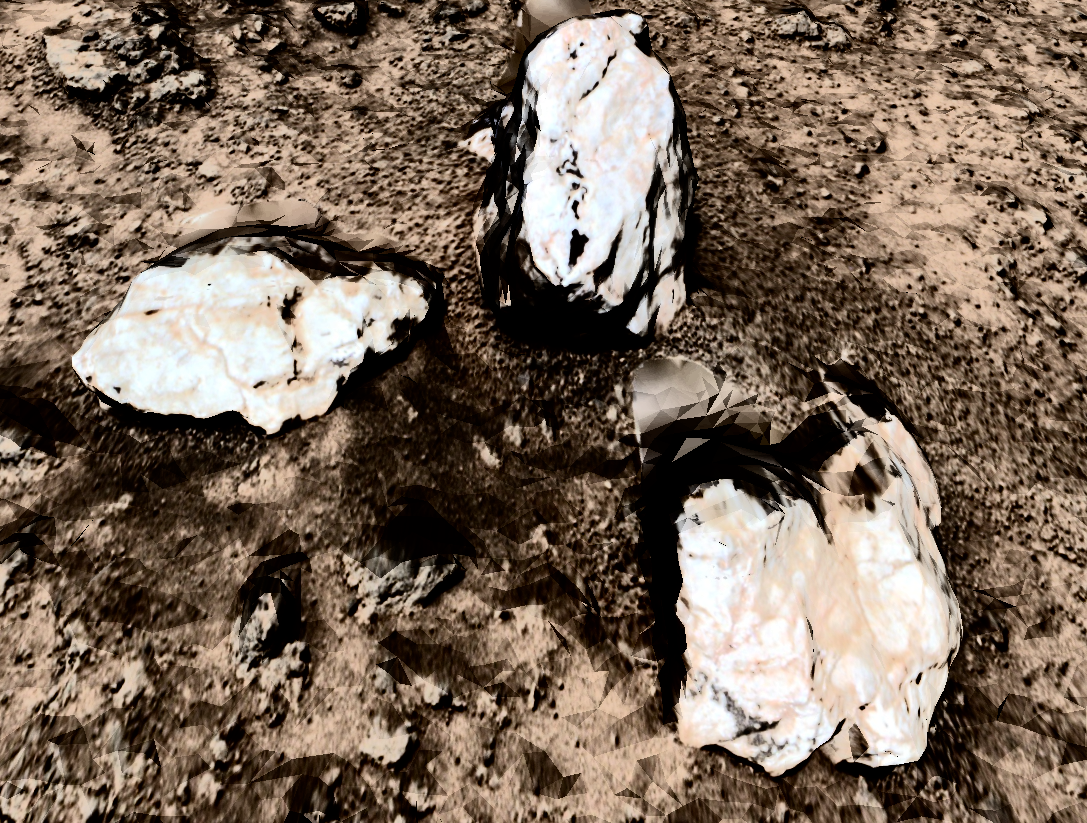,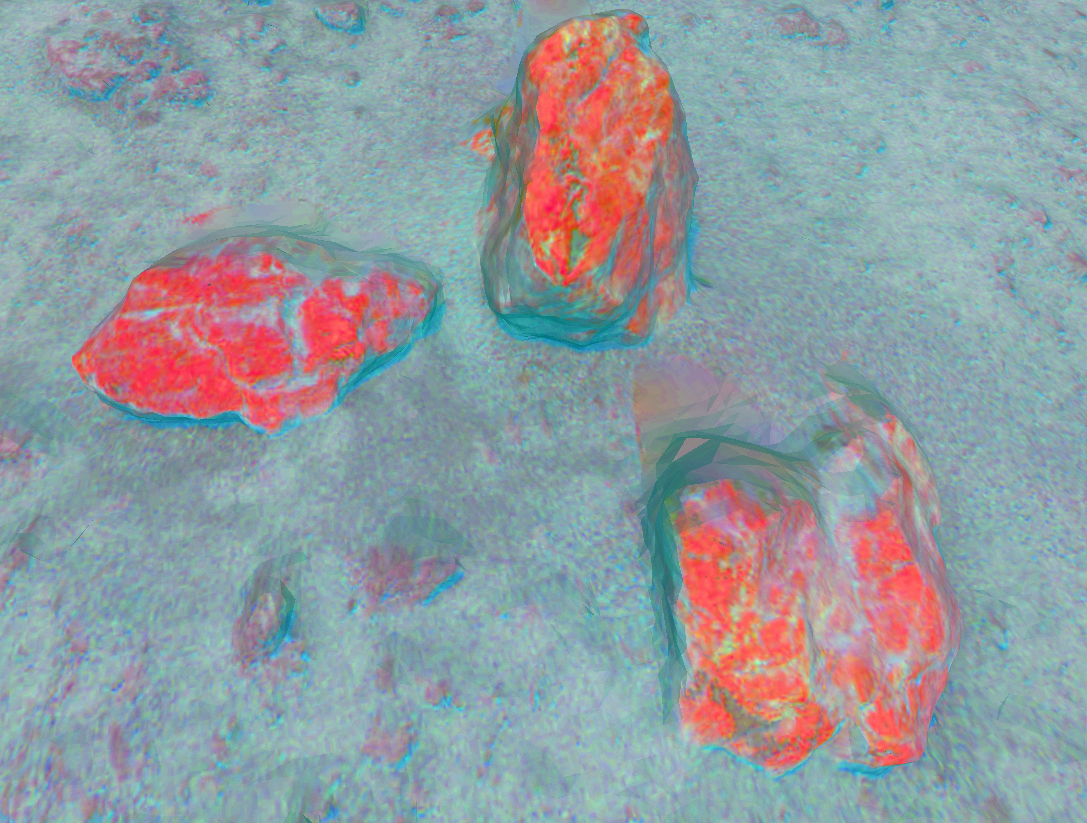} & 49.3 & -12.6 \\
\addlinespace[2pt]
GO-LED-OL\textsubscript{ar}\text{+}a & \thumbstrip{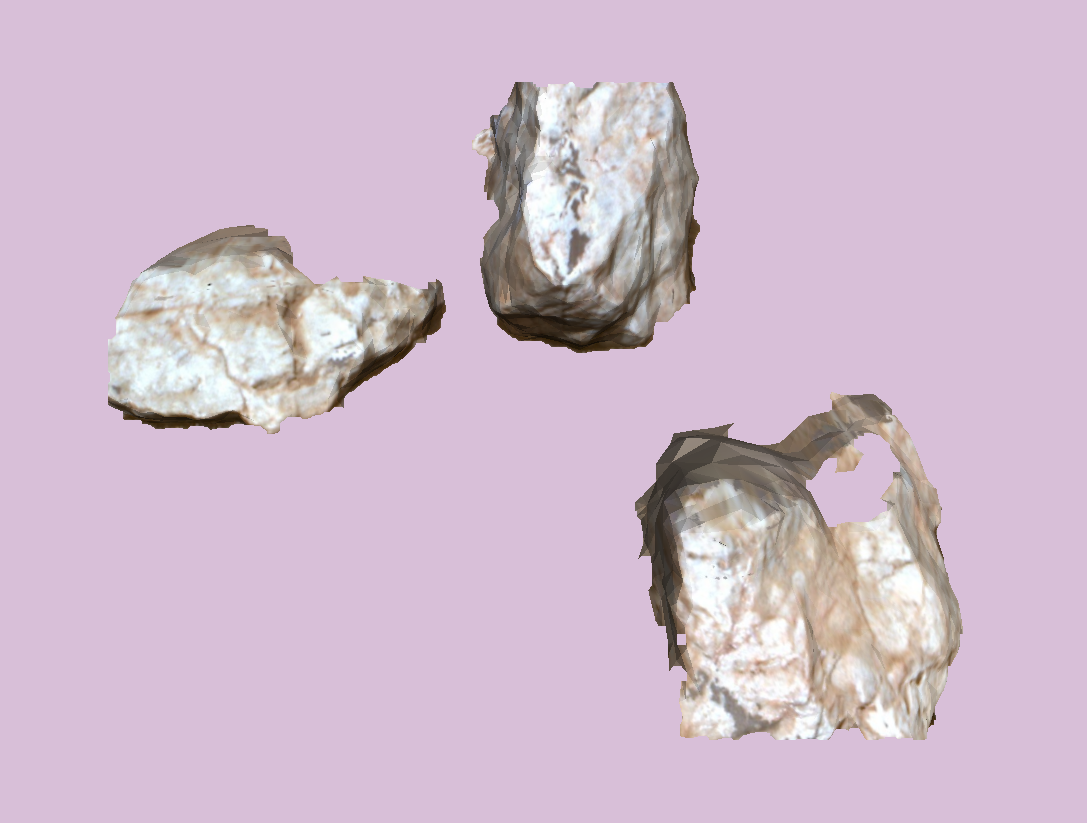,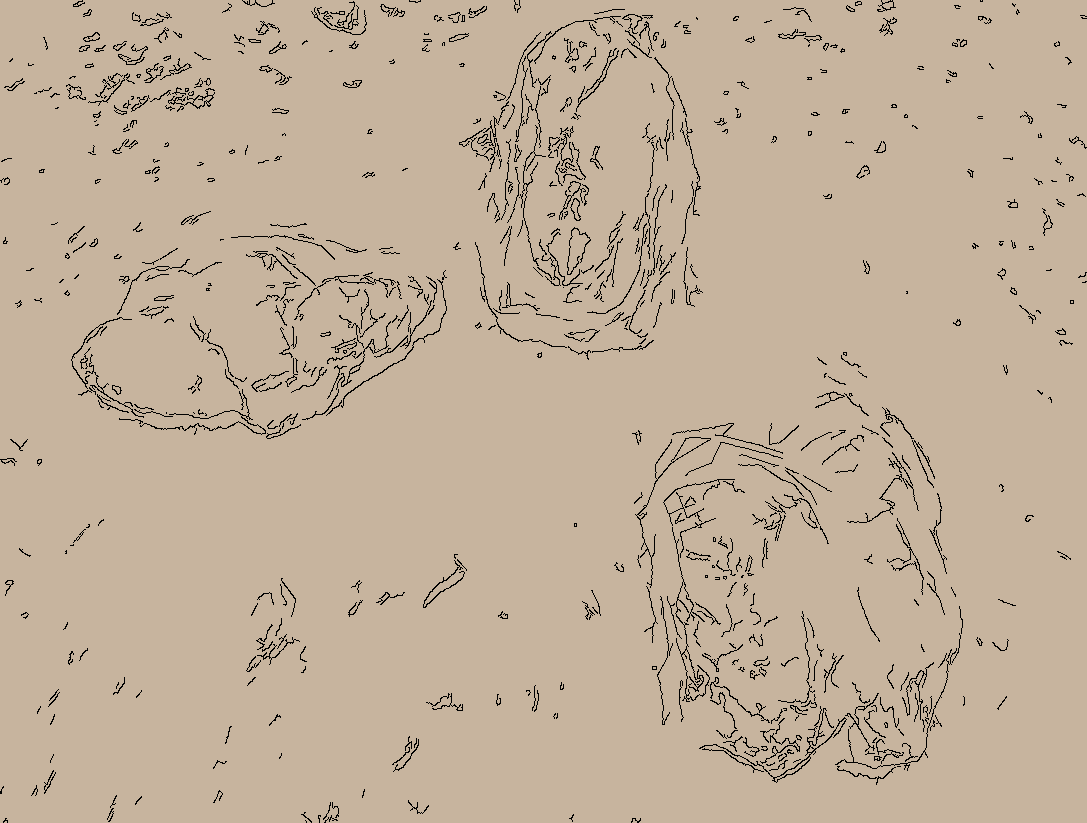,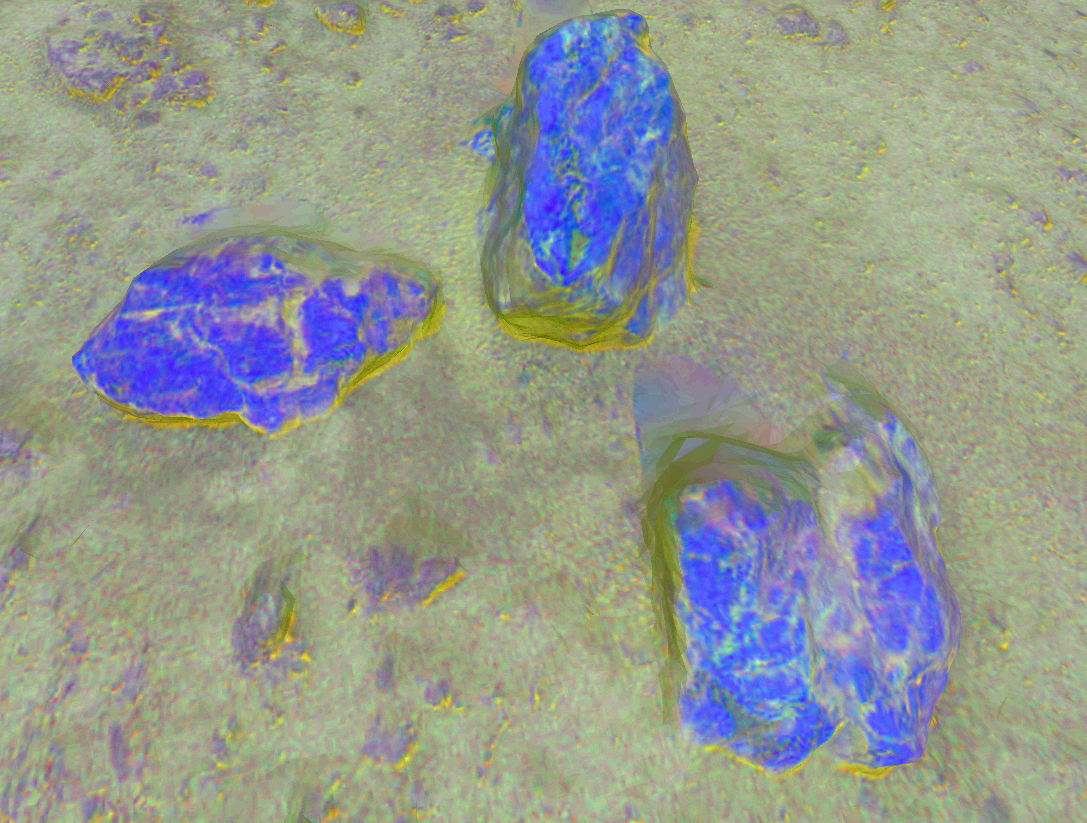,figs/wr_jzazbz_hist.png} & 40.3 & -21.8 \\
\addlinespace[2pt]
GO-LED-OL\textsubscript{ar}\text{+}b & \thumbstrip{figs/wr_ol_repr.png,figs/wr_led_hist.png,figs/wr_global_lab.png,figs/wr_jzazbz_hist.png,figs/wr_jzazbz_comp.png} & 40.3 & -21.7 \\
\addlinespace[2pt]
GO-LED-OL\textsubscript{ar}\text{+}c & \thumbstrip{figs/wr_ol_repr.png,figs/wr_led_hist.png,figs/wr_global_lab.png,figs/wr_oklab_hist.png} & 40.3 & -21.3 \\
\addlinespace[2pt]
GO-LED-OL\textsubscript{ar}\text{+}d & \thumbstrip{figs/wr_ol_repr.png,figs/wr_led_hist.png,figs/wr_global_lab.png,figs/wr_oklab_hist.png,figs/wr_oklab_comp.png} & 40.2 & -21.2 \\
\addlinespace[2pt]
\bottomrule
\end{tabularx}
\caption{Performance of Video-LLaMA-13B with our Poly+ZO+MI controller on GeoProperties-3DS benchmark and sources from additional colour spaces J\textsubscript{z}a\textsubscript{z}b\textsubscript{z} and Oklab. The best performing measure $GO\text{-}LED\text{-}OL\textsubscript{ar}$ is extended to a) one source or b) two sources from J\textsubscript{z}a\textsubscript{z}b\textsubscript{z} - and c) one source or d) two sources from Oklab. In total then a) and c) have five sources and b) and d) have six sources. The error rate of the VLM is not impacted when increasing the number sources beyond 5.}
\label{table:col_spaces}
\end{table}
\endgroup

\textbf{Ablation on Colour Spaces:}
\label{subsection:ablate_cs}
Our selection of HSV and CIELAB is grounded on comparisons with four candidate colour spaces. The performance for our controller with J\textsubscript{z}a\textsubscript{z}b\textsubscript{z} and Oklab underperforms these candidates. Results for runs on GeoProperties-3DS with four or five features extracted from multiple colour spaces suggest there is a ceiling to gains from adding new sources (see Table~\ref{table:col_spaces}). Eye-level comparisons indicate that the increase in information extracted for a single sample peaks as the differences between inputs reduces.

\begin{table*}[hbt!]
\begin{center}
\begin{threeparttable}
\begin{tabular}{l l l l l l l l l l}
\hline
                                &                      & \multicolumn{4}{l }{\textbf{Video-LLaMA-13B}} & \multicolumn{4}{l}{\textbf{Chat-UniVi-13B}} \\ \cline{3-10}
                                &                      & \textbf{Acc@5}    & \textbf{$\Delta$ on R1}    & \textbf{Acc@8}    & \textbf{$\Delta$ on R1}   & \textbf{Acc@5}   & \textbf{$\Delta$ on R1}    & \textbf{Acc@8}    & \textbf{$\Delta$ on R1}   \\ \hline
\textbf{VLM (no control)}                   &                      & 19.8 & 0.3 & 25.1 & 0.5 & 14.8 & 0.3 & 19.6 & 0.3   \\
\textbf{PID}                   &                      & 21.1 & 2.1 & 32.8 & 7.3 & 18.1 & 1.6 & 28.1 & 6.5   \\ 
\textbf{RBF Network}                   &                      & 22.2 & 2.6 & 33.0 & 7.3 & 18.9 & 0.6 & 28.5 & 5.5   \\
\textbf{Linear + SGD}          &                      & 22.8 & 2.9 & 30.1 & 6.5 & 17.5 & 2.4 & 26.0 & 5.9   \\ 
\textbf{Extended Kalman}       &                      & 23.6 & 3.6 & 34.2 & 8.5 & 21.4 & 2.9 & 27.1 & 7.2   \\
\cmidrule{1-10}
\textbf{Poly+ZO+MI (ours)} &                      & 23.6 & 3.1 & 35.8 & 9.2 & 20.8 & 2.6 & 27.7 & 7.8   \\ 
\textbf{+GO-LED-OL}            &                      & 25.9 & 4.2 & 37.1 & 12.4 & 18.6 & 3.4 & 29.1 & 9.3   \\ 
\textbf{+GH-LED}               &                      & 26.2 & 4.9 & 39.6 & 13.9 & 22.9 & 4.1 & 33.3 & 10.8   \\ 
\textbf{+GH-LED\_ar}           &                      & 30.9 & 10.3 & \textbf{53.3} & 27.2 & 26.8 & 7.8 & \textbf{44.4} & 21.5   \\ 
\textbf{+GO-LED-OL\_ar}        &                      & \textbf{31.4} & 14.3 & 44.5 & 19.7 & \textbf{27.4} & 11.4 & 40.3 & 15.3   \\ \hline
\end{tabular}
\caption{Analysis of methods on prioritising viewpoints in feature identification given a restriction on the number of actions. The objective of the VLM system in our FeatureID-3DS benchmark is to match a summary of the 3D scene by discriminating on descriptions that describe an object feature visible only from selected viewpoints. Accuracy is reported by camera action count in the correction round and the delta $\Delta$ on the measurement round. Action counts are run with counts of 8 actions and a reduced budget of 5 actions. We present above mean values over 10 runs and include reporting on variance in Appendix~\ref{sec:app_results}.}
\label{table:featureid}
\end{threeparttable}
\end{center}
\end{table*}

\subsection{Feature Identification on Short Runs}
\textbf{Benchmark:} In our first evaluation on virtual environments, we introduce the FeatureID-3DS benchmark to test feature identification over multiple objects. A VLM provides a boolean response on world features (eg ``ladder'', ``doorway'') on presentation of a language description. A scene summary containing the features is presented on the final turn. Applications with low latency between feedback and generation rely on efficient estimation~\cite{hofmann2023record}. In order to assess control methods on prioritising high-information viewpoints, camera action count is restricted to $5$ in the correction round. Counts of camera actions are the main factor in efficiency as each action entails a round of inputs and responses from the VLM. This is demonstrated by wall-clock time results in Appendix~\ref{sec:app_results} indicating that time increases at a rate that is linear to the number of actions. \textbf{Results:} Our $GO\text{-}LED\text{-}OL\textsubscript{ar}$ metric assists the Poly+ZO+MI controller to prioritise viewpoints that reduce VLM errors with fewer actions (see scores for Acc@5 in Table~\ref{table:featureid}). Analysis of viewpoint replacements in Appendix~\ref{sec:app_results} indicates that views displaying features with strong visual prominence are prioritised.

\begin{table}[hbt!]
\begin{center}
\begin{threeparttable}
\setlength{\tabcolsep}{0pt}
\begin{adjustbox}{max width=\columnwidth}
\begin{tabular}{l l l l l}
\hline
                                & \multicolumn{2}{l }{\textbf{Video-LLaMA-13B}} & \multicolumn{2}{l}{\textbf{Chat-UniVi-13B}} \\ \cline{2-5}
                                & \textbf{Acc@8}    & \textbf{$\Delta$ on R1}   & \textbf{Acc@8}    & \textbf{$\Delta$ on R1}   \\ \hline
\textbf{VLM (no control)}                   & 20.4 & 0.6 & 18.1 & 0.3   \\
\textbf{PID}                   & 23.8 & 5.0 & 22.2 & 4.4   \\ 
\textbf{RBF Network}                   & 24.6 & 5.3 & 21.1 & 3.4   \\ 
\textbf{Extended Kalman}       & 25.0 & 7.5 & 20.4 & 6.6   \\
\textbf{Linear+SGD}          & 27.4 & 10.3 & 22.4 & 9.8   \\ 
\cmidrule{1-5}
\textbf{Poly+ZO+MI (ours)} & 24.5 & 6.0 & 20.4 & 5.7   \\ 
\textbf{+GO-LED-OL}            & 28.2 & 8.7 & 22.8 & 8.0   \\ 
\textbf{+GH-LED}                & 32.7 & 13.4 & 31.1 & 12.9   \\ 
\textbf{+GO-LED-OL\_ar}            & 34.7 & 14.7 & 32.7 & 13.3   \\ 
\textbf{+GH-LED\_ar}            & \textbf{39.3} & 21.0 & \textbf{35.3} & 20.2   \\ \hline
\end{tabular}
\end{adjustbox}
\caption{Assessment on object occlusion of standard control and SGD-optimised neural methods to benchmark our controller with multi-information measures. In the PartialView-3DS benchmark, a vertical partition in the scene occludes one of the objects depending on camera positions. Mean accuracy is reported for the correction round and the delta $\Delta$ on the measurement round. Rounds consist of 8 camera actions and mean values are reported over 10 runs (see Appendix~\ref{sec:app_results} for details on variance).}
\label{table:occlusion}
\end{threeparttable}
\end{center}
\end{table}

\subsection{Object Occlusion}
\textbf{Benchmark:} Our second assessment on reasoning over virtual environments focuses on offsetting object occlusions. Scenes in the PartialView-3DS benchmark consist of two objects located on opposite sides of a partition to test the adaptation of control methods when there is full or partial occlusion of one of the objects at every viewpoint. A matching description is selected from five descriptions presented on each turn. At the end of the round, the system uses the information collected to select the correct summary of the scene. Camera action count in both rounds is set to $8$. \textbf{Results:} Our light and derivative-free Poly+ZO+MI controller using multi-information with active regret minimisation (see variants with $ar$ in Table \ref{table:occlusion}) adapts camera actions to return views that improve VLM performance and outperforms other standard control algorithms, an RBF network, and a linear layer optimised with SGD. Variance over runs for all methods are minor (see additional results from the experiments included in Appendix~\ref{sec:app_results}).


\section{Conclusion} 

In this paper, we propose a novel multi-information estimation method and efficient derivative-free control to predict camera actions that provide optimal sequences of views on 3D scenes. This method improves the performance of VLMs trained on 2D visual data when performing cross-modal tasks on multi-object 3D scenes. Numerical and theoretical analysis provides the basis for obtaining the first application of multivariate mutual information estimation to enhance the performance of VLM systems on empirical 3D tasks. As part of this research, we design and implement a framework for evaluating control and information theoretic measures, design a set of scenes to illustrate the impact of minimising regret when selecting inputs for calculating multi-information metrics, and present three novel cross-modal benchmarks on 3D multi-object scenes.

\section{Acknowledgments}

This publication is supported by the Digital Visual Studies program at the University of Zurich and funded by the Max Planck Society. RS received funding by the Swiss National Science Foundation (project MUTAMUR; no.\ 176727).


{\footnotesize
\bibliography{main.bib}}

\newpage
\appendix
\begin{appendices}
\twocolumn
\begin{center}
\textbf{\huge Appendices}
\end{center}
\vspace{0.25in}



\section*{List of Appendices}
\begin{enumerate}
    \item \nameref{sec:app_notation}
    \item \nameref{sec:app_data}
    \item \nameref{sec:app_experiments}
    \item \nameref{sec:app_results}
    \item \nameref{sec:app_numerical}
    \item \nameref{sec:app_method}
    \item \nameref{sec:app_proofs}
\end{enumerate}



\section{Notation Used in the Paper}
\label{sec:app_notation}

Notations are defined here for fast reference. 

\begin{table}[hbt!]  
\centering            
\caption{Reference List for Standard Notation.}
\label{table:notation}
\begin{tabular}{l l}  
\toprule
\textbf{Notation}    & \textbf{Usage in this paper} \\ 
\midrule  
$\alpha$ & Observed information \\
$\mathbb{A}$ & Interaction matrix \\
$\mathfrak{B}$ & Constituent units of entropy sources \\
$\mathfrak{b}_{\mathfrak{d}}$ & Down unit \\
$\mathfrak{b}_{\mathfrak{u}}$ & Up unit \\
$Cap$ & Information capacity of a channel \\
$\mathcal{D}$ & Dataset \\
$\Delta$ & Difference \\
$Dim$ & Dimension \\
$\eta$ & Step size \\
$\gamma$ & Margin \\ 
$H$ & Entropy \\
$\mathbb{M}$ & Mutual information \\
$MI$ & Multi-information \\
$MI\textsubscript{ar}$  & MI with active regret minimisation \\
$\omega$ & A computational solution \\
$\pi$ & Policy \\
$v$ & View \\
$\vec{w}$ & Weight vector of a hyperplane \\
$x$-, $y$-, $z$-axis & Axes \\
$x$, $y$, $z$ & Coordinates \\
\bottomrule
\end{tabular}
\end{table}

\section{Specifications for Benchmarks and Numerical Analysis}
\label{sec:app_data}

\subsection{Benchmarks}

Cross-modal benchmarks for VLMs where visual inputs are 3D inputs include problems with individual objects~\cite{xue2024ulip} and aligning inputs to improve 3D shape understanding~\cite{liu2024openshape}. Search algorithms are proposed by \cite{voigt-etal-2023-paparazzi} to retrieve optimal viewpoints of 3D objects to improve a 2D CLIP model (see Figure \ref{fig:cam_vps}). We propose benchmarks that focus on reducing errors when reasoning over visual properties, understanding object features with limited views, and handling view occlusions to improve reasoning by VLMs trained on 2D visual inputs on cross-modal tasks with 3D multi-object scenes. 

\begin{figure}[hbt!]
  \centering
  \makebox[\linewidth][c]{
    \includegraphics[scale=0.50]{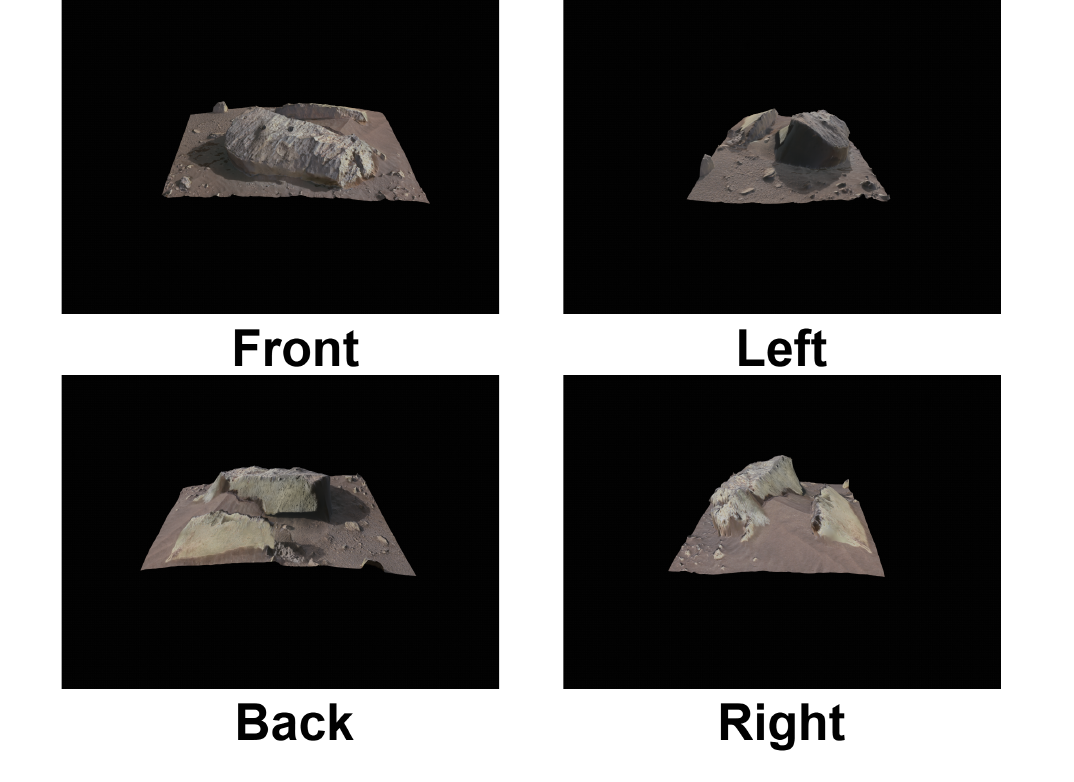}}
  \caption[Default Viewpoints for In‑scene Camera]{Systems trained on 2D data reason over 3D scenes using a set of viewpoints from the in‑scene camera. The objective for camera control methods is to predict the sequence of viewpoints with the highest likelihood of returning a correct assessment of the scene by the VLM. Scenes in all our benchmarks contain multiple objects from a single class.}
  \label{fig:cam_vps}
\end{figure}

\begin{figure}[hbt!]
  \centering
  \makebox[\linewidth][c]{\includegraphics[scale=0.50]{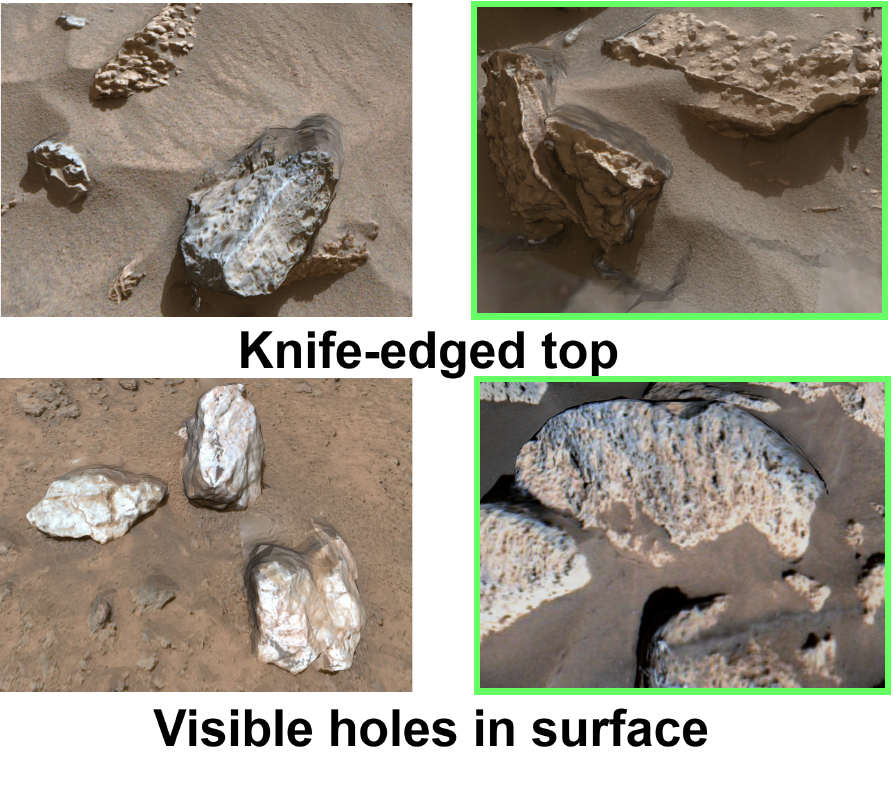}}
  \caption[Properties of Samples from GeoProperties-3DS]{Samples from our GeoProperties-3DS benchmark with close-ups of properties correctly identifying a single object in the scene. Descriptions refer to the largest rock or boulder and match scenes on the right. VLMs return false positives and this weakness is identified by presenting the same descriptions for the samples on the left where the property does not apply to objects in view.}
  \label{fig:geoproperties}
\end{figure}

Our GeoProperties-3DS benchmark consists of sets of scenes that we extract from 3D mesh models of rocks, regolith, and other geological features (see Figure \ref{fig:geoproperties}) generated from observations performed as part of NASA Mars 2020 missions \cite{bell2022geological}. Each scene contains a small set of objects. A total of 6 collections are used in our experiments with each group consisting of 5 individual scenes. We prepare textual descriptions of the objects in scenes describing surface features visible from a subset of viewpoints. The description matches a single member of the 5 scenes in each collection. Our benchmark is designed to assess methods that reduce the likelihood of false positives by 2D VLMs.  We extract scenes from 3D models of the Raton target (sol 130), the area around the Rochette rock (sol 196), the Bastide outcrop (sol 204), the Mure and Dorie outcrop (sol 168), the Rocky Top in Hogwallow Flats (sol 466), the rocks in Badger Creek (sol 1222), and White Rocks (sol 1313) \cite{bell2022geological,owl_creek}. Models are 3D reconstructions from images captured by the Mastcam-Z instrument and navigation cameras mounted on the Perseverance rover \cite{bell2021mars} - and cameras on board the Ingenuity Mars Helicopter \cite{maki2025ingenuity}.

Scenes in FeatureID-3DS are composed from a set of models drawn from ShapeNetCore topLevelSynsetId $04460130$ (\textit{category: Tower})~\cite{chang2015shapenet} and a floor mesh in RGBA \([0.9, 0.9, 0.7, 1.0]\). The dataset comprises 60 glb files with language descriptions. Human-generated descriptions of objects are formed into a sentence for each viewpoint with a template. The true description is one of 5 samples presented as a list. To ensure matches from the true sample, matching modifiers in negative samples are replaced. Each text file is composed of descriptions from 6 viewpoints and a single summary string for the 60 scenes in the dataset. A feature is identified in the description for one of the viewpoints where visibility is confirmed.

\begin{figure}[hbt!]
  \centering
  \makebox[\linewidth][c]{\includegraphics[scale=0.50]{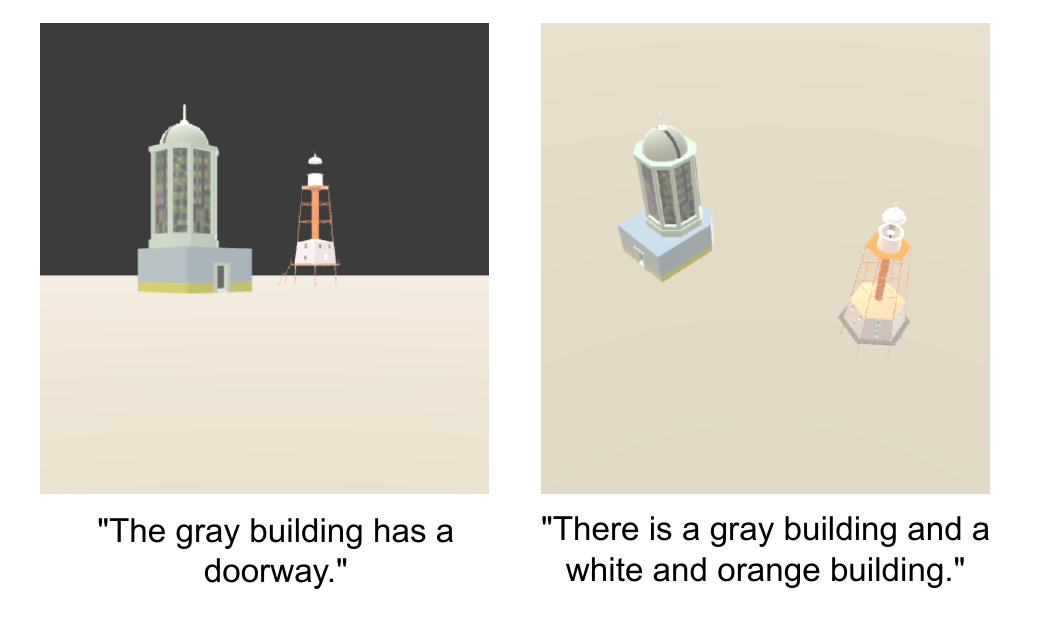}}
  \caption[Samples from the FeatureID-3DS dataset]{Samples from the FeatureID-3DS dataset.}
  \label{EffEdsample}
\end{figure}

PartialView-3DS scenes are generated from ShapeNetCore topLevelSynsetId $[3001627, 4379243]$ (\textit{category: Chair, Table}). Objects are separated by a partition mesh intersecting a floor mesh and limiting visibility to a single item from four of six cardinal viewpoints. A three-stage process to synthesize descriptions for scenes starts with human-generated natural language descriptions of constituent objects from \cite{han2020shapecaptioner}. Modifiers are verified by hand and corrected to match the objects. Sentences are generated for viewpoints by adding an indefinite article and period. Summary descriptions of the scene consist of object-level texts conjoined into a single sentence. The PartialView-3DS benchmark consists of a total of 60 glb files and paired description files. Language descriptions refer to the 3D scene from one of 6 viewpoints.

\begin{figure}[hbt!]
  \centering
  \makebox[\linewidth][c]{\includegraphics[scale=0.45]{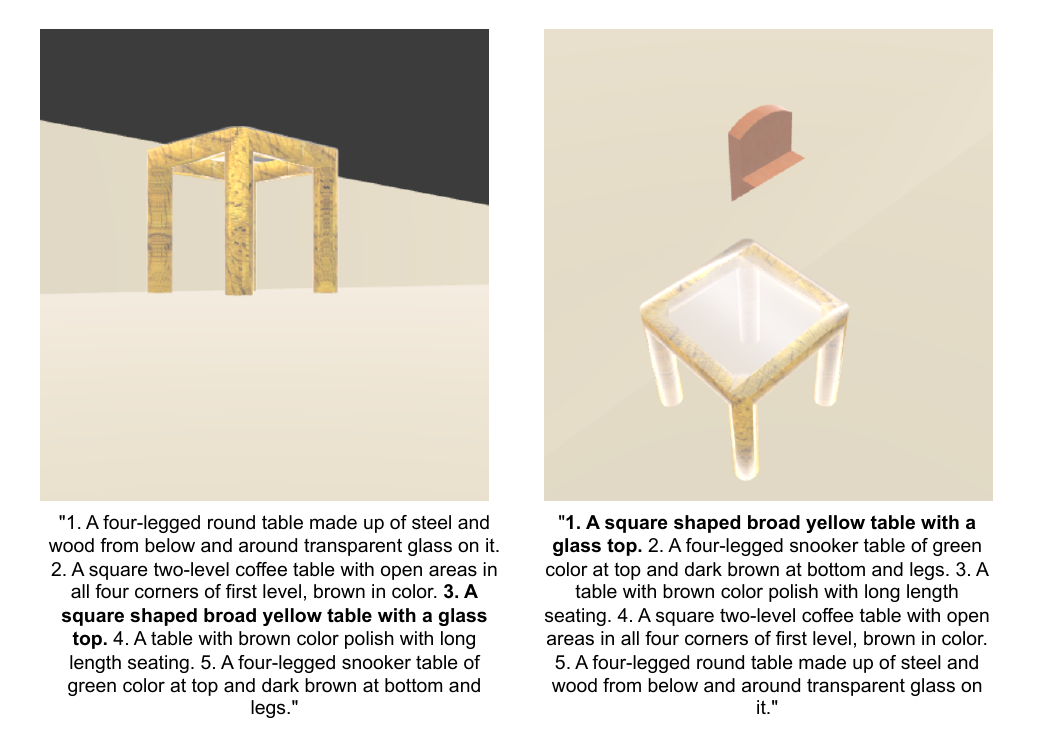}}
 \caption[Samples from the PartialView-3DS dataset]{Samples from the PartialView-3DS dataset. Descriptions in bold are the description matching the viewpoint.}
 \label{PartialView-3DSsample}
\end{figure}

\subsection{Diagnostic for Numerical Analysis}

Multi-information estimation methods where $n>2$ variables represent cross-modal inputs are evaluated with our UC-3DS-MI dataset (see Figure~\ref{fig:ucmscmsample}). Data design is balanced with 24 uniform scenes and an equal number of complex scenes. Each scene contains two polygon objects initialised in one of 6 position groups defined in relation to the centroid of a floor mesh. Abstract polygons are selected to limit the impact of class bias in the training data of VLM systems. Language descriptions and scenes are split between concentration on colour or geometry. Complexity relates to matches or differences between objects on these criteria. VLM systems provide predictions on the correctness of descriptions in relation to 6 viewpoints for each 3D scene.

\begin{figure}[h!]
\centering
\includegraphics[width=\columnwidth,keepaspectratio,trim=0 0 0 0,clip]{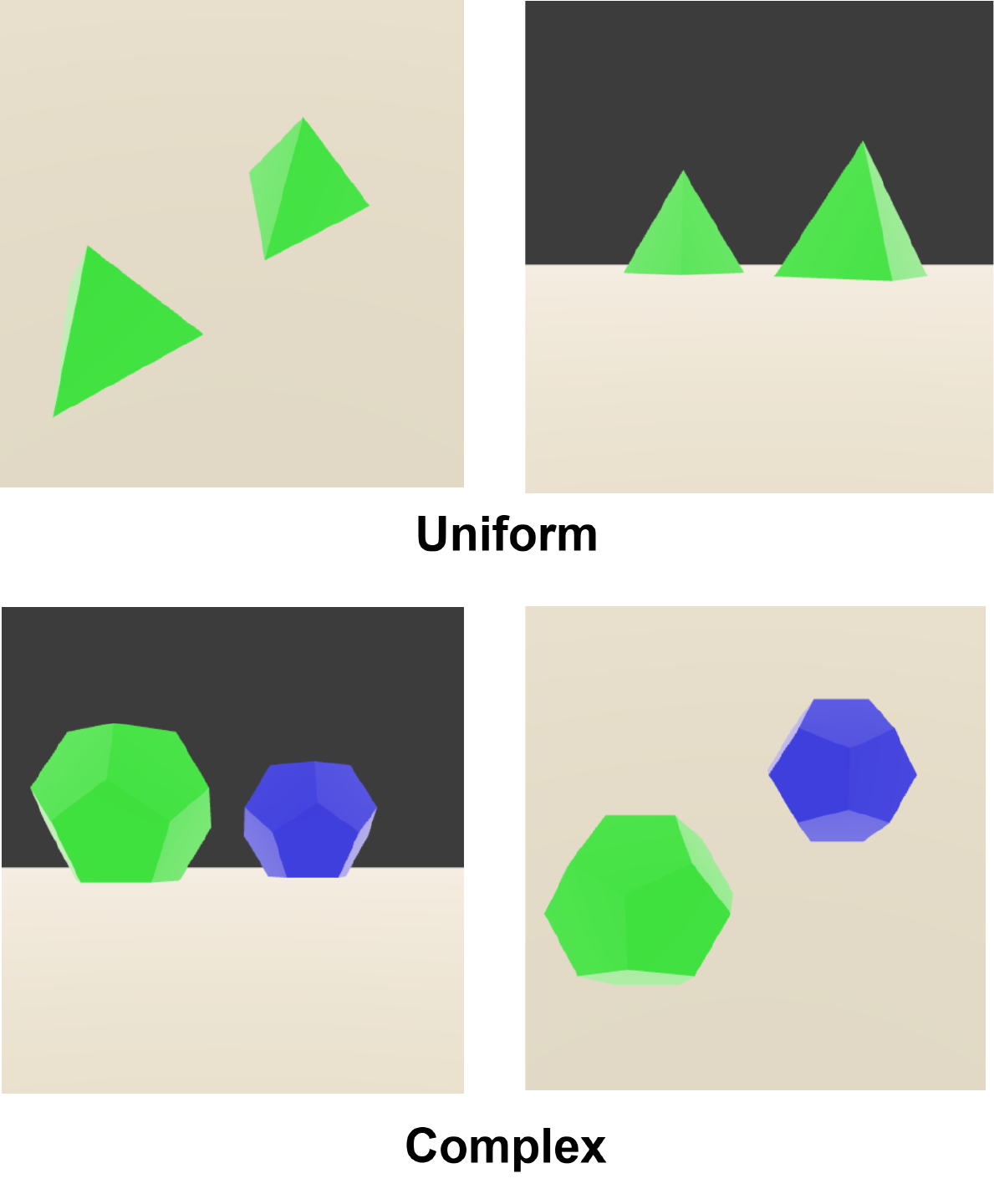}
\caption{Samples of uniform and complex scenes from the UC-3DS-MI dataset.}
  \label{fig:ucmscmsample}
\end{figure}

\subsection{Qualitative Analysis}

Our method selects optimal viewpoints in a sequence with a controller guided by information theoretic measures. The basis for our method is the principle that a VLM's errors in assessments will change given a visual input captured from an optimal viewpoint in relation to the textual input. In selected instances, assessment errors are returned by VLMs on all viewpoints presented to the VLM (see Figure \ref{fig:q_imgs}). A resulting hard limit exists on improvements in the accuracy of methods that run only in inference. 

\begin{figure}[hbt!]
\centering
\includegraphics[width=\columnwidth,keepaspectratio,trim=0 0 0 0,clip]{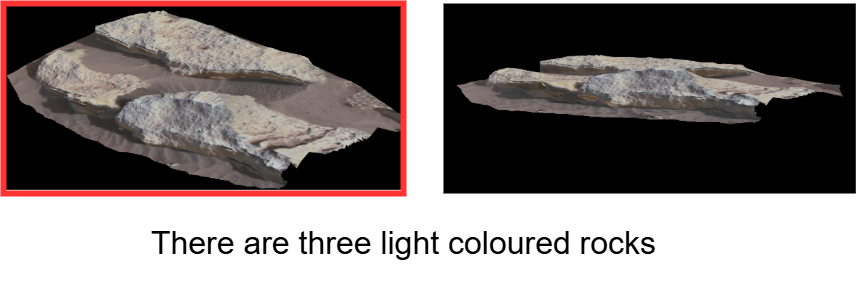}
\caption{We perform a qualitative assessment and find instances of incorrect predictions returned by the VLMs used in the assessments on optimal viewpoints and simple descriptions. In the samples presented here, the VLM returns a label of False for the view on the left. As a result, a hard limit on improvements from methods that are limited to inference-time changes exists.}
  \label{fig:q_imgs}
\end{figure}


\section{Additional Settings for Experiments}
\label{sec:app_experiments}

In this section, we provide core specifications and computing requirements to perform the experiments in the main paper. A run consists of $R=2$ rounds for the reported results with a sequence of actions estimated during the correction round following a measurement round and $n$ demonstrations. Control operations performed by all evaluated methods are updating demonstration data with system decisions in the measurement round for $(X, Y)$ viewpoints and $z$-axis levels, updating coefficients measured on the set of decisions with corresponding viewpoint labels, and predicting the set of camera actions for the correction round. Methods receive prediction errors for viewpoints where an error was marked in demonstrations and during the measurement round. 

To assess the results in Tables 2 and 3 in the main paper, we present additional details on task settings and technical specifications. \textbf{Number of Demonstrations} Demonstrations are set at $5\%$ of scenes for each benchmark: $n\text{=}3$ for feature identification (FeatureID-3DS) and object occlusion (PartialView-3DS). \textbf{Additional VLM system specifications} Weights for VLM systems are llava-v1.5-13b for Video-LLaMA-13B and vicuna-13b-v1.5 for Chat-UniVi-13B~\cite{liu2023llava,NEURIPS2023_91f18a12}. \textbf{Additional camera and video settings} A main scene camera is initialised. Video is recorded at $60 fps$. \textbf{Software} In-scene camera operations and actions are defined with open source versions of Harfang 3.2.4 and HarfangHighLevel libraries~\cite{Harfangurl}. Video is created with OpenCV 4.9.0.80. \textbf{Infrastructure} A single NVIDIA A100 80GB GPU is used for all runs reported in the experiments section to support VLM systems at inference time. Our controller adds no GPU processing or VRAM memory requirements additional to the hardware allocation for running the VLM. In-scene camera operations are performed with a single NVIDIA GeForce RTX 2080 Super 8GB GPU.


\section{Additional Results from Experiments}
\label{sec:app_results}

\begin{figure*}[hbt!]
    \centering
        \begin{minipage}{\linewidth}
            \centering
            \underline{\bf \small GO-LED-OL}
            \par
            \vspace{1em}
            \bf \small Uniform
            \hspace{20em}
            \bf \small Complex
            \\[-0.2ex]

            \begin{minipage}{0.24\textwidth}
                \centering
                \includegraphics[width=\textwidth]{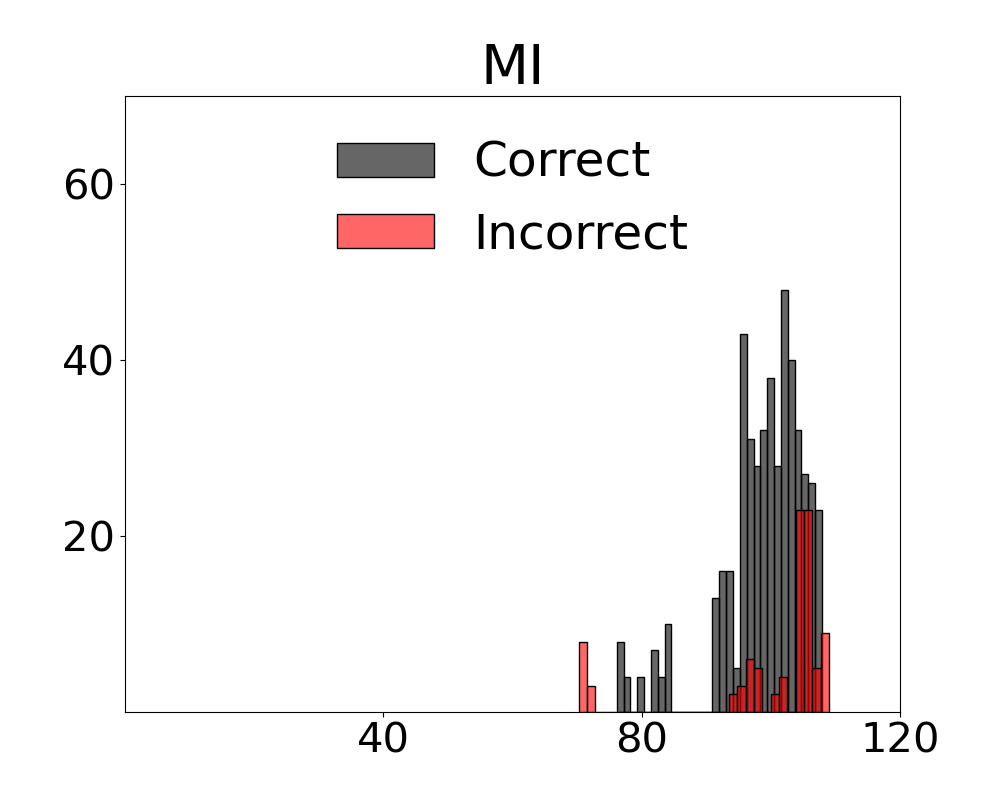}
                \label{fig:multi-info1}
            \end{minipage}
            \hspace{-0.5em}
            \begin{minipage}{0.24\textwidth}
                \centering
                \includegraphics[width=\textwidth]{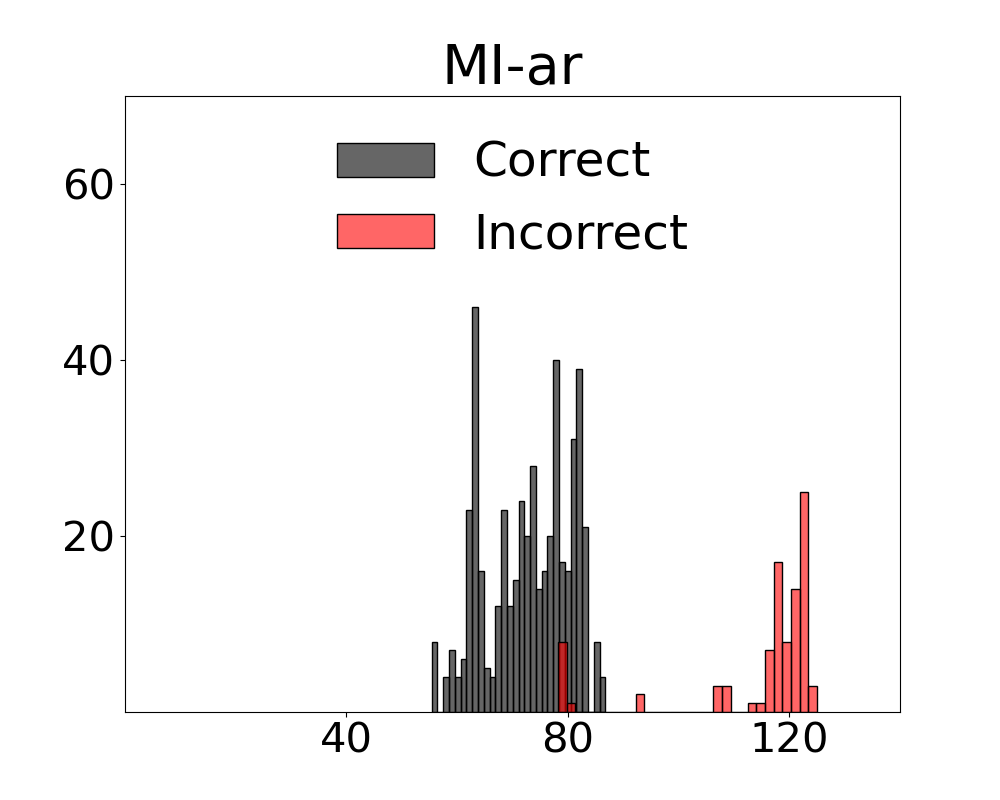}
                \label{fig:multi-info3}
            \end{minipage}
            \hspace{-0.1em}
            \raisebox{-0.35\height}{\rule{0.5pt}{0.11\textheight}}
            \begin{minipage}{0.24\textwidth}
                \centering
                \includegraphics[width=\textwidth]{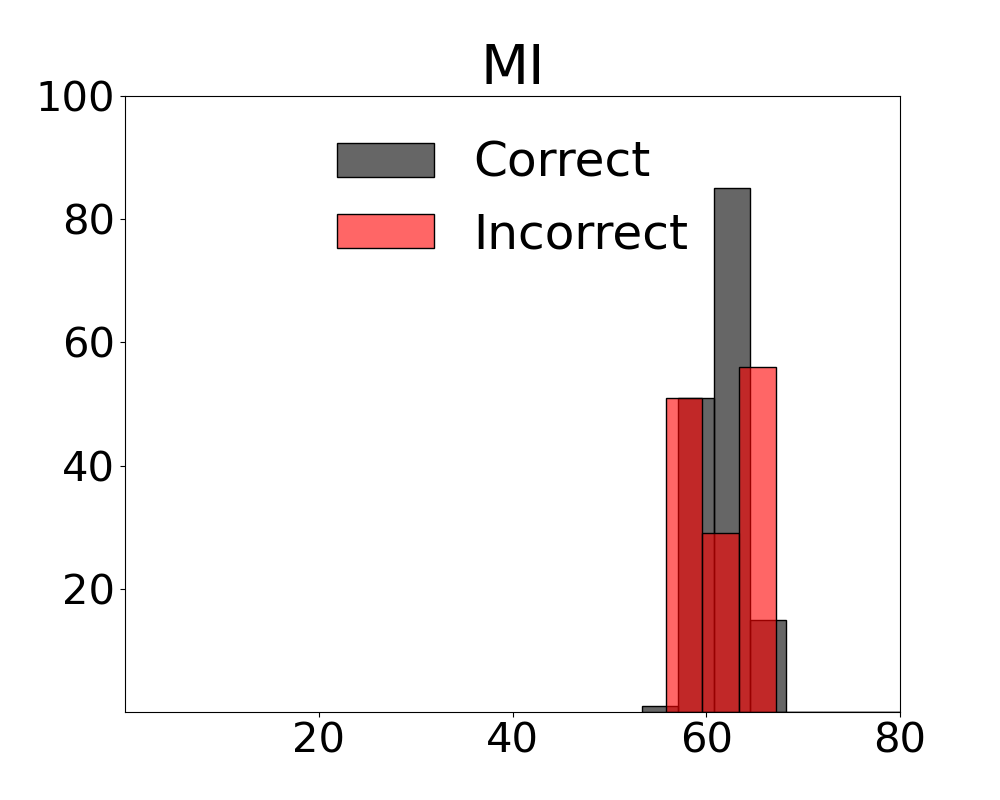}
                \label{fig:multi-info-regret1}
            \end{minipage}
            \hspace{0.01em}
            \begin{minipage}{0.24\textwidth}
                \centering
                \includegraphics[width=\textwidth]{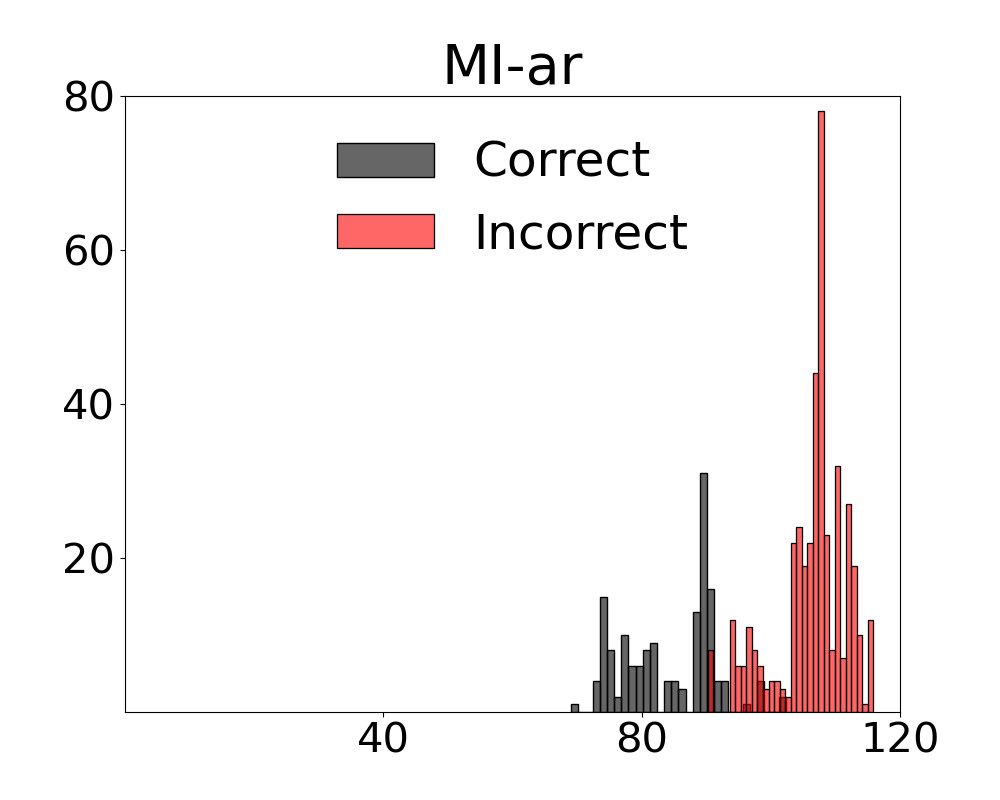}
                \label{fig:multi-info-regret3}
            \end{minipage}
            \\[-2.5ex]

            \caption{The $MI\text{-}ar$ metric groups correct and incorrect scores
            in distinct regions of the distribution, for both uniform objects and more
            complex scenes.}
            \label{fig:hist_uc_horizontal_goled}
        \end{minipage}
\end{figure*}

\begin{figure*}[hbt!]
    \centering
        \begin{minipage}{\linewidth}
            \centering
            \underline{\bf \small GH-LED}
            \par
            \vspace{1em}
            \bf \small Uniform
            \hspace{20em}
            \bf \small Complex
            \\[-0.2ex]

            \begin{minipage}{0.24\textwidth}
                \centering
                \includegraphics[width=\textwidth]{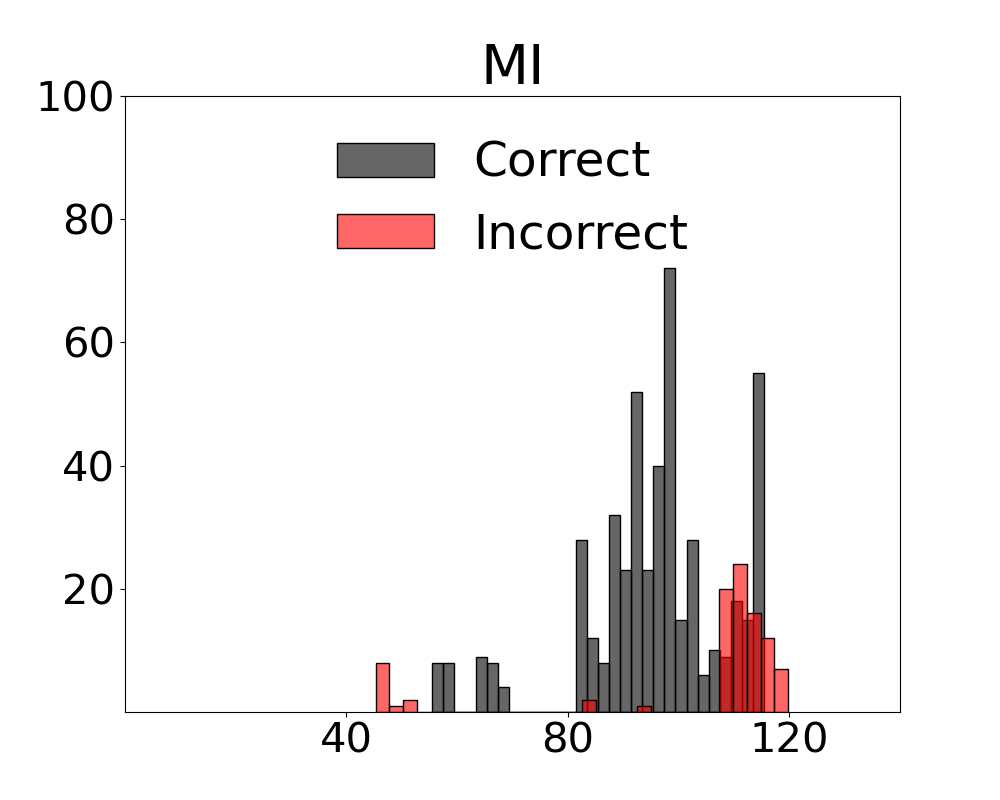}
                \label{fig:multi-info1}
            \end{minipage}
            \hspace{-0.5em}
            \begin{minipage}{0.24\textwidth}
                \centering
                \includegraphics[width=\textwidth]{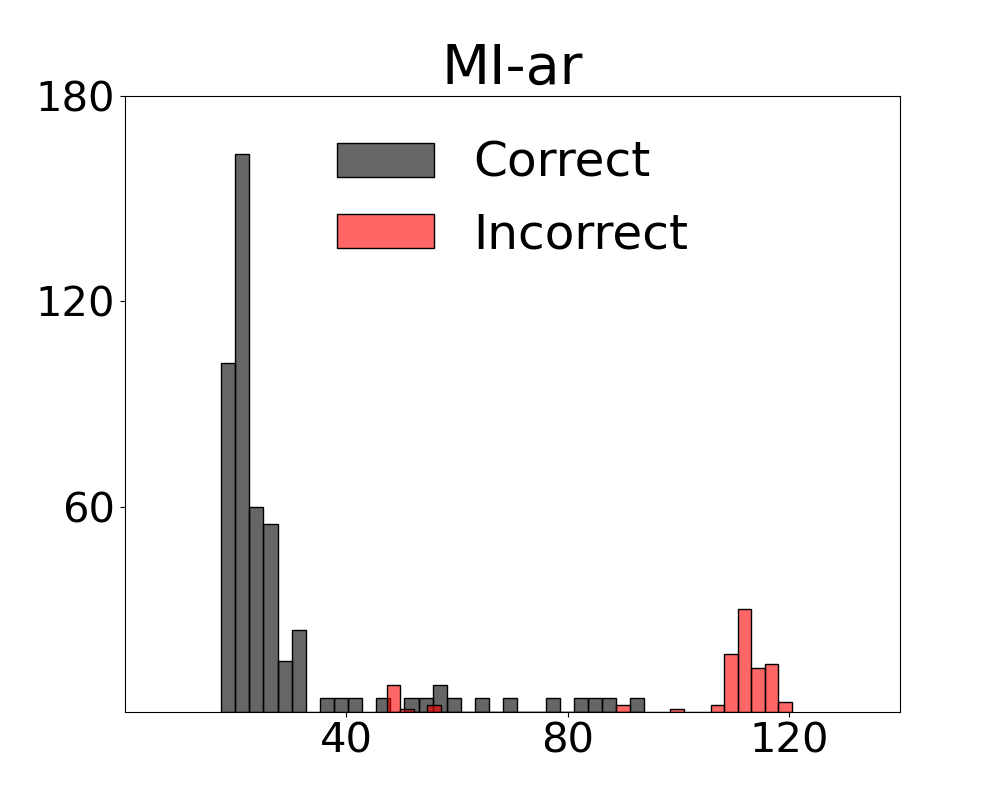}
                \label{fig:multi-info3}
            \end{minipage}
            \hspace{-0.1em}
            \raisebox{-0.35\height}{\rule{0.5pt}{0.11\textheight}}
            \begin{minipage}{0.24\textwidth}
                \centering
                \includegraphics[width=\textwidth]{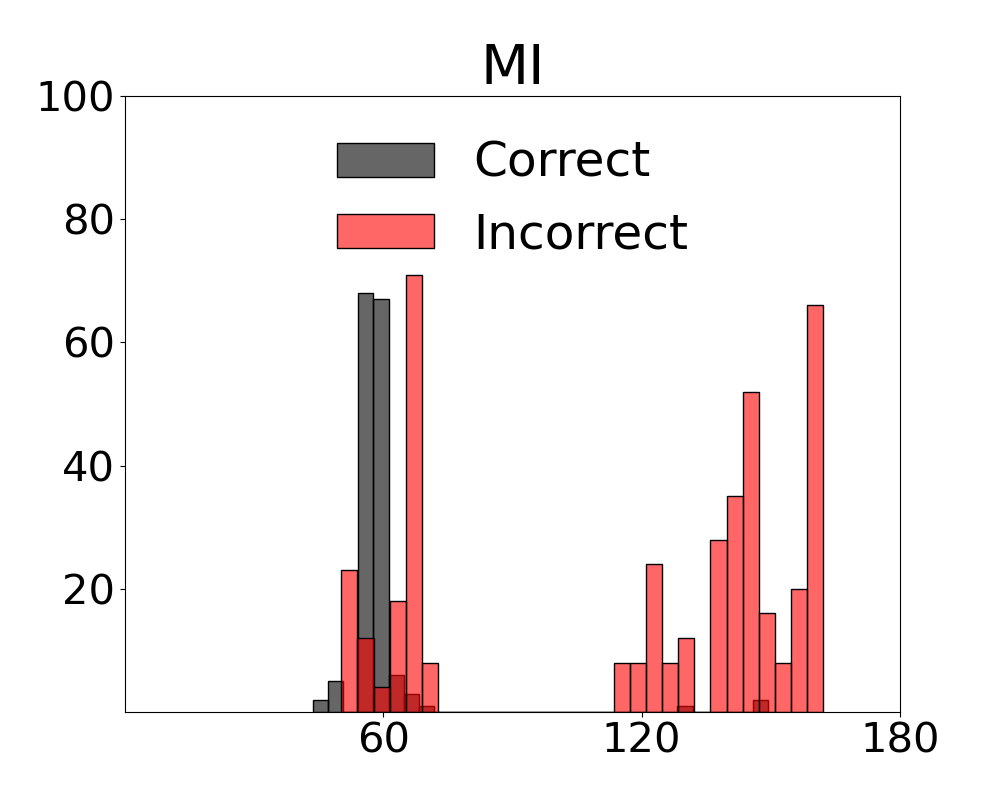}
                \label{fig:multi-info-regret1}
            \end{minipage}
            \hspace{0.01em}
            \begin{minipage}{0.24\textwidth}
                \centering
                \includegraphics[width=\textwidth]{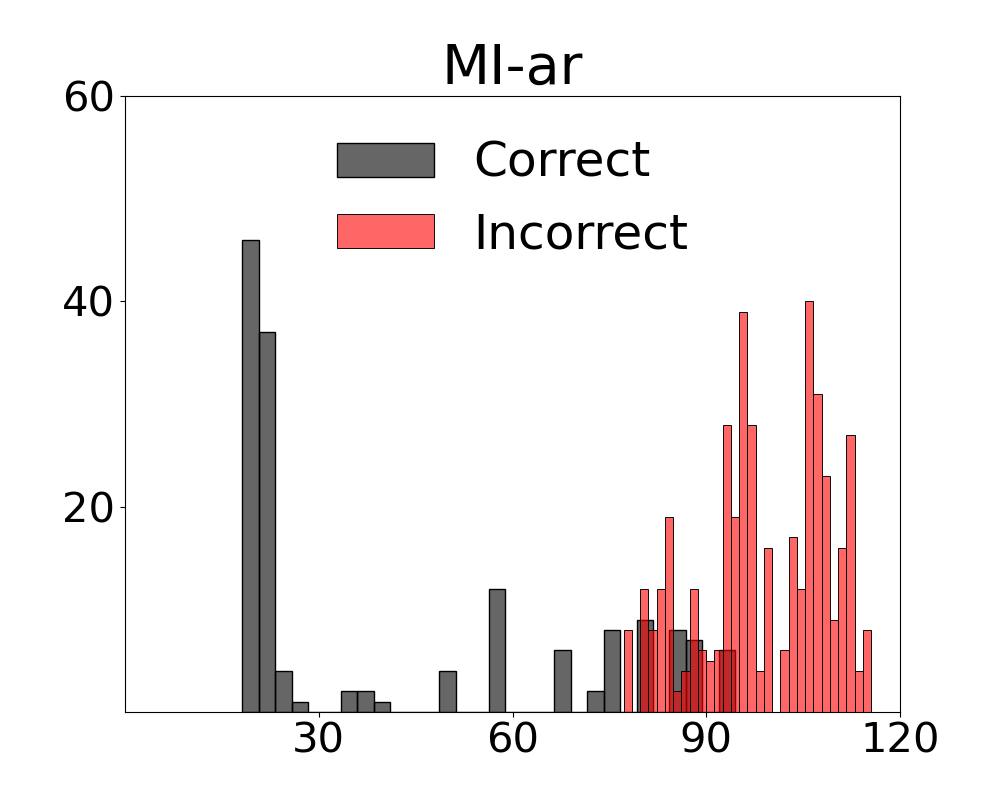}
                \label{fig:multi-info-regret3}
            \end{minipage}
            \\[-2.5ex]

            \caption{Active minimisation of regret distributes correct and incorrect scores 
            in identifiable concentrations for a second multi-information metric.}
            \label{fig:hist_uc_horizontal_gh}
        \end{minipage}
\end{figure*}

\subsection{Score Distributions for Multivariate Metrics on Uniform and Complex Scenes}

We present fine-grained detail on the efficacy of our multivariate metrics with active regret minimisation to distribute the scores for correct and incorrect responses in separate regions in Figures~\ref{fig:hist_uc_horizontal_goled} and~\ref{fig:hist_uc_horizontal_gh}. These are compared to variants with no $ar$ on scenes in our diagnostic defined as uniform and complex. 


\begin{figure*}[hbt!]
    \centering
        \begin{minipage}{\linewidth}
            \centering
            \begin{minipage}{0.25\textwidth}
                \centering
                \hspace{2em}
                \underline{\bf \small GO-LED-OL\_ar}
                \includegraphics[width=\textwidth]{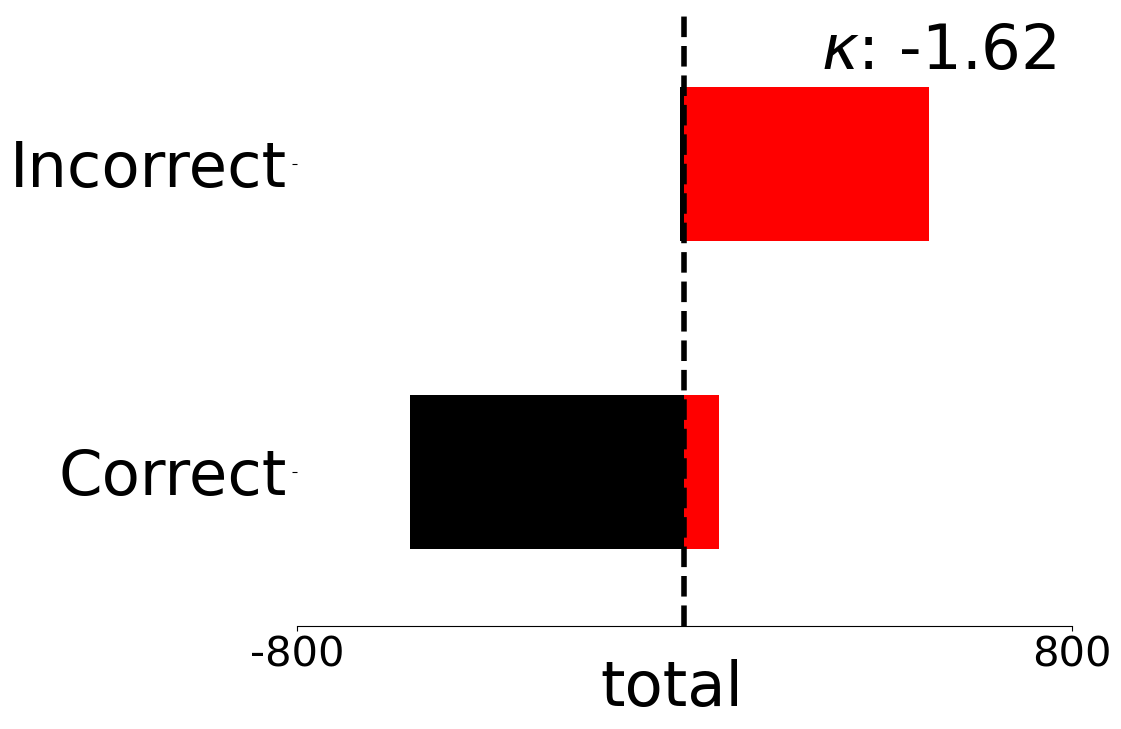}
                \label{fig:multi-info1}
            \end{minipage}%
            \begin{minipage}{0.25\textwidth}
                \centering
                \hspace{3em}
                \underline{\bf \small GO-LED-OL}
                \includegraphics[width=\textwidth]{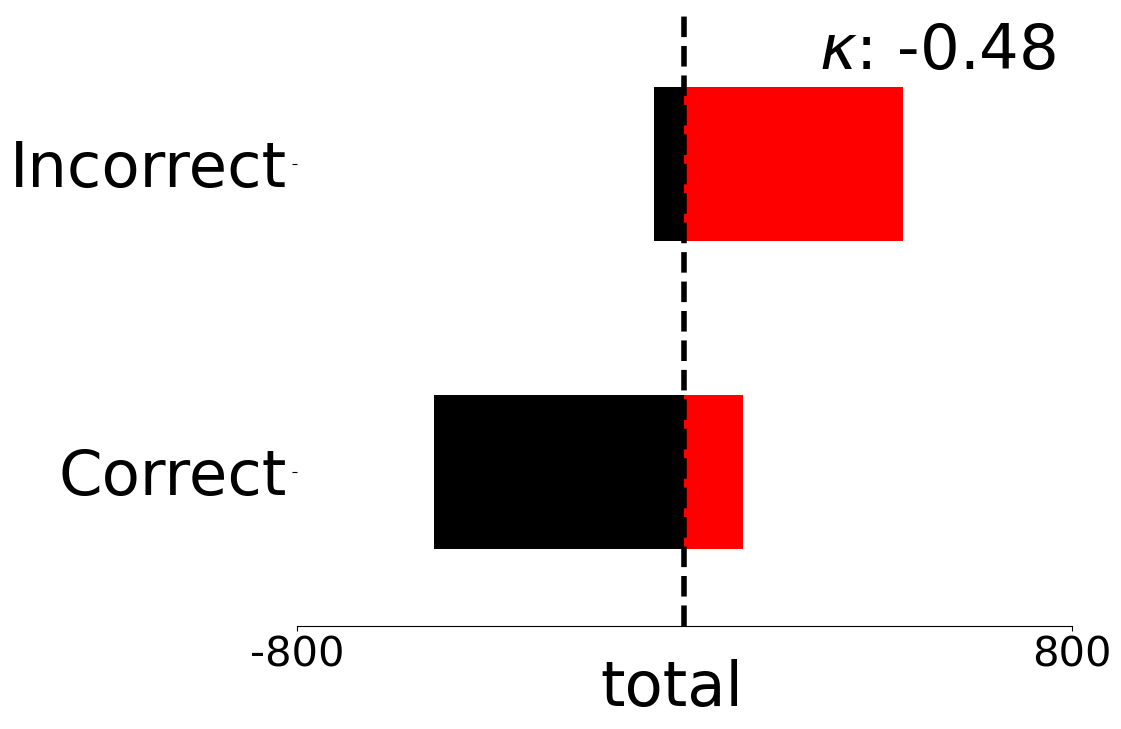}
                \label{fig:multi-info3}
            \end{minipage}%
            \begin{minipage}{0.25\textwidth}
                \centering
                \hspace{2.5em}
                \underline{\bf \small GO}
                \includegraphics[width=\textwidth]{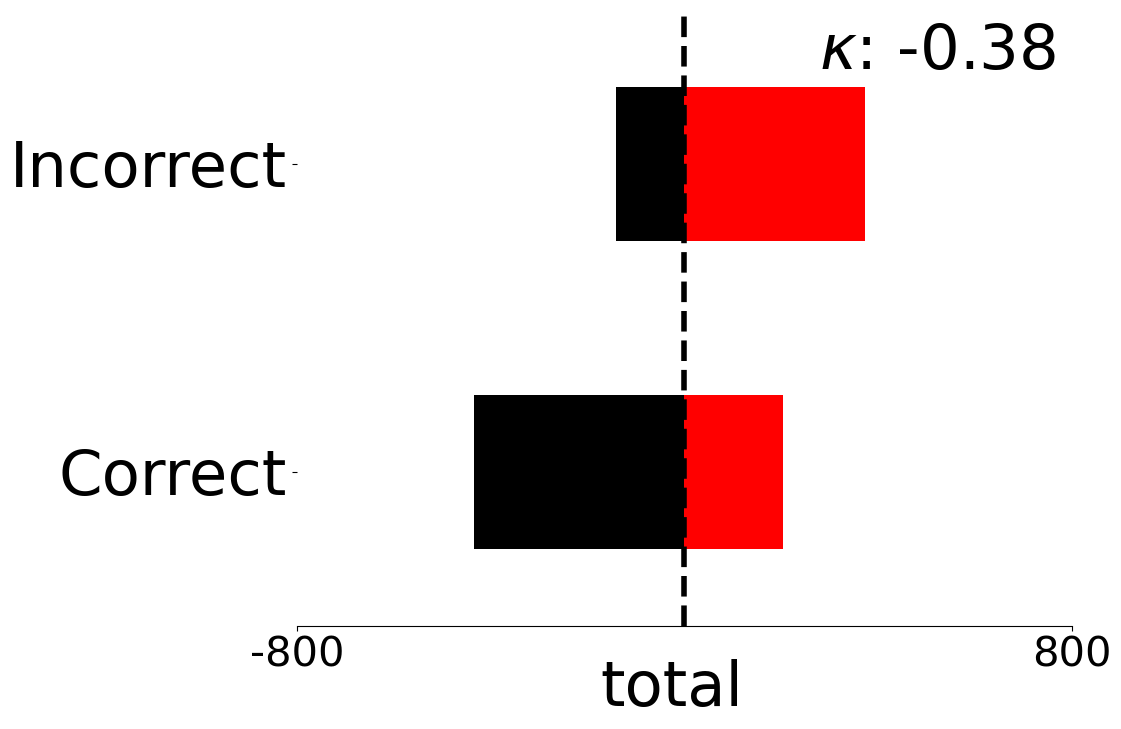}
                \label{fig:multi-info-regret1}
            \end{minipage}%
            \begin{minipage}{0.25\textwidth}
                \centering
                \hspace{2.5em}
                \underline{\bf \small OL}
                \includegraphics[width=\textwidth]{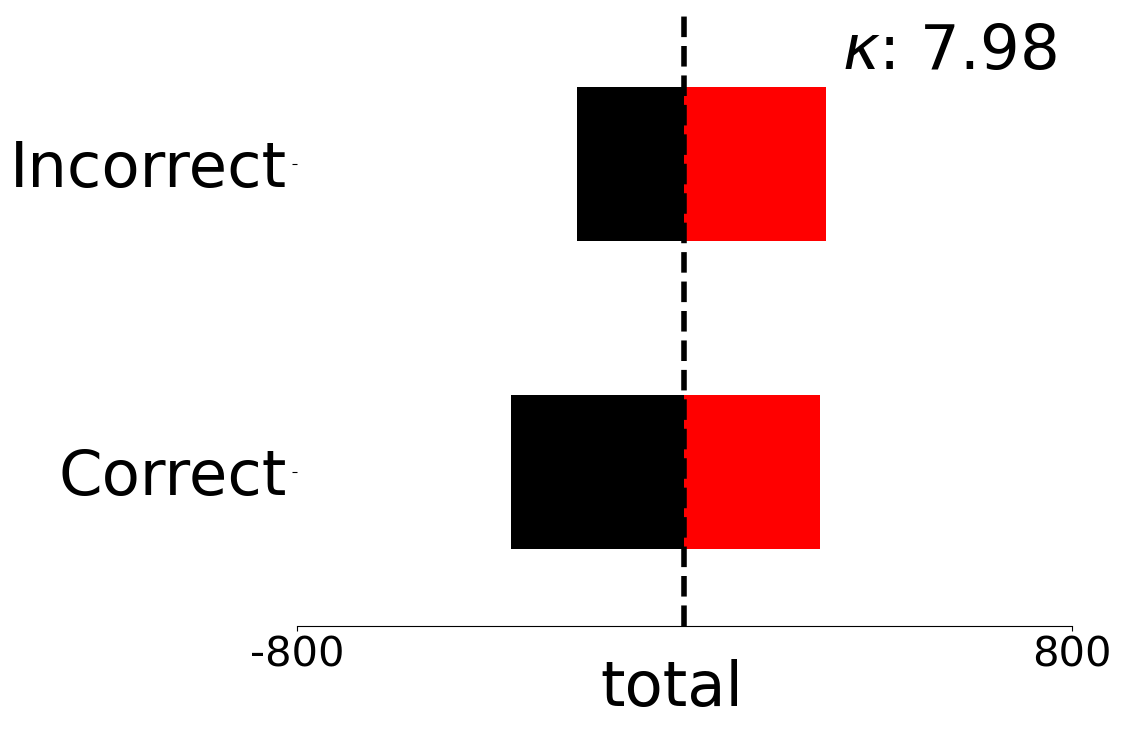}
                \label{fig:multi-info-regret3}
            \end{minipage}
            
            \caption{$GO\text{-}LED\text{-}OL$ scores for correct and incorrect responses around the median.
            LED score is shown in Figure~\ref{fig:gla_med_correct}. The median is preferred
            over the mean for low skew in small-sample tests. Kurtosis $\kappa$ indicates
            moderate to low tail concentration for multi-information with active regret.}
            \label{fig:gla_go_med_correct}
        \end{minipage}
\end{figure*}

\begin{figure*}[hbt!]
    \centering
        \begin{minipage}{\linewidth}
            \centering
            \begin{minipage}{0.25\textwidth}
                \centering
                \vspace{1em}
                \hspace{2em}
                \underline{\bf \small GH-LED\_ar}
                \includegraphics[width=\textwidth]{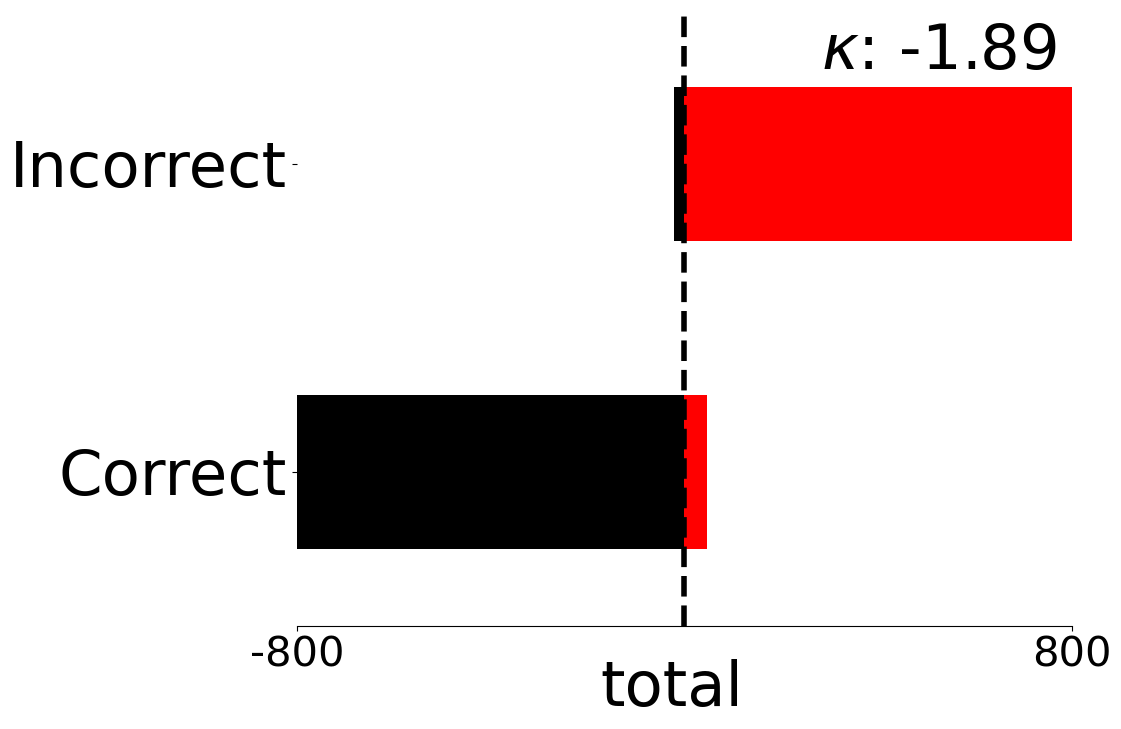}
                \label{fig:multi-info1}
            \end{minipage}%
            \begin{minipage}{0.25\textwidth}
                \centering
                \hspace{3em}
                \underline{\bf \small GH-LED}
                \includegraphics[width=\textwidth]{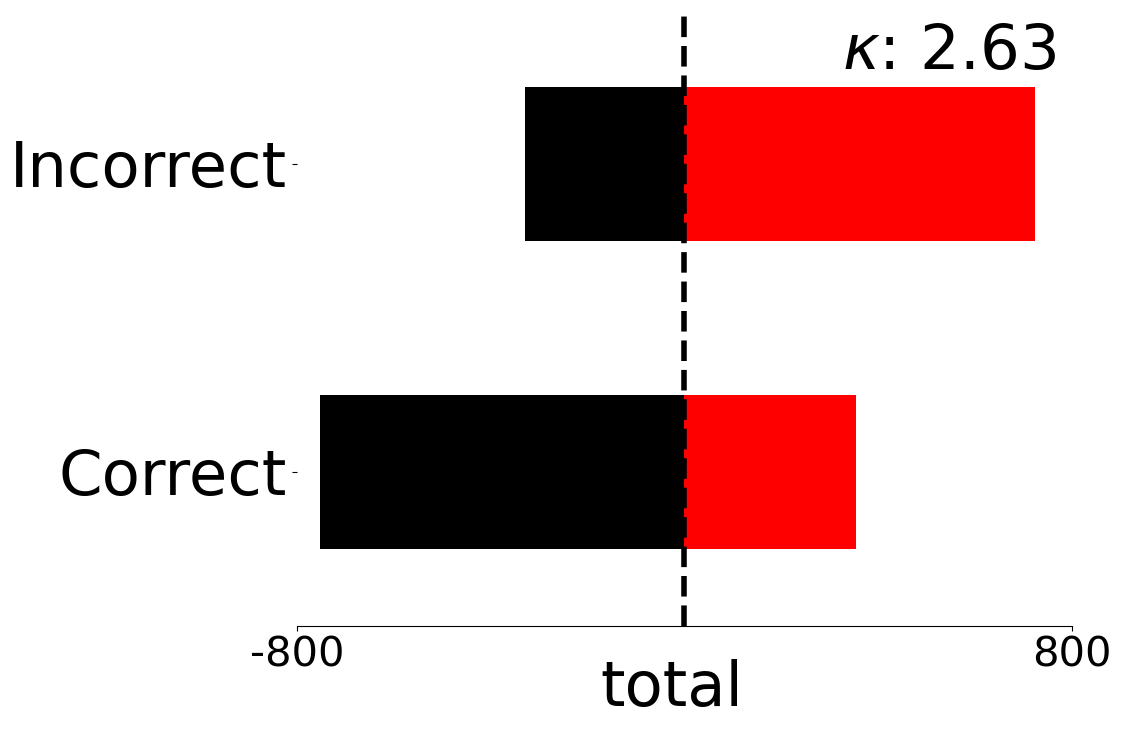}
                \label{fig:multi-info3}
            \end{minipage}%
            \begin{minipage}{0.25\textwidth}
                \centering
                \hspace{2.5em}
                \underline{\bf \small GH}
                \includegraphics[width=\textwidth]{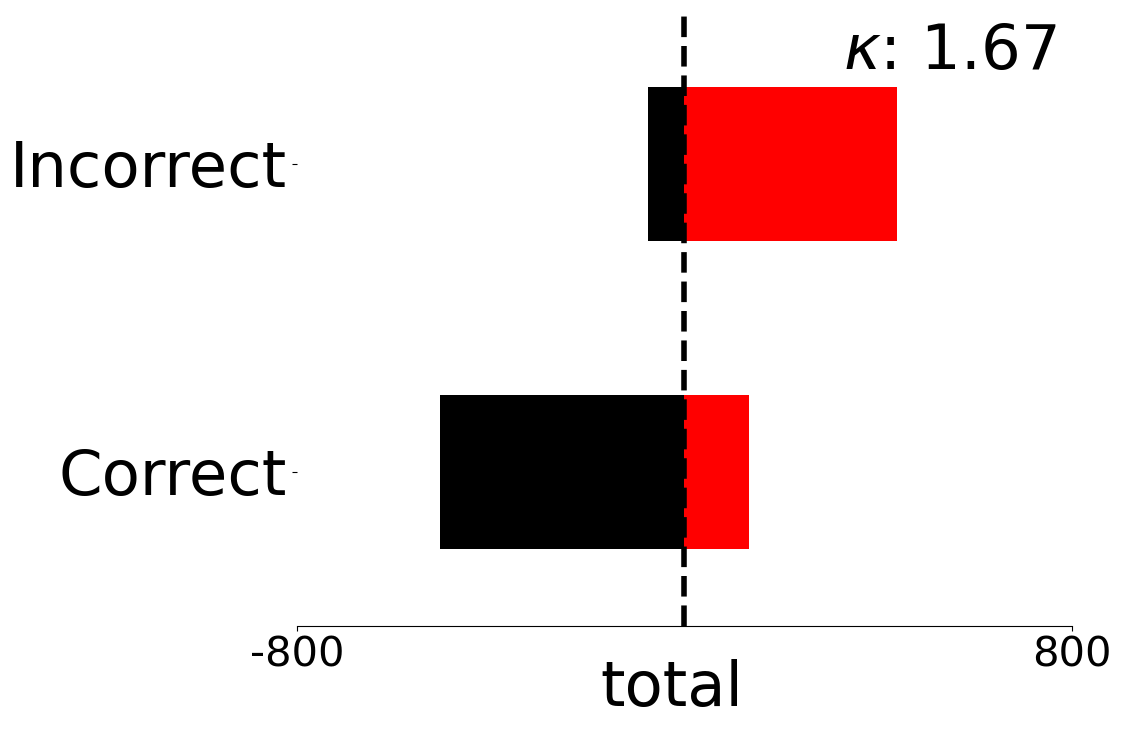}
                \label{fig:multi-info-regret1}
            \end{minipage}%
            \begin{minipage}{0.25\textwidth}
                \centering
                \hspace{2.5em}
                \underline{\bf \small LED}
                \includegraphics[width=\textwidth]{figs/led_score_all_bar.png}
                \label{fig:multi-info-regret3}
            \end{minipage}
            
            \caption{GH-LED scores for correct and incorrect responses distributed around
            the median.}
            \label{fig:gla_med_correct}
        \end{minipage}
\end{figure*}

\subsection{Score Distributions for Multivariate Metrics in Relation to Median Scores}

In Figures~\ref{fig:gla_go_med_correct} and~\ref{fig:gla_med_correct}, an additional analysis of score distributions for the two variants of our metrics with active regret minimisation $ar$ are analysed in relation to the median scores. This assessment is also provided for variants with no $ar$ and univariate metrics.

\subsection{Score Distributions for Univariate Metrics}
We provide a visualisation of how univariate measures - computed with the constituent variables in our multivariate metrics - perform as indicators of the complexity of inputs in relation to the decisions outputted by the VLM in Figure~\ref{fig:hist_comp_vertical}. The distributions indicate that any variable taken on its own provides scores that conflate samples resulting in a correct or incorrect decision by the system.

\begin{figure*}[hbt!]
    \centering
        \begin{minipage}{\linewidth}
            \centering
            \begin{minipage}{0.24\textwidth}
                \centering
                \hspace{2em}
                \includegraphics[width=\textwidth]{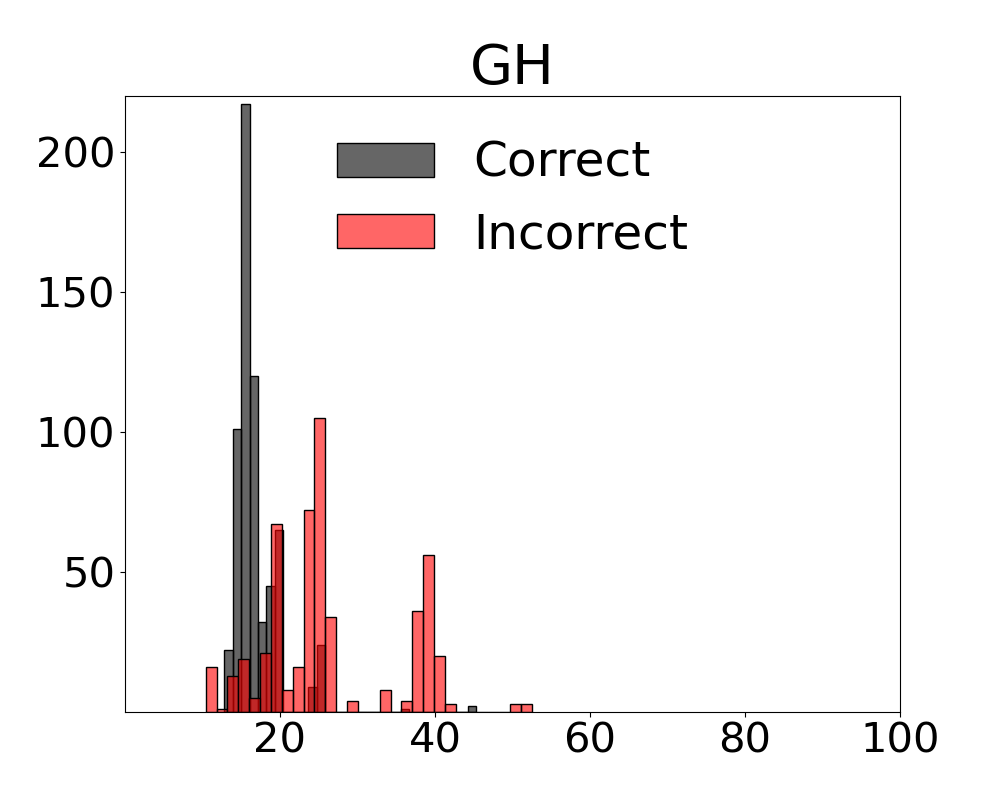}
                \label{fig:left1}
            \end{minipage}%
            \begin{minipage}{0.24\textwidth}
                \centering
                \hspace{3em}
                \includegraphics[width=\textwidth]{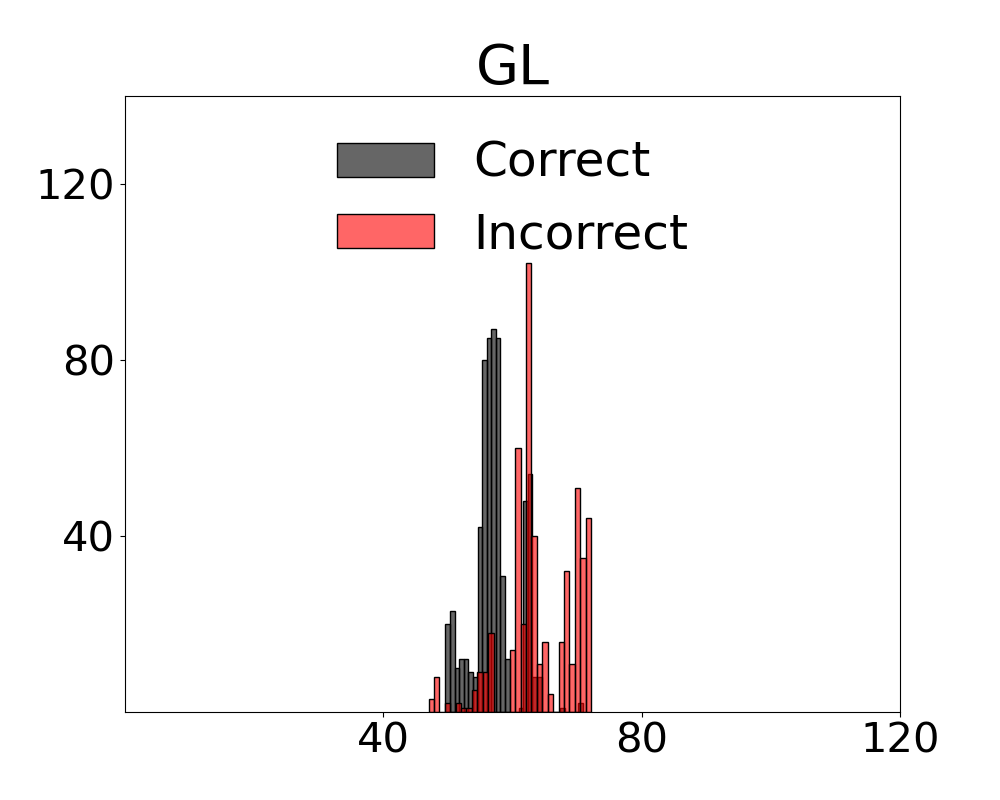}
                \label{fig:left2}
            \end{minipage}%
            \begin{minipage}{0.24\textwidth}
                \centering
                \hspace{2.5em}
                \includegraphics[width=\textwidth]{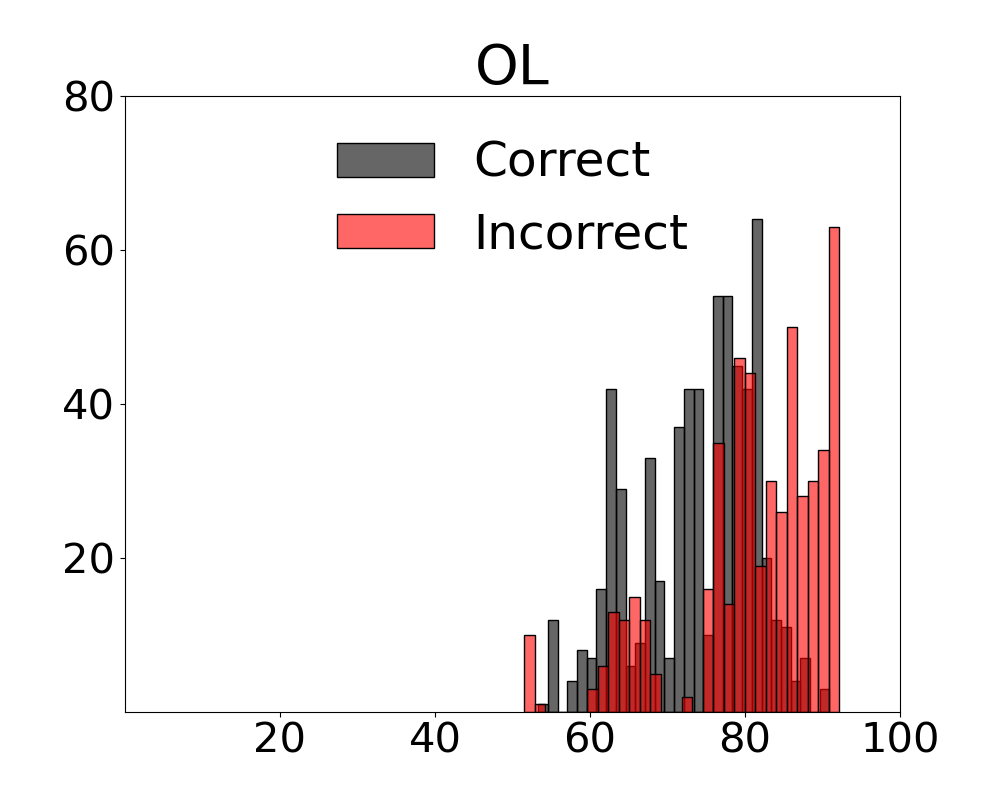}
                \label{fig:right1}
            \end{minipage}%
            \begin{minipage}{0.24\textwidth}
                \centering
                \hspace{2.5em}
                \includegraphics[width=\textwidth]{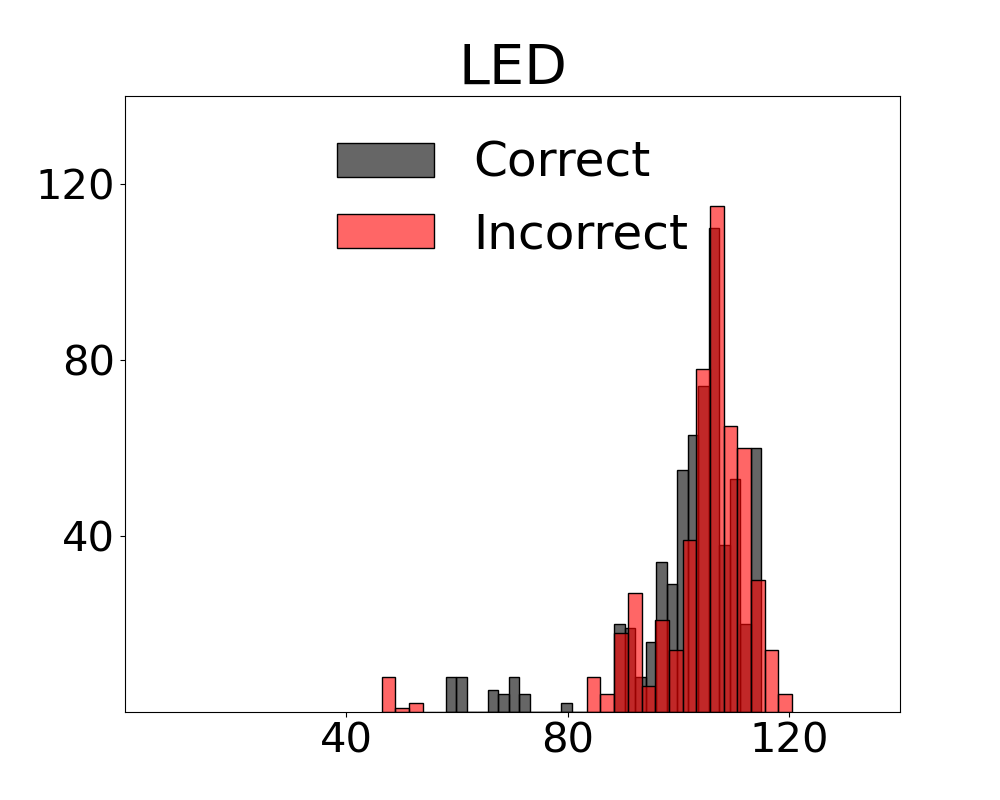}
                \label{fig:right2}
            \end{minipage}

            \caption{Score distributions for single visual inputs in relation to 
            correct and incorrect responses.}
            \label{fig:hist_comp_vertical}
        \end{minipage}
\end{figure*}

\subsection{Neural Methods and Metrics with Active Regret}
We assess $MI\text{-}ar$ metrics with the two neural methods in our evaluations on the GeoProperties-3DS benchmark (see Table~\ref{table:neural_mizo}). Limited improvements on these methods are a reasonable expectation in the low-data regimes of the tasks that our research focuses on. Neural methods with standard first-order optimisation such as Stochastic Gradient Descent (SGD) are 
better suited to large-scale data scenarios~\citep{anthony2009learning}. In contrast our approach is effective for scenarios suited to online optimisation with limited spare samples for training.

\begin{table}[hbt!]
\centering
\begin{tabular}{p{5.0cm} p{0.8cm} @{\hspace{0.9cm}} p{1.6cm}}
\toprule
{} & \textbf{BER@8\,$\downarrow$} & \textbf{$\Delta$ on R1} \\ \midrule
\textbf{Video-LLaMA-13B} \\ 
\hspace{1mm} Ours+GO-LED-OL\textsubscript{ar} & 41.3 & -18.9 \\
\hspace{1mm} Ours+GH-LED\textsubscript{ar} & 40.4 & -21.0 \\
\hspace{1mm} Linear+SGD+GO-LED-OL\textsubscript{ar} & 54.2 & -7.6 \\
\hspace{1mm} Linear+SGD+GH-LED\textsubscript{ar} & 54.6 & -7.2 \\
\hspace{1mm} RBF+GO-LED-OL\textsubscript{ar} & 54.6 & -6.3 \\
\hspace{1mm} RBF+GH-LED\textsubscript{ar} & 54.5 & -6.4 \\
\cmidrule{1-3}
\textbf{Chat-UniVi-13B} \\ 
\hspace{1mm} Ours+GO-LED-OL\textsubscript{ar} & 43.1 & -18.4 \\
\hspace{1mm} Ours+GH-LED\textsubscript{ar} & 42.2 & -20.4 \\
\hspace{1mm} Linear+SGD+GO-LED-OL\textsubscript{ar} & 54.3 & -8.5 \\
\hspace{1mm} Linear+SGD+GH-LED\textsubscript{ar} & 54.9 & -7.9 \\
\hspace{1mm} RBF+GO-LED-OL\textsubscript{ar} & 56.4 & -6.2 \\
\hspace{1mm} RBF+GH-LED\textsubscript{ar} & 56.7 & -5.9 \\
\bottomrule
\end{tabular}
\caption{Neural methods with our MI-ZO metrics on the GeoProperties-3DS benchmark compared to our controller (Poly+ZO+MI).}
\label{table:neural_mizo}
\end{table}

\subsection{Variation in Results for Experiments}

We report standard deviations $\sigma$ for the two benchmark experiments by system and method in Tables~\ref{table:varbenchoo} and~\ref{table:varbenchfi}. Distributions of scores over runs are presented in Figures \ref{fig:varobjoc} and \ref{fig:varfi}. Means of each set of runs are displayed as black points and bars indicate the credible interval. 

\begin{table}[hbt!]
\centering
\begin{tabular}{p{5.7cm} p{1.2cm}}
\toprule
                                & $\boldsymbol{\sigma@8}$ \\ \midrule
\textbf{Video-LLaMA-13B} \\ 
\hspace{1mm} PID & 0.52 \\
\hspace{1mm} Extended Kalman & 0.50 \\
\hspace{1mm} Linear+SGD & 1.16 \\
\hspace{1mm} Poly+ZO+MI (ours) & 0.38 \\
\hspace{1mm} +GO-LED-OL & 1.76 \\
\hspace{1mm} +GH-LED & 1.32 \\
\hspace{1mm} +GO-LED-OL\textsubscript{ar} & 2.12 \\
\hspace{1mm} +GH-LED\textsubscript{ar} & 2.07 \\
\cmidrule{1-2}
\textbf{Chat-UniVi-13B} \\ 
\hspace{1mm} PID & 0.74 \\
\hspace{1mm} Extended Kalman & 0.40 \\
\hspace{1mm} Linear+SGD & 1.58 \\
\hspace{1mm} Poly+ZO+MI (ours) & 0.53 \\
\hspace{1mm} +GO-LED-OL & 1.18 \\
\hspace{1mm} +GH-LED & 1.65 \\
\hspace{1mm} +GO-LED-OL\textsubscript{ar} & 1.12 \\
\hspace{1mm} +GH-LED\textsubscript{ar} & 1.45 \\
\bottomrule
\end{tabular}
\caption{Standard deviation over 10 runs on the object occlusion benchmark.}
\label{table:varbenchoo}
\end{table}

\begin{table}[hbt!]
\centering
\begin{tabular}{p{5.0cm} p{0.8cm} p{0.8cm}}
\toprule
                                & $\boldsymbol{\sigma@5}$ & $\boldsymbol{\sigma@8}$ \\ \midrule
\textbf{Video-LLaMA-13B} \\ 
\hspace{1mm} PID & 0.50 & 0.88 \\
\hspace{1mm} Extended Kalman & 0.57 & 0.78 \\
\hspace{1mm} Linear+SGD & 1.08 & 1.83 \\
\hspace{1mm} Poly+ZO+MI (ours) & 0.89 & 1.10 \\
\hspace{1mm} +GO-LED-OL & 1.26 & 1.94 \\
\hspace{1mm} +GH-LED & 1.67 & 2.47 \\
\hspace{1mm} +GO-LED-OL\textsubscript{ar} & 1.98 & 2.65 \\
\hspace{1mm} +GH-LED\textsubscript{ar} & 1.96 & 2.16 \\
\cmidrule{1-3}
\textbf{Chat-UniVi-13B} \\ 
\hspace{1mm} PID & 0.43 & 0.64 \\
\hspace{1mm} Extended Kalman & 0.47 & 0.51 \\
\hspace{1mm} Linear+SGD & 1.11 & 1.79 \\
\hspace{1mm} Poly+ZO+MI (ours) & 0.63 & 1.06 \\
\hspace{1mm} +GO-LED-OL & 1.11 & 1.84 \\
\hspace{1mm} +GH-LED & 1.56 & 1.61 \\
\hspace{1mm} +GO-LED-OL\textsubscript{ar} & 1.20 & 2.84 \\
\hspace{1mm} +GH-LED\textsubscript{ar} & 1.48 & 2.88 \\
\bottomrule
\end{tabular}
\caption{Standard deviation over 10 runs on the feature identification benchmark.}
\label{table:varbenchfi}
\end{table}

\begin{figure}[hbt!]
    \centering
    
    \begin{minipage}{\linewidth}
        \centering
        \includegraphics[width=\textwidth]{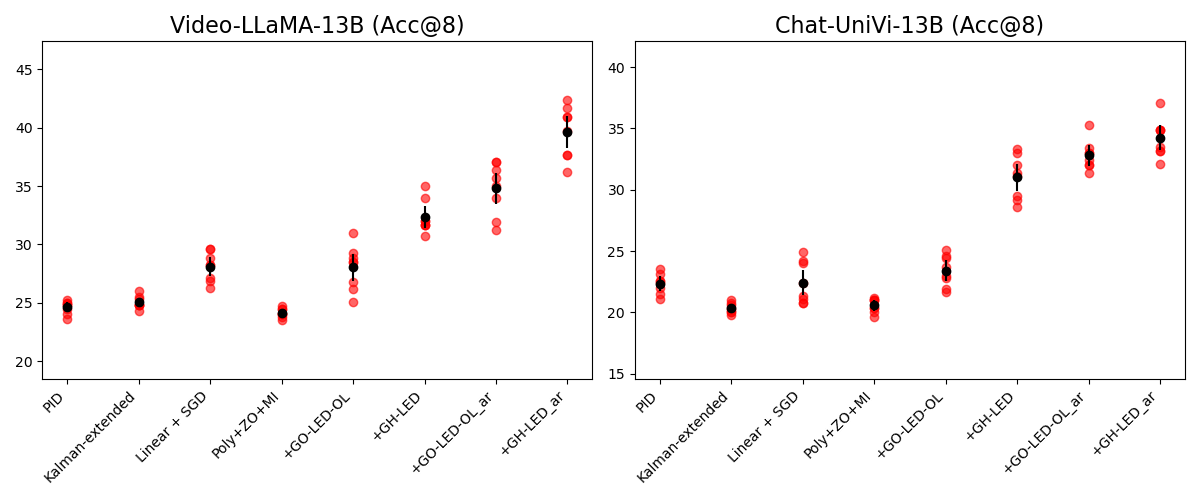}
        \label{fig:upper}
    \end{minipage}

    \caption{Variance over runs by method in our evaluation on handling object occlusion.}
    \label{fig:varobjoc}
\end{figure}

\begin{figure}[hbt!]
    \centering
    
    \begin{minipage}{\linewidth}
        \centering
        \bf \small Video-LLaMA-13B
        \par
        \includegraphics[width=\textwidth]{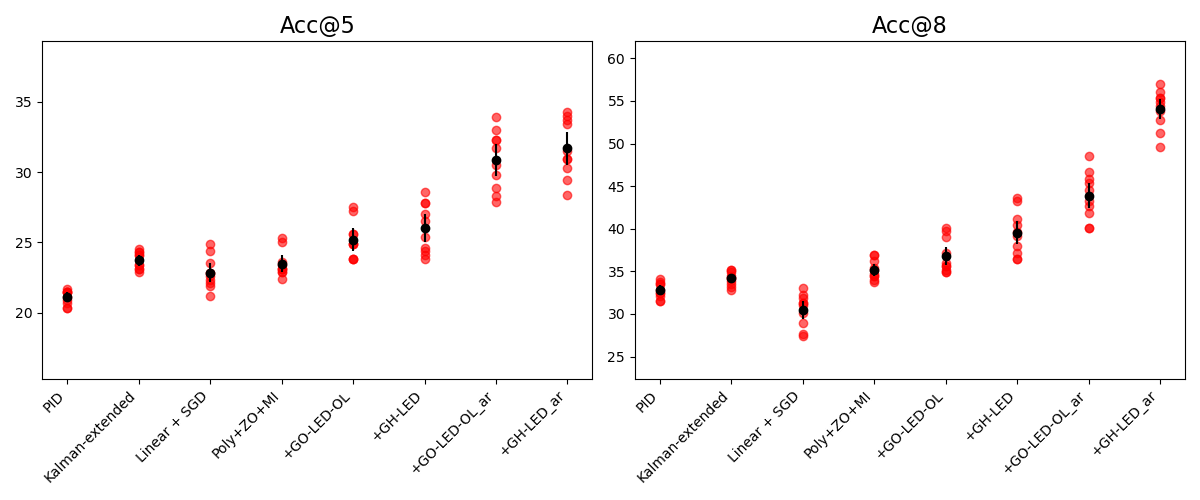}
        \label{fig:upper}
    \end{minipage}
    
    \vspace{1em}
    
    \begin{minipage}{\linewidth}
        \centering
        \bf \small Chat-UniVi-13B
        \par
        \includegraphics[width=\textwidth]{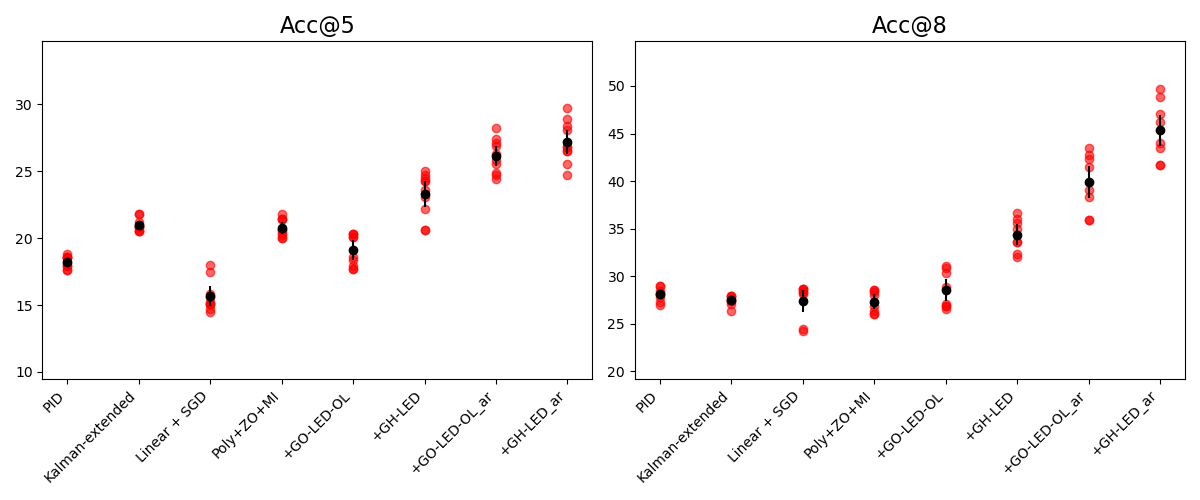}
        \label{fig:lower}
    \end{minipage}
    
    \caption{Variance over runs by method for feature identification on a budget of camera actions.}
    \label{fig:varfi}
\end{figure}

\subsection{Number of Operations}

Positioning of a camera in a 3D scene results in minimal computational overhead. The leading factor in assessing efficiencies of methods is the number of camera actions as each action incurs a constant increment in the overall time required to process inputs (see Figure~\ref{fig:time_line}). We observe that computational efficiency of all methods evaluated is a secondary concern. To provide a complete picture, reporting on mean operations per point of accuracy for each control method is presented in Figure~\ref{fig:avg_ops_acc}. Our methods require fewer CPU operations to improve the accuracy of a Video-LLaMA-13B system on FeatureID-3DS benchmark - the main benchmark that we propose to assess time efficiency.

\begin{figure}[hbt!]
    \begin{minipage}{\linewidth}
        \centering
        \textbf{ \small {Wall-clock Time by Action Count}}

        \begin{minipage}{\linewidth}
            \centering
            \begin{minipage}{\textwidth}
                \centering
                \includegraphics[width=\textwidth]{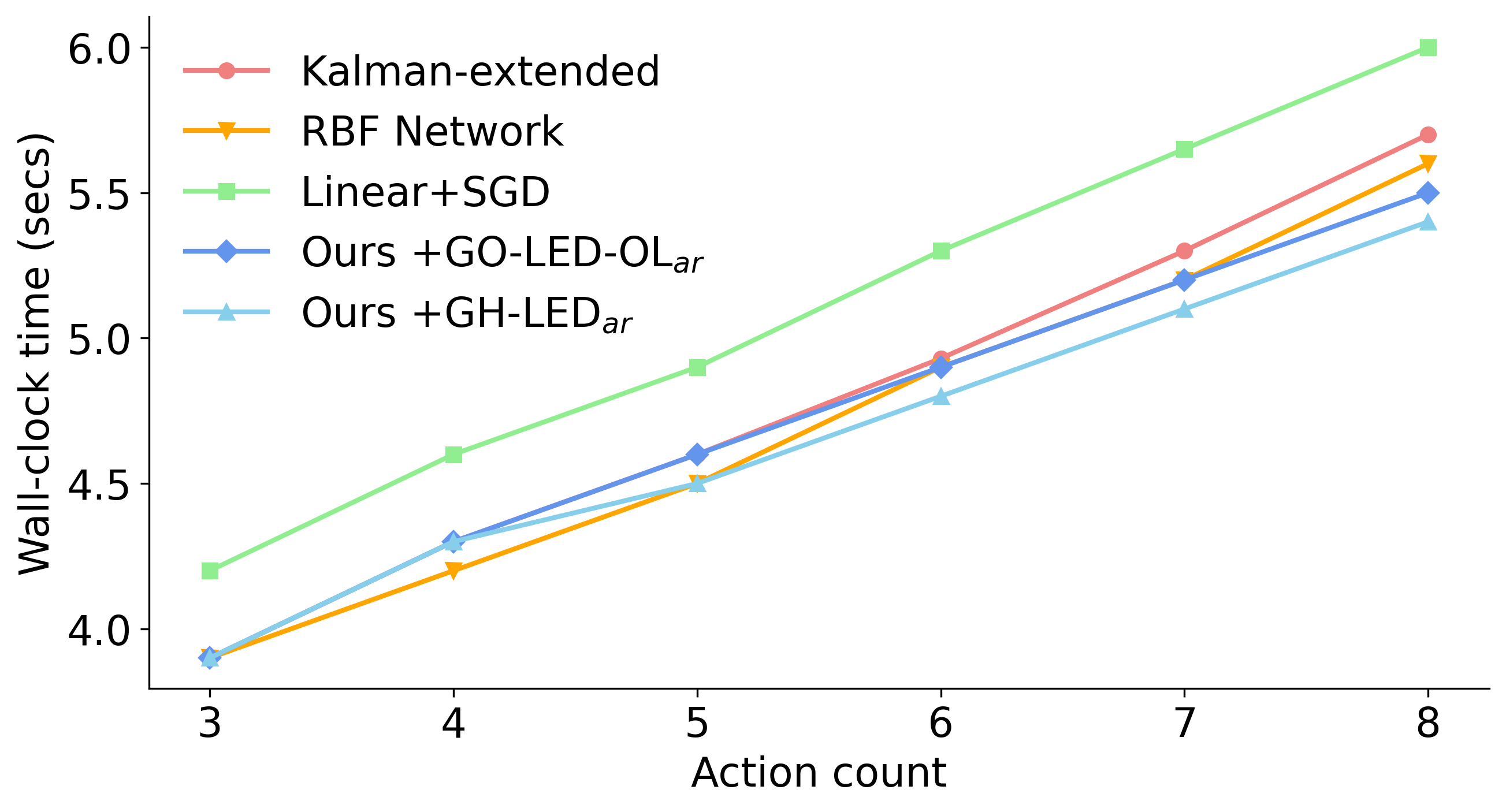}
                \label{fig:avg_ops}
            \end{minipage}
        \end{minipage}

        \caption{Wall-clock time in relation to the number of actions for a Video-LLaMA-13B system in the FeatureID-3DS benchmark is a linear progression. A single camera action equals one conversation turn with the result that action count is the primary factor in time to process inputs.}
        \label{fig:time_line}
    \end{minipage}
\end{figure}

\textbf{Results:} Our $GO\text{-}LED\text{-}OL\textsubscript{ar}$ metric assists the Poly+ZO+MI controller to prioritise viewpoints that reduce VLM errors with fewer actions (see scores for Acc@5 in Table~\ref{table:featureid}). Analysis of viewpoint replacements provided in Appendix~\ref{sec:app_results} indicates that views displaying features with strong visual prominence are prioritised. Results for the above suggest only minor variations between runs and are presented in Appendix~\ref{sec:app_results}.

\begin{figure}[hbt!]
    \begin{minipage}{\linewidth}
        \centering
        \textbf{ \small {Mean Operations / Accuracy $\downarrow$ }}

        \begin{minipage}{\linewidth}
            \centering
            \begin{minipage}{\textwidth}
                \centering
                \includegraphics[width=\textwidth]{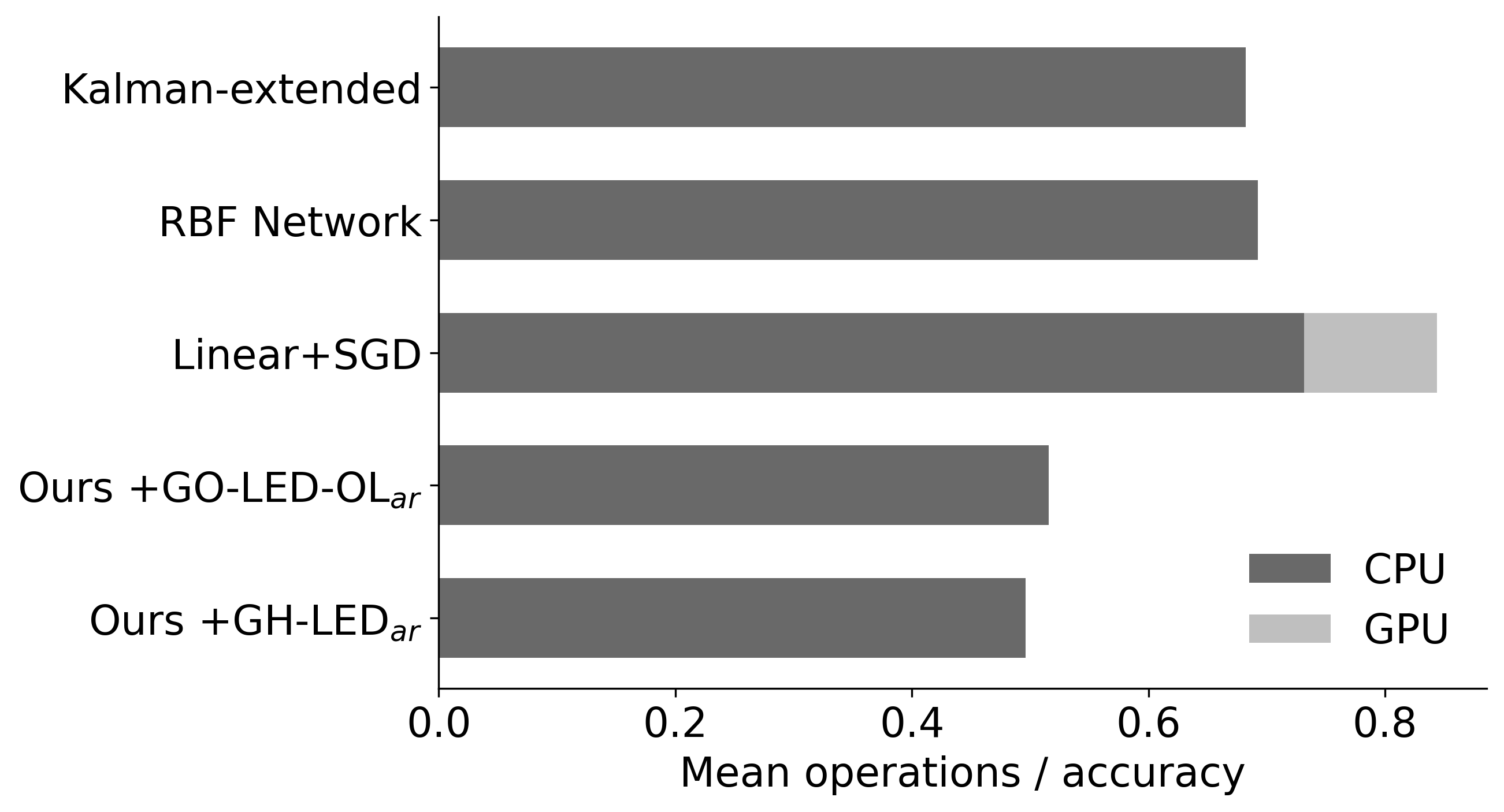}
                \label{fig:avg_ops}
            \end{minipage}
        \end{minipage}

        \caption{Computational efficiency of evaluated methods is assessed by reporting the average CPU (billions of instructions) and GPU (G-FLOPs) required to gain a percentage point of accuracy over runs in the range $[3, 8]$ actions. Scores are for each method with Video-LLaMA-13B in the FeatureID-3DS benchmark.}
        \label{fig:avg_ops_acc}
    \end{minipage}
\end{figure}

\subsection{Assessments with Language Inputs}
 We evaluate our best performing method and system pair for the FeatureID-3DS on input pairs with modified descriptions to assess the influence of phrasing. To conduct the first test, descriptions are replaced by n-grams containing random selections of alphabetic characters. These samples are generated using the method proposed by \cite{chu-etal-2022-signal}. Performance collapses with linguistic inputs that are unrelated to scenes. We also perform an assessment with descriptions consisting of an additional sentence. Declines in performance are due to working with VLMs trained on shorter textual inputs (see Table~\ref{table:sentence_cnt}).

\vspace{-0.4em}
\begin{table}[hbt!]
\begin{center}
\begin{threeparttable}
\setlength{\tabcolsep}{0.5pt}
\begin{tabular}{l l l l l}
\hline
                                & \multicolumn{2}{l }{\textbf{rnd n-grams}} & \multicolumn{2}{l}{\textbf{2 sentences}} \\ \cline{2-5}
                                & \textbf{Acc@8}    & \textbf{$\Delta$ on best}   & \textbf{Acc@8}    & \textbf{$\Delta$ on best}   \\ \hline
\textbf{Ours +GO-LED-OL\_ar}            & 15.2 & -29.3 & 35.3 & -9.2   \\ 
\textbf{Ours + GH-LED\_ar}            & 15.3 & -38.0 & 43.9 & -9.4   \\ \hline
\end{tabular}
\caption{Performance on FeatureID-3DS for the best variant (ie Video-LLaMA-13B with our controller and $MI\textsubscript{ar}$ metrics) using n-grams with alphabetic characters selected at random (rnd n-grams) and longer descriptions of 2 sentences. The delta is the drop in performance compared with the default in the paper of 1-sentence descriptions.}
\label{table:sentence_cnt}
\end{threeparttable}
\end{center}
\end{table}
\vspace{-0.8em}


\section{Additional Results from the Numerical Analysis}
\label{sec:app_numerical}

\subsection{Object Counts in Scenes}
An analysis of the ability for VLMs to discriminate 3D scenes based on the object counts in view is performed with scenes from our diagnostic dataset and additional samples. We split UC-3DS-MI into a subset with two objects and scenes with a single object. A third subset is generated with three objects to complete the data required to run the test. Weak performances for all methods on scenes with three objects are due to the inability for the open source VLMs to reason over scenes with a high number of objects.

\begin{table}[hbt!]
\centering
\begin{tabularx}{\linewidth}{X l l l}
\toprule
                                & \multicolumn{3}{l}{\textbf{Number of Objects}} \\ 
                                & \textbf{2} & \textbf{3} & \textbf{4} \\ \midrule

VLM (no control) & 0.34 & 0.31 & 0.21 \\
Ours+GH-LED$\textsubscript{ar}$ & 0.78 & 0.72 & 0.24 \\
Ours+GO-LED-OL$\textsubscript{ar}$ & 0.83 & 0.79 & 0.26 \\ 
\bottomrule
\end{tabularx}
\caption{Results on scenes grouped into sets based on the counts of polygon objects in view. Tests are run with a standalone Video-LLaMA-13B (VLM with no control) and our Poly+ZO+MI controller with the active regret measures $GH\text{-}LED\textsubscript{ar}$ and $GO\text{-}LED\text{-}OL\textsubscript{ar}$.}
\label{table:obj_cnt}
\end{table}

\subsection{Posterior Concentration} 
 Sensitivity in relation to combined visual and linguistic outputs is dependent on the combination of mixture components in computing MI. The advantages of negating overlaps between variables are displayed in the distances between scores for $MI\text{-}ar$. This relation between minimising regret and sensitivity in MI on limited samples motivates a diagnostic on statistical complexity in model optimisation. We study changes in the parameters of a logistic regression model estimated in a framework for Gibbs sampling. Our metric is the difference between maximum and minimum posterior concentrations that starts with samples from the conditional distribution $p(X, y) = p(rX \mid y) \cdot p(y)$: 
\begin{equation}
\Delta y = \max(y_t) - \min(y_t)
\end{equation}
where
\begin{equation}
\begin{split}
y_t & \sim \text{Bernoulli}\left(\frac{1}{1 + e^{-X_{t-1}}}\right), \quad p(X_t y_t) = \\
& p(y_t \mid X_t \cdot \text{Bernoulli}\left(\frac{1}{1 + e^{-r_t}}\right)
\end{split}
\end{equation}
Model selection is determined by the low model bias and theoretical guarantees provided by logistic regression when the target for real-based measures is $\{0, 1\}$~\citep{efron1975efficiency}. Our diagnostic examines variance in the posterior $\beta$ of the model while a logistic regression is applied to cumulative counts of paired MI estimates in increments of $6$ samples. PC dispersion is a point difference over the inverse of the standard deviation $\sigma$ for $\beta$ 
\begin{equation}
\text{PC Dispersion} = \frac{1}{\sigma(p(\beta | X, y))}
\end{equation}
where $p(\beta | X, y) \propto p(y | X, \beta) p(\beta)$. Our end measure quantifies absolute dispersion of variance in the model as the number of input samples increases. Posterior concentration is proportional to model stability supplied by the metric over the range of sample sizes from the increment when at least one of each of $\{0, 1\}$ labels is recorded. 

\begin{table}[hbt!] 
\begin{center}
\begin{tabular}{l c c}
\hline
\bf \small PC Dispersion  $\downarrow$
& \multicolumn{1}{c}{\textbf{MI}} & \multicolumn{1}{c}{\textbf{MI-ar}} \\ \cline{1-3} 
{\textbf{GO-LED-OL}}      & 298 & 23\\ 
{\textbf{GH-LED}}   & 162 & 55  \\ 
\hline
\end{tabular}
\end{center}
\caption {Our novel PC dispersion measure demonstrates the reduced levels of variance for MI metrics computed with active regret minimisation.}
\label{table:pcd}
\end{table}

A low value on the change in posterior concentration for $MI\text{-}ar$ methods in Table~\ref{table:pcd} indicates stable updates in the beta parameter of the model over a run. Sensitivity for $MI\text{-}ar$ variants to information (see Table 1 in the main paper) supports application in scenarios with limited demonstrations and online feedback.

Plots to illustrate the differences in posterior concentration for variants of our metrics are provided in Figure~\ref{fig:hist_go_led_cv_pcd}. Posterior concentration is proportional to model stability supplied by the metric over the range of sample sizes from the increment when at least one each of $\{0, 1\}$ labels is recorded. A low value on the change in posterior concentration for $MI\text{-}ar$ methods indicates stable updates in the beta parameter of the model over a run.

\subsection{Model Stability and Number of Data Samples}
An additional diagnostic quantifies the stability of a model fitted using MI variants with and without active regret minimisation (see Figure~\ref{fig:mean_scores}). Given the low availability of data for 3D scenes paired with language descriptions, a desirable property is to enable fitting a model with a minimum number of samples from dataset $\mathcal{D}$ that predict the boolean target of system responses. We design the analysis in the form of direct comparison to demonstrate the contribution of MI methods in training a model to identify viewpoints commensurate with the quantity of information content presented by a scene.

\begin{figure}[hbt!]
    \begin{tcolorbox}[
        colframe=black,colback=white,boxrule=1.0pt,width=\linewidth,
        left=0pt,right=0pt,top=0pt,bottom=0pt,arc=0pt
    ]
        \begin{minipage}{\linewidth}
            \centering
            \underline{\bf \small PC Dispersion $\downarrow$}
            \vspace{1em}

            \begin{minipage}{\linewidth}
                \centering
                \begin{minipage}{0.45\textwidth}
                    \centering
                    \includegraphics[width=\textwidth]{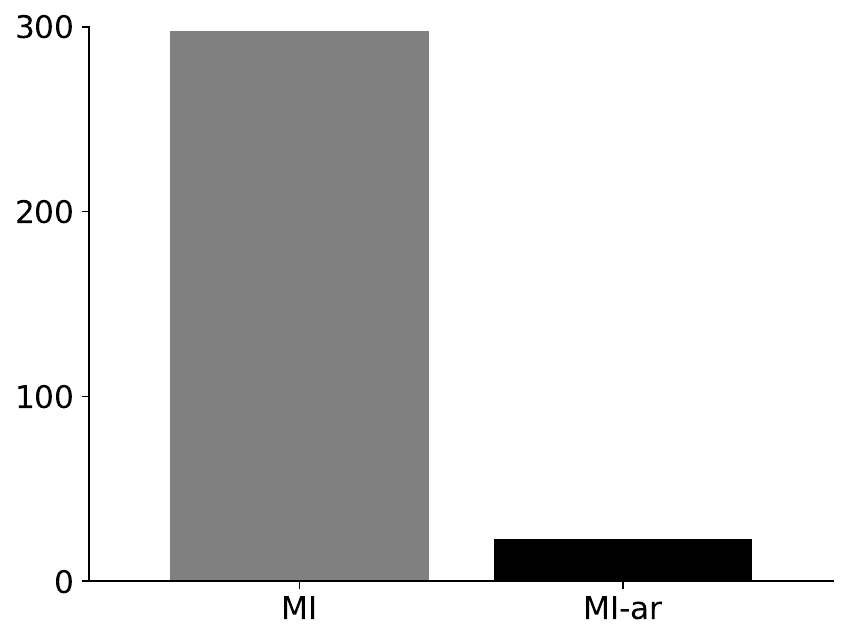}
                    \label{fig:cent2a}
                \end{minipage}
                \hfill
                \begin{minipage}{0.45\textwidth}
                    \centering
                    \includegraphics[width=\textwidth]{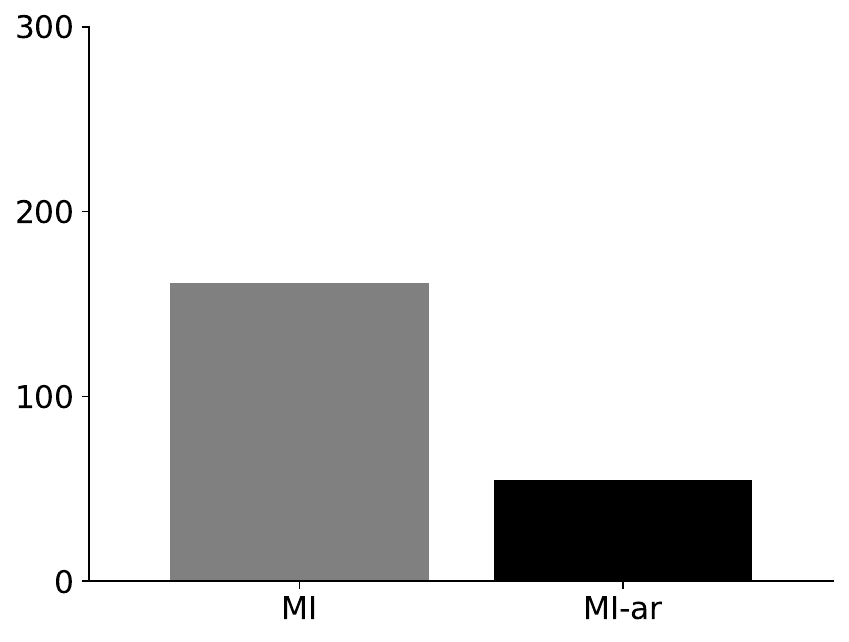}
                    \label{fig:cent2b}
                \end{minipage}

                \bf \small GO-LED-OL
                \hspace{1.8cm}
                \bf \small GH-LED
            \end{minipage}

            \caption{Variants of the $GO\text{-}LED\text{-}OL$ and $GH\text{-}LED$ metrics with 
            active regret minimisation $MI\text{-}ar$ are compared against versions calculated with standard MI on PC dispersion. Our diagnostic is derived with Gibbs sampling and measures variance in models receiving different numbers of samples.}
            \label{fig:hist_go_led_cv_pcd}
        \end{minipage}
    \end{tcolorbox}
\end{figure}

\begin{figure}[hbt!]
    \begin{minipage}{\linewidth}
        \centering
        \underline{\bf \small Mean Scores} \\[0.7em]
        \begin{minipage}{\linewidth}
            \centering
            \begin{minipage}{0.45\textwidth}
                \centering
                \includegraphics[width=\textwidth]{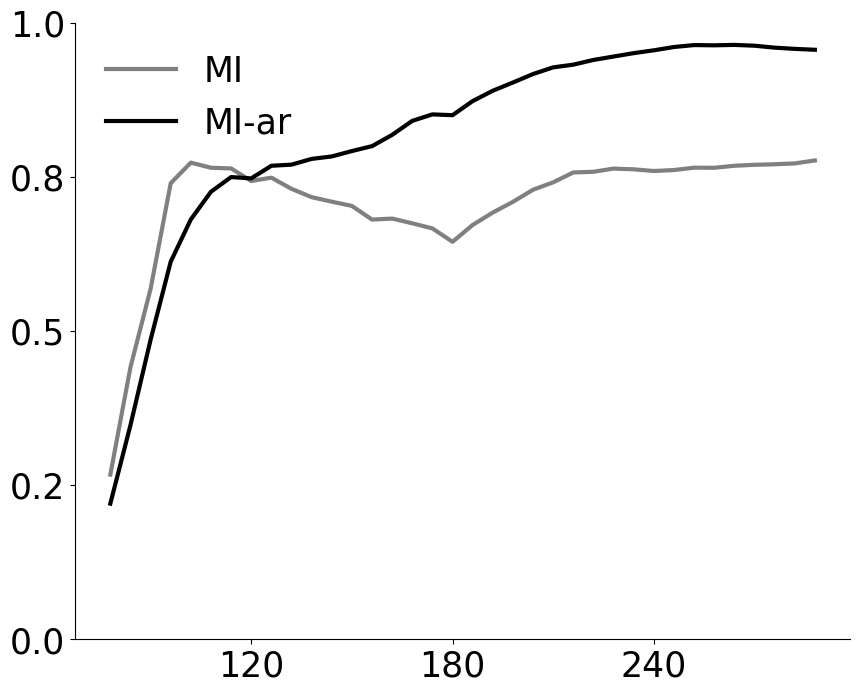}
                \label{fig:cent2}
            \end{minipage}
            \hfill
            \begin{minipage}{0.45\textwidth}
                \centering
                \includegraphics[width=\textwidth]{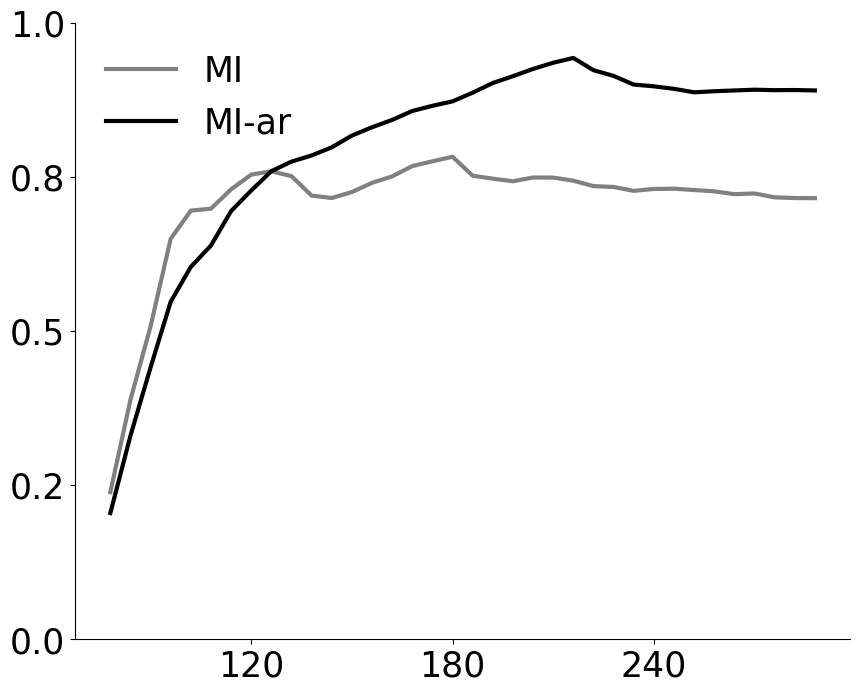}
                \label{fig:cent1}
            \end{minipage}
        \end{minipage}

        \bf \small GO-LED-OL \hspace{2.7cm} \bf \small GH-LED
        \caption{Model stability in relation to the number of data samples. Variants with $MI\text{-}ar$ show a stable progression in measuring information content.}
        \label{fig:mean_scores}
    \end{minipage}
\end{figure}



\section{Additional Details on the Method}
\label{sec:app_method}

\subsection{Controller}

We provide the full specification for our controller to predict camera actions. Assuming a prior process where a VLM returns predictions on a sample scene $\varsigma$ and language description $l$ from dataset $\mathcal{D}$, a set of functions models these predictions and MI-ZO to output a sequence of camera actions. The actions position the camera to return a set of viewpoints defined by $(X, Y)-$ and $z-$axes. Functions are described by subprocess and in a figure with detailed specifications on individual operations at each stage. 

\textbf{Interval Estimation} To enable working directly with continuous MI scores, a function converts values to intervals $i$ with selection of interval size based on entropy maximisation
\begin{equation}
    X_{MI} \leftarrow \text{interval}_i = \arg\max_{\Delta_i} H(x_k, \Delta_i)
\end{equation}
on proposals for cut points generated with Halton sequences~\cite{halton2005quasi}.  We prefer this estimation method to a random number generator to reduce variation. Variance is further reduced by computing a mean over cut point proposals returning maximum entropy. The result is an automatic process that requires no set interval widths in experiments.

\textbf{Component Models} Two component models $CM$ filter data in the responses with the proxy labels $\hat{Y}$ generated during the first step in Algorithm~\ref{alg:cmu}. During the correction round, proxy labels replace the boolean values on the match of a viewpoint with its description that are provided as actual labels during the measurement round. To limit processing time, modeling is performed in both models with iterative least squares. 

Component Model 1 increments the probability of prediction error $\mathbb{P}^{[x \neq x']}$ by the system for a viewpoint $Vp$ where an error was marked in the prior round. Coefficients are measured for the set of decisions $Dec$ with corresponding viewpoint labels and the demonstration data are updated. Test error rates and score-based measures on scene attributes are processed by Component Model 2 that ranks $z-$axis levels $Dim^z$ for each viewpoint. Traces of the covariance matrices $tr$ in each component model are retained as indicators of model fit $\lim_{n \to \infty}$. 

\textbf{Central Unit} Outputs are passed to the Central Unit of the controller. Acceptance of VLM feedback and decisions on the $z-$axis level by viewpoint are modeled using the covariance traces normalised and converted to an inverse factor. View-level data are passed to update element values over a low dimensional representation of the scene in the form of an interaction matrix.

\textbf{Interaction Matrix} An interaction matrix $\mathbb{A}$ is a graph $XYZ$ and a superset over the scene composed of an element drawn from set $XY$ (element $xy$) with set $Z$ (element $z$). Elements in each case are unique instances by position. Set $XY$ is a cyclic graph consisting of adjacent nodes with a direct edge to any element in the fully connected graph of set $Z$. Pendant nodes in the factor graph $XY$ define the adjacency matrix of graph $XYZ$. A graph product will result in the targeted structure for the graph and a strong product $\boxtimes$ provides a specific advantage in preserving the connectivity of the factors in the edges of any vertex set.

\onecolumn

\begin{center}
\begin{minipage}{\textwidth}
\centering
\begin{algorithm}[H]
\caption{Controller Module}
\label{alg:cmu}
\begin{algorithmic}
\STATE {\bfseries Input:} Visual scene $\varsigma$, Description $l$, Correctness $y$
\STATE Generate proxy labels on few samples with derivative-free estimation
\FOR{views in scene $\varsigma$} 
    \STATE $MI \leftarrow g_{MI-ZO}(H(Dim_n), Y)$ \hspace{68mm} $\triangleright$ Compute MI scores with MI-ZO
    \STATE $X_{MI} \leftarrow \text{interval}_i = \arg\max_{\Delta_i} H(x_k, \Delta_i)$
    \STATE Return proxy labels by viewpoint
\STATE \textbf{return} $\hat{Y}$
\ENDFOR
\FOR{$\varsigma$ in $\mathcal{D}$}
    \STATE $CM_1 \leftarrow 0$
\ENDFOR
\STATE 
\STATE \textbf{Function} Component Model 1 
\STATE {\bfseries Input:} Decision $Dec$, proxy labels $\hat{Y}$
\FOR{views in scene $\varsigma$}
    \STATE $\{\mathbb{P}_i\}_{i=1}^{m} \leftarrow \min_{Dec} \sum_{k=1}^{n} \left(Dec(Vp_k) - \hat{y}_k\right)^2$ \hspace{61mm} $\triangleright$ Model prediction errors
    \STATE $tr\left(\{R_k\}_{k=1}^{n}\right) \leftarrow tr\left(\left\{ \left(Dec(Vp_k) - \hat{y}_k\right)^2 \right\}_{k=1}^{n}\right)$ \hspace{69mm} \(\triangleright\) Compute trace
\STATE \textbf{return} Set of error probabilities per viewpoint $\mathbb{P}^{[x \neq x']}$, Metric on model fit $tr\left(\{R_j\}_{j=1}^{n}\right) \mathbb{P}^{[x \neq x']}$
\ENDFOR

\STATE 
\STATE \textbf{Function} Component Model 2 
\STATE \textbf{Input:} $z-$axis levels $Vp_{Dim^z}$, proxy labels $\hat{Y}$
\FOR{views in scene $\varsigma$}
    \STATE $\{CS_j\}_{j=1}^{n} \leftarrow \min_{Vp_{Dim^z}} \sum_{k=1}^{m} \left(Vp_{{Dim^z}_k} - \hat{y}_k\right)^2$ \hspace{53mm} \(\triangleright\) Model $z-$axis rankings
    \STATE $tr\left(\{R_j\}_{j=1}^{n}\right) \leftarrow tr\left(\left\{ \left(Vp_{{Dim^z}_j} - \hat{y}_j\right)^2 \right\}_{j=1}^{n}\right)$ \hspace{69mm} $\triangleright$ Compute trace
\STATE \textbf{return} Confidence in $z-$axis level predictions $CS_j$, Metric on model fit $tr\left(\{R_j\}_{j=1}^{n}\right) CS_j$
\ENDFOR
\STATE 
\STATE \textbf{Function} Central Unit 
\STATE \textbf{Input:} Error probability per viewpoint $\mathbb{P}^{[x \neq x']}$, Confidence in $z-$axis level predictions $CS_j$, \\ Metrics on model fit $(\mathbb{P}^{[x \neq x']}, tr\left(\{R_j\}_{j=1}^{n}\right) CS_j)$
\FOR{views in scene $\varsigma$}
    \STATE $\sum_{i=1}^{p} \mathbb{P}_i \cdot \left(\sum tr\left(R_j^2\right) \right)^{-1} \leq \tau$ \hspace{85mm} $\triangleright$ Viewpoint operations 
    \STATE $\sum_{j=1}^{n} CS_j \cdot \left(\sum tr\left(R_j^2\right) \right)^{-1} \leq \tau$ \hspace{72mm} $\triangleright$ Confidence score operations
    \STATE $\{\mathfrak{a}_i\}_{i=1}^{p} \leftarrow \boxtimes\left(\mathbb{A}(XY, Z), (Out_x, Out_y)\right)$ \hspace{64mm} $\triangleright$ Update interaction matrix
\ENDFOR
\STATE {\bfseries return} Set of camera actions $\{\mathfrak{a}_1, \mathfrak{a}_2, \ldots, \mathfrak{a}_n\}$
\end{algorithmic}
\end{algorithm}
\end{minipage}
\end{center}

\newpage


\section{Theoretical Analysis}
\label{sec:app_proofs}

We present details of the theoretical analysis and proofs for the MI-ZO algorithm. The section begins by proving nonpositivity when estimating multi-information, presents the reformulation of the problem in our approach using a duality, moves the analysis into the online setting, and concludes with a proof to Theorem~\ref{theorem:miexp}, which states that a function exists that places an upper bound on elements in component variables that reduce the expressivity of multivariate combinations:

\begin{theorem}[Function with upper bound on nonpositive contribution]\label{theorem:miexp}
There exists a function that combines a set of $n>2$ entropies $H(Dim_n)$ with an upper bound on the nonpositive contribution of reductive units to an output estimate MI that is constant with a bound on the inner product of the vector on total units and the vector of observed information.
\end{theorem}

\subsection{Multi-information}
In this section, we define the conditions and detail a proof for when multi-information is nonpositive. We default to multi-information as a term for multiple mutual information terms~\cite{studeny1987asymptotic}. We use a semicolon to indicate a function that is not symmetric. In all instances, joint entropy $(x, y) < \infty$.

We begin with the interpretation of mutual information as differences in Shannon entropies~\cite{mcallester2020formal}:
\begin{definition}
Mutual information for a variable $x$, $\mathbb{M}(x: x)$ is equal to the entropy of $x$ $(H(x))$. For discrete random variables $(x, y)$, mutual information $\mathbb{M}(x, y)$ is the sum of entropy $x$ and entropy $y$ minus the joint entropy of $(x, y)$: 
\begin{equation}
\mathbb{M}(x:y)=H(x) + H(y) - H(x,y). 
\end{equation}
\end{definition}

\begin{definition}
For a pair of variables in the set $X$, the mutual information of any member $H(x_i)$ is in the distribution over $x_i$ in relation to $y$. To limit redundancy between members of $X$, multi-information MI between multiple input variables and the target $y$ is a combination of single and joint entropies~\cite{te1980multiple}:
\begin{equation}
MI(X:y) = H(x|y) + H(y) - H(X, y) - [H(x_1, y) + H(x_2, y) - H(x_1) - H(x_2)]
\end{equation}
\end{definition}

\begin{definition}
Additional joint entropies over all pairs of variables are included for $X$ with more than $2$ members. For any $n$, multi-information follows the chain rule and is summarised as
\begin{equation}
MI(X_n) = \sum_{k=1}^n (-1)^{k+1} \left( \sum_{I \subseteq {1, 2, ..., n}, |I| = k} H(X_I | Y) \right)
\end{equation}
\end{definition}

\begin{lemma}[Nonpositivity]\label{lemma:nonpos}
Let each member of $X$ be a set $A_n$, then the intersection of $A$ is a positive or negative value when members of $X$ are greater than $1$.
\end{lemma}

\begin{proof}[Proof of Lemma~\ref{lemma:nonpos}]
First consider that when $X$ contains exactly two elements, then the intersection by the inclusion-exclusion principle for $n\text{=}3$ is:
\begin{equation}\label{eq:condition1}
|X_1 \cup X_2 \cup X_3| = |X_1| + |X_2| + |X_3| - |X_1 \cap X_2| - |X_2 \cap X_3| - |X_1 \cap X_3| + |X_1 \cap X_2 \cap X_3|
\end{equation}
when $X = \{x_1, x_2, \dots, x_n\} \text{where } |x_i| < \infty \text{ for all } i = 1, 2, \dots, n$
and
\begin{equation}
\text{If } \exists I \subseteq \{1, 2, ..., n\}, |I| > 1 \text{ such that } H(X_I) > \sum_{i \in I} H(x_i) - H(X_I | Y), \text{ then } MI(X;Y) \not\geq 0
\end{equation}

Then if the condition is met:
\begin{equation}
H(X_1 \cup X_2 \cup X_3) \leq H(X_1) + H(X_2) + H(X_3) - H(X_1 \cap X_2) - H(X_1 \cap X_3) - H(X_2 \cap X_3) + H(X_1 \cap X_2 \cap X_3)
\end{equation}
the case in Equation~\ref{eq:condition1} extends by the submodularity of entropy~\cite{10.5555/3020336.3020377} to: 
\begin{equation}
H(X_1 \cap X_2) + H(X_1 \cap X_3) + H(X_2 \cap X_3) + H(X_1 \cup X_2 \cup X_3) \leq H(X_1) + H(X_2) + H(X_3) + H(X_1 \cap X_2 \cap X_3)
\end{equation}
\end{proof}

\subsection{Separation by Hyperplane}

In this section, we state assumptions and introduce terms for separable regions in a space over the variables with finite dimensions. Our problem of identifying information by unit type when $n>2$ is converted to a duality. Definitions and a lower bound are provided for a function that separates constituent units of inputs.

\begin{assumption}
 For the set of elements in $X$ in a multivariate function that results in a single output, note that for $H(X_1), H(X_2), \ldots H(X_n)$, each input is conditionally independent from the other inputs given certain subsets of the variables~\cite{te1980multiple}. We assume that each member consists of multiple units $\mathfrak{B} \subseteq \mathbb{R}^D$.    
\end{assumption}

\begin{definition}
The state of any unit $\mathfrak{b} \in \mathfrak{B}$ is additive or reductive in relation to the product of the MI calculation. To specify these states, we use the term $\mathfrak{u}$ to denote $Up$ and $\mathfrak{d}$ to denote $Down$.
\end{definition}

\begin{definition}
A real vector space with all $\mathfrak{B}$ contains affine subspaces with dimension $n-1$ in $\mathbb{R}^n$ creating regions defined by the inequalities 
\begin{equation}
\vec{w}_1 \, x_1 + \vec{w}_2 \, x_2 + \cdots + \vec{w}_n \, x_n > \mathfrak{c}
\end{equation}
and 
\begin{equation}
\vec{w}_1 \, x_1 + \vec{w}_2 \, x_2 + \cdots + \vec{w}_n \, x_n < \mathfrak{c}
\end{equation}
where $\vec{w}$ is a weight parameter of the hyperplane in $\forall 2 + \cdots{\vec{w}} \in 2 + \vec{W}, \exists \vec{w} \in \vec{W} : \vec{w} \neq 0$ and $\mathfrak{c}$ is a constant term. 
\end{definition}

\begin{definition}
The margin $\gamma$ is the distance between $\mathfrak{b}_{\mathfrak{u}}^*$ and $\mathfrak{b}_{\mathfrak{d}}^*$ where
\begin{equation}
(\mathfrak{b}_{\mathfrak{u}}^*, \mathfrak{b}_{\mathfrak{d}}^*) = \operatorname*{arg\,min}_{\mathfrak{b}_{\mathfrak{u}} \in \mathfrak{Bu}, \mathfrak{b}_{\mathfrak{d}} \in \mathfrak{Bd}} d(\mathfrak{a})
\end{equation}
and $d$ is the scaling of the vector $\vec{v}$ that is orthogonal to the hyperplane applied as $\mathfrak{b}_{\mathfrak{u}}^* = \mathfrak{b}_{\mathfrak{d}}^* + d \frac{\vec{v}}{\|\vec{v}\|}$.
\end{definition}

\begin{remark} 
Note that $\gamma$ can also be expressed in terms of the weight vector $\vec{w}$, which is normal to the hyperplane:
\begin{equation}
\gamma = \frac{2}{\|\vec{w}\|}
\end{equation}
where $\vec{w}^T x + b = 0$ defines the separating hyperplane, and $\vec{w}^T$ is the transpose of $\vec{w}$.
\end{remark}

\begin{theorem}[Function to maximise the margin]\label{theorem:optsep}
For any $X$, an optimisation process $f(x) = \vec{w}^T x + b$ exists to solve the minimax form of deriving the margin using the hyperplane in $\mathbb{R}^D$ that maximises the distance between $\mathfrak{b}_{\mathfrak{u}}$ and $\mathfrak{b}_{\mathfrak{d}}$:
\begin{equation}
\max_{\mathfrak{b}_{\mathfrak{d}} \in \mathfrak{Bd}} \min_{\mathfrak{b}_{\mathfrak{u}} \in \mathfrak{Bu}} f(\mathfrak{b}_{\mathfrak{u}}, \mathfrak{b}_{\mathfrak{d}}) = \min_{\mathfrak{b}_{\mathfrak{d}} \in \mathfrak{Bd}} \max_{\mathfrak{b}_{\mathfrak{u}} \in \mathfrak{Bu}} f(\mathfrak{b}_{\mathfrak{d}}, \mathfrak{b}_{\mathfrak{u}}).
\end{equation}
\end{theorem}

\begin{definition}
A full version of the function $f(x) = \vec{w}^T x + b$ for any $X$ aims to 
\begin{equation}
\begin{aligned}
& \text{maximize} && t \\
& \text{subject to} && \vec{w}^T x_i - bias \geq t, \quad i = 1, \ldots, N, \\
& && \vec{w}^T y_i - bias \leq -t, \quad i = 1, \ldots, M, \\
& && \|\vec{w}\|_2 \leq 1.
\end{aligned}
\end{equation}

\end{definition}

\begin{proof}[Proof of Theorem~\ref{theorem:optsep}]

We will convert the primal problem into a dual and derive a lower bound on the latter following classic Lagrangian duality~\cite{rockafellar1974conjugate}. First if we apply a multiplier $\lambda$ for the constraint $||\vec{w}||^2_2 \leq 1$, a sum of Lagrange multiplications for each group $\sum_{i=1}^M u_i y_i$ and $\sum_{i=1}^M v_i y_i$ can be derived as follows:
\begin{equation}
\begin{aligned}
L(\vec{w}, bias, \gamma, u, v, \lambda) = &-t + \sum_{i=1}^{D} u_i (t + bias - \vec{w}^T x_i) \\
&+ \sum_{i=1}^M v_i (t - bias + \vec{w}^T y_i) + \lambda (\|\vec{w}\|_2^2 - 1)
\end{aligned}
\end{equation}
where $\vec{w}$ and $bias$ are the weight and bias parameters of the hyperplane, $M$ is the number of elements in the group, $u_i$ is the multiplier for the constraints $\vec{w}^T x_i - bias \geq t$, and $v_i$ is the multiplier for the constraints $\vec{w}^T x_i - bias \geq t$. 

We have formulated the primal problem with weak duality and can now progress to the second stage of deriving from this new form a lower bound by finding the supremum of the factor $\sup_{\lambda \in \mathbb{R}^m} q(\lambda)$
\begin{equation}
\label{equation:dual}
\begin{aligned}
& \text{maximize} && \sup_{\lambda \in \mathbb{R}^m} q(\lambda) > q(\bar{\lambda}) = -1 \cdot \bar{\lambda} = D \cdot \text{eigmin}(Q) \\
& \text{subject to} && \sum_{i=1}^D u_i = \frac{1}{2}, \quad u_i \geq 0, \\
& && \sum_{i=1}^M v_i = \frac{1}{2}, \quad v_i \geq 0.
\end{aligned}
\end{equation}
noting that $Q$ is the matrix of the quadratic form of the problem and $q(\bar{\lambda}) = -1 \cdot \bar{\lambda} = n \cdot \text{eigmin}(Q)$ is the lower bound derived from $Q$.
\end{proof}

\subsection{Online Setting}
We start by defining a formulation of the separating function proposed in the last section in an online learning setting in the form $Out^t = Out^{t-1} - \eta^t L(Out^{t-1})$ where some outcome $Out$ is iterated using step size $\eta$. We specify optimisation on real values and prove an algorithm with active regret minimisation on information. 

\begin{definition}
We underline the difference of full and partial information settings characterised by a specified number $\tau$ of rounds where the function has access to observed information denoted as $\alpha$:
\begin{equation}
\alpha_t = \begin{cases} 
1 & \text{if } t \leq \tau \\
0 & \text{otherwise}
\end{cases}
\end{equation}
\end{definition}

\begin{definition}
A computational solution $\omega$ is updated by an online process where a point is mapped to the next nearest point in the convex set $C$ of potential variable values
\begin{equation}\label{eq:olsolu}
\begin{aligned}
\omega^t &= \begin{bmatrix} \vec{w}^t \\ bias^t \\ \gamma^t \end{bmatrix} \leftarrow proj_c \left( \begin{bmatrix} \vec{w}^{t-1} \\ bias^{t-1} \\ \gamma^{t-1} \end{bmatrix} - \eta_t \cdot \alpha_t \cdot \nabla_{\omega} L(\omega^{t-1}) \right) \; \text{, where} \;
\nabla_{\omega} L(\beta^{t-1}) = \begin{bmatrix}
\nabla_{\vec{w}} L \\
\nabla_{bias} L \\
\nabla_\gamma L
\end{bmatrix}
\end{aligned}
\end{equation}

and gradients $\{\nabla_{\vec{w}} L, \nabla_{bias} L, \nabla_{\gamma} L\}$ are given by $\sum_{i=1}^N u_i \nabla (\vec{w}^{T} x_i - bias - \gamma) + \sum_{i=1}^M v_i \nabla (\vec{w}^{T} y_i - bias + \gamma)$ with $\nabla_{\vec{w}} L$ including the additional regularisation step $\lambda \nabla_{\vec{w}} (\|\vec{w}\|_2^2 - 1)$.
\end{definition}

\begin{lemma}
\label{lemma:regdist}
Let regret be the difference in Expectation $\mathbb{E} [\pi_t(out_t)]$ and $\mathbb{E} [\pi_t^*(out_t)]$, where $\pi^*$ is the optimal policy in the set $\Pi$. Then for every $t$ to $\tau$, 
\begin{equation}
\text{Regret}_t = \text{Regret}_{t-1} \cdot (1 - \eta_t \cdot \frac{\mathfrak{b}_{\mathfrak{u},t-1} - \mathfrak{b}_{\mathfrak{d},t-1}}{\mathfrak{B}_{t-1}}) + \text{constant}.
\end{equation}
\end{lemma}

\begin{remark}
We have the equivalence between regret and the distance between $\mathfrak{b}_{\mathfrak{u}}$ and $\mathfrak{b}_{\mathfrak{d}}$ such that an indicator function $\mathbb{I}$

\begin{equation}\label{eq:dualreg}
\text{Regret} \equiv \sum_{i=1}^{M} \left[ \mathbb{I}\left(\vec{w}^{\top} \mathfrak{b}_{\mathfrak{u},i} + bias \leq 0 \right) \cdot \mathbb{I}\left(\mathfrak{b}_{\mathfrak{d},i} = 1\right) + \mathbb{I}\left(\vec{w}^{\top} \mathfrak{b}_{\mathfrak{u},i} + bias > 0 \right) \cdot \mathbb{I}\left(\mathfrak{b}_{\mathfrak{d},i} = -1\right) \right].
\end{equation}

then regret in relation to the dual form in Equation \ref{equation:dual} is

\begin{equation}
\begin{aligned}
\text{Regret}^t \sum_{i=1}^{M} &= \left[ \left( \max(0, \vec{w}^{\top} x_i + bias) \cdot (1 - y_i) + \max(0, -(\vec{w}^{\top} x_i + bias)) \cdot (1 + y_i) \right) \cdot \alpha_t \right] \\
&+\alpha_t \frac{1}{2} \left( \sum_{i=1}^{M} \left| \vec{w}^{\top} x_i + bias \right| \cdot \alpha_t \cdot \left( \mathfrak{b}_{\mathfrak{u},i} - \mathfrak{b}_{\mathfrak{d},i} \right) \right).
\end{aligned}
\end{equation}

\end{remark}

\begin{proof}[Proof of Lemma~\ref{lemma:regdist}]
The proof of Lemma~\ref{lemma:regdist} builds on the definition of redundancy in universal compression and a subsequent proof given in \cite[Ch.13]{cover2006elements}. We present a worked example of the above with a set of inputs $X = {x_1, x_2, \ldots, x_n}$ following an unknown distribution $p_\theta$. For every $p$ in $P$, there is an associated prior distribution $h \in \{h_1, h_2, \ldots, h_m\}$ over parameters $\Theta$. The definition for the $Redundancy(p_\theta, q)$ of a code is estimated using the Kullback-Leibler divergence
\begin{equation}
Redundancy(p_{\theta}, q) = KL(p_{\theta} \| q) = \sum_x p_{\theta}(x) \log \left(\frac{p_{\theta}(x)}{q(x)}\right)
\end{equation}
where $q$ is the inferred distribution. The aim is to minimise the minimax formulation of $Redundancy$ for all $p_\theta$:
\begin{equation}
Redundancy^* = min_q max_\theta KL(p_\theta \| q).
\end{equation}
We assume that $X = \{1, 2, 3\}$ and $\theta = \{0, 1\}$ with a probability associated when observing $X=2$ in both distributions $p_1$ and $p_2$ denoted as $\mathbb{a}$. If $p_1 = (1 - \mathbb{a}, \mathbb{a}, 0)$, then the KL divergence between $q$ and $p_1$ is
\begin{equation}
KL(p_1 \| q) = (1 - \mathbb{a}) \log \left(\frac{1 - \mathbb{a}}{q_1}\right) + \mathbb{a} \log \left(\frac{\mathbb{a}}{q_2}\right)
\end{equation}
and for $p_2$ where $p_2 = (0, \mathbb{a}, 1 - \mathbb{a})$:
\begin{equation}
KL(p_2 \| q) = \mathbb{a} \log \left(\frac{\mathbb{a}}{q_2}\right) + (1 - \mathbb{a}) \mathbb{a} \left(\frac{1 - \mathbb{a}}{q_3}\right).
\end{equation}
The optimal $q$ for the problem is the candidate with minimum KL for both $p_1$ and $p_2$ is $q_1 = q_3 = \frac{1 - \mathbb{a}}{2}$ and $q_2 = \mathbb{a}$
\begin{equation}
q = \left(\frac{1 - \mathbb{a}}{2}, \mathbb{a}, \frac{1 - \mathbb{a}}{2}\right).
\end{equation}

We now turn to the information capacity of a channel $Cap$. In the minimax case, we have
\begin{equation}
\label{equation:mmcap}
Cap = \max_{h(\theta)} \min_q \sum_\theta h(\theta) KL(p_\theta \| q)
\end{equation}
recalling that $q$ is inferred. Then we derive $q_{h}(x)$ over all $\Theta$ as follows:
\begin{equation}
q_{h}(x) = \sum_{\theta = 1}^m h(\theta) p_{\theta}(x).
\end{equation}
Let the mutual information $\mathbb{M}$ representing the change in uncertainty on observing the input $X$ be
\begin{equation}
\mathbb{M}_{h}(\theta; X) = \sum_\theta h(\theta) \sum_x p_\theta(x) \log \left(\frac{p_\theta(x)}{q_{h}(x)}\right).
\end{equation}

Given the inferred distribution
\begin{equation}
q_{h}(x) = \left(\frac{1 - \mathbb{a}}{2}, \mathbb{a}, \frac{1 - \mathbb{a}}{2}\right)
\end{equation}
the $\mathbb{M}$ for $\theta=1$ and $\theta=2$ is
\begin{equation}
\mathbb{M}_{h}(\theta=1; X) = h_1 \left[(1 - \mathbb{a}) \log \left(\frac{1 - \mathbb{a}}{\frac{1 - \mathbb{a}}{2}}\right) + \mathbb{a} \log \left(\frac{\mathbb{a}}{\mathbb{a}}\right)\right] = 0.5 \left[(1 - \mathbb{a}) \log 2\right]
\end{equation}
and
\begin{equation}
\mathbb{M}_{h}(\theta=2; X) = h_2 \left[\mathbb{a} \log \left(\frac{\mathbb{a}}{\mathbb{a}}\right) + (1 - \mathbb{a}) \log \left(\frac{1 - \mathbb{a}}{\frac{1 - \mathbb{a}}{2}}\right)\right] = 0.5 \left[(1 - \mathbb{a}) \log 2\right].
\end{equation}

Then finding the optimal $q^*(\theta)$ yields the maximum $Cap$. Since $h(\theta) = \{0.5, 0.5\}$ maximises the mutual information, the optimal distribution also is $h^*(\theta) = \{0.5, 0.5\}$:
\begin{equation}
q_{h^*}(x) = \left(\frac{1 - \mathbb{a}}{2}, \mathbb{a}, \frac{1 - \mathbb{a}}{2}\right)
\end{equation}
\end{proof}

\subsection{Proof of Theorem 1}

\begin{remark}
    We follow \cite{krichevsky1968relation} in drawing an equivalence between redundancy and regret in the context of information theory. In the online setting, the dual form in Equation~\ref{eq:dualreg} is extended to express the relation between minimising regret and increasing the separation between $\mathfrak{b}_{\mathfrak{u}}$ and $\mathfrak{b}_{\mathfrak{d}}$ in the presence of observed information $\alpha$. To find the supremum of the infimum in the minimax form, we have
\begin{equation}\label{eq:regdual}
\begin{aligned}
\text{Regret}^*_t = \sup_{\mathfrak{b}_{\mathfrak{d}} \in \mathfrak{B}} \inf_{Q_{X_n}} \left[ \left( \max(0, \vec{w}^{\top} x_i + bias) \cdot (1 - y_i) + \max(0, -(\vec{w}^{\top} x_i + bias)) \cdot (1 + y_i) \right) \cdot \alpha_t \right] \\ + \frac{1}{2} \sum_{i=1}^M \left| \vec{w}^{\top} x_i + bias \right| \cdot \alpha_t \cdot \left( \mathfrak{b}_{\mathfrak{u},i} - \mathfrak{b}_{\mathfrak{d},i} \right)
\end{aligned}
\end{equation}
\end{remark}
where the relation of $\alpha$ to $\vec{w}$ is $\left| \langle \alpha, \vec{w} \rangle \right| \leq \| \alpha \| \cdot \| \vec{w} \|.$

\begin{proof}[Proof of Theorem 1]
We start with the framework of empirical risk minimisation $Rk$ in online learning listed in several forms by \cite{cesa2004generalization} 
\begin{equation}
Rk(\omega) = \frac{1}{M} \sum_{i=1}^{M} \left[ \max(0, \vec{w}^{\top} x_i + bias) \cdot (1 - y_i) + \max(0, -(\vec{w}^{\top} x_i + bias)) \cdot (1 + y_i) \right] \cdot \alpha_t
\end{equation}
where the parameters are scaled by $\alpha_\tau$ such that
\begin{equation}
\min_{\forall t, \; \alpha_t = 1} \left( \frac{1}{M} \sum_{i=1}^{M} f(x_i, \vec{w}, bias, y_i) \right) \leq Rk(\omega).
\end{equation}

Then in the setting where $\alpha$ is not present, the information capacity of the channel is dependent on the function $f_{mon}$ over the processes of minimising $\min Rk(\omega)$ and maximising distance $\max d$ being monotonic:
\begin{equation}
    \max_{\pi(\theta)} Cap_{\pi}(\theta; X) \propto f_{mon} \left( \min_{\forall t, \; \alpha_t = 1} \left( \frac{1}{M} \sum_{i=1}^{M} f(x_i, \vec{w}, bias, y_i) \right), \; \min Rk(\omega), \; \max d_{Up, Down} \right).
\end{equation}

Reordering the above in terms of the solution in Equation~\ref{eq:olsolu}, $f_{mon}$ holds \textit{iff} a policy $\pi^*$ that penalises the squared deviation such that the approximation of $\vec{w} + bias$ after the period $\tau$ is within $\Delta$ of the estimated $\mu$ in relation to the empirical mean $\mu_{emp}$ when $\alpha$ is present:

\begin{equation}
\begin{aligned}
f_{mon} \left( \min_{\forall t, \; \alpha_t = 1} \left( \frac{1}{M} \sum_{i=1}^{M} f(x_i, \vec{w}, bias, y_i) \right), \; \min Rk(\omega), \; \max d_{Up, Down} \right)
\iff \\ \min_{\forall pred, \; \alpha_t = 0} \left( \frac{1}{M} \sum_{i=1}^{M} \pi^*_i \cdot \left| \omega(x_i) - \frac{1}{2} (\vec{w} + bias) \right|^2 \right) \leq \Delta
\end{aligned}
\end{equation}

We note that in the above $pred$ is the period following $\tau$. The monotonicity observed in numerical results in Table 1 is supported by defining a solution $\omega_{zo}$ where the objective is to predict a policy that results in $\left( \mu' - \mu_{emp} \right)^2$. Building on the dual form of regret in Equation~\ref{eq:regdual}, we have the correlation between $\alpha$ and $\vec{w}$ such that
\begin{equation}
\alpha \cdot \vec{w} = \sum_{i=1}^{M} \alpha_i \cdot \vec{w}_i.
\end{equation}
Then by Cauchy-Schwarz, $\gamma$ is supported when constraining the inner product between $\alpha_i$ and $\vec{w}$
\begin{equation}
\gamma \leq \lambda \cdot \sum_{i=1}^{M} \alpha_i \cdot \vec{w}_i
\end{equation}
such that the result of the factor $\lambda$ is equivalent to the result of $\pi^*$. In terms of the process for the solution $\omega_{zo}$, we have
\begin{equation}
\gamma \leq \left( \frac{\lambda}{\min_{\omega_{zo}} \left( \frac{1}{M} \sum_{i=1}^{M} \left| \mu' - \mu_{emp} \right|_{\text{pred}} \right)} \right) \cdot \sum_{i=1}^{M} \alpha_i \cdot \vec{w}_i
\end{equation}
and extend the dual form in Equation~\ref{equation:dual} for an upper bound on a full run with $\omega_{zo}$ that includes both periods $\tau$ and $pred$:
\begin{equation}
\begin{aligned}
\gamma \leq \left( \frac{\lambda'}{\min_{\omega'} \left( \frac{1}{M} \sum_{i=1}^{M} \left| \mu' - \mu_{emp} \right|_{pred} \right)} \right) \cdot \sum_{i=1}^{M} & \left( \int_{t=1}^{pred} \vec{w}_{i,t} \, dt \right. \\
& \left. + \int_{t=pred}^{\tau} \alpha_{i,t} \cdot \vec{w}_{i,t} \, dt \right).
\end{aligned}
\end{equation}

\end{proof}

\rule{0pt}{0.1cm}

\end{appendices}

\end{document}